\pdfoutput=1
\documentclass[nohyperref]{article}

\usepackage{microtype}
\usepackage{nowidow}
\PassOptionsToPackage{pdftex}{graphicsx}

\usepackage{booktabs} 

\usepackage[hidelinks]{hyperref}

\usepackage[accepted]{icml2023}

\usepackage{amsmath}
\usepackage{amssymb}
\usepackage{mathtools}
\usepackage{amsthm}
\usepackage{tabularx}
\usepackage{booktabs, makecell, tabularx}
\usepackage{rotating}
\usepackage[export]{adjustbox}
\usepackage{stmaryrd}
\usepackage{multirow}
\usepackage[style=base]{caption}
\usepackage{subcaption}
\usepackage{bm}
\usepackage{tikz}
\usepackage{xcolor,colortbl}

\definecolor{first}{rgb}{0.70588235 0.89019608 0.69411765}
\definecolor{second}{rgb}{0.82352941 0.96078431 0.81568627} 
\definecolor{third}{rgb}{0.90588235 0.98039216 0.90196078} 

\newcommand{\third}{}  
\newcommand{\first}{}  
\newcommand{\second}{}  

\def\E{{\mathbf E}}

\def\0{{\mathbf 0}}

\newcommand{\abs}[1]{\left\lvert#1\right\rvert}

\DeclareMathOperator*{\argmax}{arg\,max}
\DeclareMathOperator*{\argmin}{arg\,min}

\newcommand{\tabf}[1]{\fontseries{b}\selectfont{#1}}

\newcommand{\ff}{\mathfrak{F}}
\newcommand{\pp}{\mathfrak{p}}

\newcommand{\fred}[1]{{\color{magenta}[\textsc{Fred:} #1]}}
\newcommand{\marco}[1]{\textcolor{blue}{[\textsc{Marco:} #1]}}

\usepackage[capitalize]{cleveref}

\creflabelformat{equation}{#2#1#3}

\theoremstyle{plain}
\newtheorem{theorem}{Theorem}[section]
\newtheorem{proposition}[theorem]{Proposition}

\theoremstyle{definition}

\theoremstyle{remark}

\usepackage[textsize=tiny]{todonotes}

\icmltitlerunning{Bayesian Metric Learning for Uncertainty Quantification in Image Retrieval}

\begin{document}

\twocolumn[
\icmltitle{Bayesian Metric Learning for Uncertainty Quantification in Image Retrieval}



\icmlsetsymbol{equal}{*}

\begin{icmlauthorlist}
\icmlauthor{Frederik Warburg}{equal,dtu}
\icmlauthor{Marco Miani}{equal,dtu}
\icmlauthor{Silas Brack}{dtu}
\icmlauthor{Søren Hauberg}{dtu}
\end{icmlauthorlist}

\icmlaffiliation{dtu}{Technical University of Denmark}

\icmlcorrespondingauthor{Frederik, Marco}{\{frwa,mmia\}@dtu.dk}

\icmlkeywords{Machine Learning, ICML}

\vskip 0.3in
]



\printAffiliationsAndNotice{\icmlEqualContribution} 

\begin{abstract}
We propose the first Bayesian encoder for metric learning. Rather than relying on neural amortization as done in prior works, we learn a distribution over the network weights with the Laplace Approximation. We actualize this by first proving that the contrastive loss is a valid log-posterior. We then propose three methods that ensure a positive definite Hessian. Lastly, we present a novel decomposition of the Generalized Gauss-Newton approximation. Empirically, we show that our Laplacian Metric Learner (LAM) estimates well-calibrated uncertainties, reliably detects out-of-distribution examples, and yields state-of-the-art predictive performance. 
\end{abstract}

\vspace{-1.5\baselineskip}
\section{Introduction}
Metric learning seeks data representations where similar observations are near and dissimilar ones are far. This construction elegantly allows for building retrieval systems with simple nearest-neighbor search. Such systems easily cope with a large number of classes, and new classes can organically be added without retraining. While these retrieval systems show impressive performance, they quickly, and with no raised alarms, deteriorate with out-of-distribution data~\cite{shi2019probabilistic}. In particular, in safety-critical applications, the lack of uncertainty estimation is a concern as retrieval errors may propagate unnoticed through the system, resulting in erroneous and possibly dangerous decisions.



\begin{figure}
    \flushright
    \includegraphics[width=0.915\linewidth]{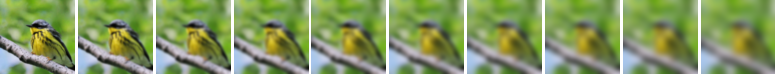} 
    \includegraphics[width=\linewidth]{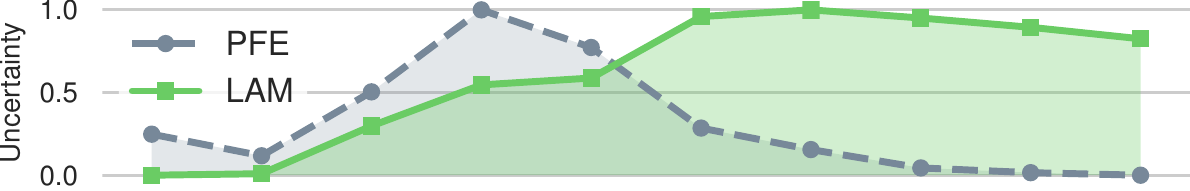}
    \includegraphics[width=0.915\linewidth]{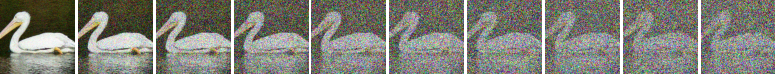}
    \includegraphics[width=\linewidth]{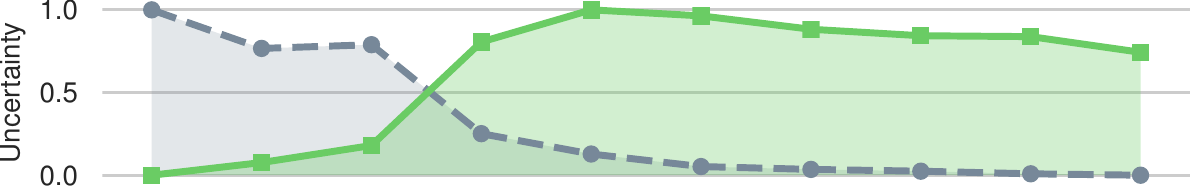}
    \includegraphics[width=0.915\linewidth]{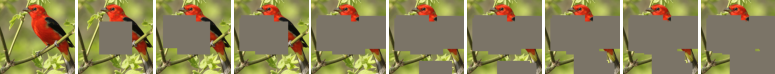}
    \includegraphics[width=\linewidth]{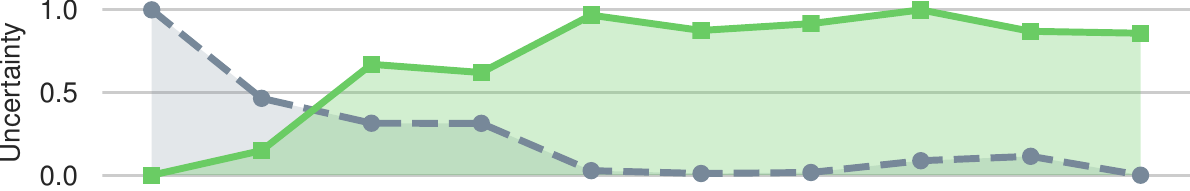}
    \caption{\textbf{Reliable stochastic embeddings.} Current state-of-the-art, PFE~\cite{shi2019probabilistic}, and LAM (ours) learn stochastic representations. LAM estimates reliable uncertainties of the latent representation that intuitively follow the amount of blur, noise, or occlusion in the input image.}
    
    \label{fig:teaser}
\end{figure}

We present the Laplacian Metric Learner (LAM) to estimate reliable uncertainties of image embeddings as demonstrated in \cref{fig:teaser}.  LAM is the first Bayesian method proposed for metric learning. We learn a distribution over the network weights (weight posterior) from which we obtain a stochastic representation by embedding an image through sampled neural networks.
The Bayesian formulation has multiple benefits, namely (1) robustness to out-of-distribution examples, (2) calibrated in-distribution uncertainties, and (3) a slight improvement in predictive performance.\looseness=-1


Our method extends the Laplace Approximation~\citep{mackay1992laplace} for metric learning. We present a probabilistic interpretation of the contrastive loss~\cite{hadsell2006dimensionality} which justifies that it can be interpreted as an unnormalized negative log-posterior.
We then propose three solutions to ensure a positive definite Hessian for the contrastive loss and present two approaches to compute the Generalized Gauss-Newton~\cite{foresee1997ggn} approximation for $\ell_2$-normalized networks. Finally, we boost our method with the online training procedure from \citet{miani_2022_neurips} and achieve state-of-the-art performance.


We are not the first to consider uncertainty quantification in image retrieval. Seminal works~\citep{shi2019probabilistic, oh2018modeling} have addressed the lack of uncertainties in retrieval with \emph{amortized inference} \citep{gershman2014amortized}, where a neural network predicts a stochastic embedding. The issues with this approach are that (1) it requires strong assumptions on the distribution of the embedding, (2) the networks are often brittle and difficult to optimize, and (3) out-of-distribution detection relies on the network's capacity to extrapolate uncertainties. As neural networks extrapolate poorly \citep{xu2021neural}, the resulting \emph{predicted} uncertainties are unreliable for out-of-distribution data \citep{detlefsen2019reliable} and are thus, in practice, of limited value.

In contrast, our method does not assume any distribution on the stochastic embeddings, is simple to optimize, and does not rely on a neural network to extrapolate uncertainties. Instead, our weight posterior is derived from the curvature of the loss landscape and the uncertainties of the latent embeddings deduced (rather than learned) with sampling. We show through rigorous experiments that this leads to reliable out-of-distribution performance and calibrated uncertainties in both controlled toy experiments and challenging real-world applications such as bird, face, and place recognition.

\section{Related Work}






\textbf{Metric learning} attempts to map data to an embedding space, where similar data are close together and dissimilar data are far apart. This is especially useful for retrieval tasks with many classes and few observations per class such as place recognition \citep{Warburg_2020_CVPR} and face recognition \citep{Schroff2015Facenet} or for tasks where classes are not well-defined, such as food tastes \citep{wilber2015snack} or narratives in online discussions \citep{ebert2022narratives}.

There exist many metric losses that optimize for a well-behaved embedding space. We refer to the excellent survey by \citet{musgrave2020metric} for an overview. We here focus on the \emph{contrastive loss}~\cite{hadsell2006dimensionality}
\begin{align}
  \begin{split}
    \mathcal{L}_{\text{con}}(\theta) ={}& \frac{1}{2} \|f_\theta(x_a)-f_\theta(x_p)\|^2 \\ &+
    \frac{1}{2}\max\left(0, m- \|f_\theta(x_a)-f_\theta(x_n)\|^2\right),
    \label{eq:contrastive}
  \end{split}
\end{align}
which has shown state-of-the-art performance \citep{musgrave2020metric} and is one of the most commonly used metric losses. Here, $f_{\theta}$ is a neural network parametrized by $\theta$ which maps from the observation space to the embedding space. 
The loss consists of two terms, one that attracts observations from the same class (\textit{anchor} $x_a$ and \textit{positive} $x_p$), and one that repels observations from different classes (\textit{anchor} $x_a$ and \textit{negative} $x_n$). The margin $m$ ensures that negatives are repelled sufficiently far. 
We will later present a probabilistic extension of the contrastive loss that allows us to learn stochastic, rather than deterministic, features in the embedding space. 

\textbf{Uncertainty in deep learning} is currently studied across many domains to mitigate fatal accidents and allow for human intervention when neural networks make erroneous predictions. Current methods can be divided into methods that apply amortized optimization to train a neural network to predict the parameters of the output distribution, and methods that do not. The amortized methods, best known from the variational autoencoder (VAE) \citep{kingma2013vae, rezende2014stochastic}, seem attractive at first as they can directly estimate the output distribution (without requiring sampling), but they suffer from mode collapse and are sensitive to out-of-distribution data due to the poor extrapolation capabilities of neural networks \citep{nilisnick2019knownothing, detlefsen2019reliable}. `Bayes by Backprop' ~\citep{blundell2015bayesbybackprop} learns a distribution over parameters variationally but is often deemed too brittle for practical applications. Alternatives to amortized methods includes deep ensembles \citep{Lakshminarayanan2016deepensembles}, stochastic weight averaging (SWAG) \citep{Maddox2019swag}, Monte-Carlo dropout \citep{gal2015mcdropout} and Laplace Approximation (LA) \citep{laplace1774memoire, mackay1992laplace} which all approximate the generally intractable weight posterior $p(\theta | \mathcal{D})$ of a neural network. We propose the first principled method to approximate the weight posterior in metric learning.

\textbf{Laplace approximations (LA)} can be applied for every loss function $\mathcal{L}$ that can be interpreted as an unnormalized log-posterior by performing a second-order Taylor expansion around a chosen weight vector $\theta^*$ such that
\begin{align}\label{eq:taylor_laplace}
  \begin{split}
    \mathcal{L}(\theta)
    \approx{}&
    \mathcal{L}^*
    + (\theta - \theta^*)^{\top} \nabla \mathcal{L}^* \\
    & + \frac{1}{2} (\theta - \theta^*)^{\top} \nabla^2 \mathcal{L}^* (\theta - \theta^*),
  \end{split}
\end{align}
where $\mathcal{L}^*$ is the loss evaluated in $\theta^*$. Imposing the unnormalized log-posterior to be a second-order polynomial is equivalent to assuming the posterior to be Gaussian.
If $\theta^*$ is a MAP estimate, the first-order term vanishes, and the second-order term can be interpreted as a precision matrix, the inverse of the covariance.
Assuming $\theta^*$ is a MAP estimate, this second-order term is negative semi-definite for common (convex) supervised losses, such as the mean-squared error and cross-entropy.
Recently, \citet{daxberger2021laplaceredux} demonstrated that post-hoc LA is scalable and produces well-behaved uncertainties for classification and regression.
The Laplacian Autoencoder (LAE) \citep{miani_2022_neurips} improves on the post-hoc LA with an online Monte Carlo EM training procedure to learn a well-behaved posterior.
It demonstrates state-of-the-art uncertainty quantification for unsupervised representation learning.
We first extend LA to the contrastive loss and achieve state-of-the-art performance with the online EM training procedure.

\textbf{Uncertainty in metric learning} is not a new idea \citep{vilnis2014word}, but the large majority of recent methods apply amortized inference to predict distributions in the embedding space \citep{warburg2020btl, chang2020data, chun2021crossmodal, oh2018modeling, shi2019probabilistic, song2019crossmodal, sun2019prob4pose}, making them sensitive to mode collapse and out-of-distribution data. \citet{taha2019dropout, taha2019ensembles} explore deep ensembles and Monte-Carlo dropout as alternatives, but these methods suffer from increased training time, poor empirical performance, and limited Bayesian interpretation~\citep{daxberger2021laplaceredux}. We explore LA in metric learning and attain state-of-the-art performance.



\section{Laplacian Metric Learning}

To perform Bayesian retrieval, we propose to estimate the weight posterior of the embedding network $f_{\theta}$ such that we can sample data embeddings to propagate uncertainty through the decision process. The embedding network is parametrized by $\theta \in \Theta$ and trained with the contrastive loss.
The network maps an image $x \in \mathcal{X}:=\mathbb{R}^{HWC}$ to an embedding $z \in \mathcal{Z}$, which is restricted to be on a $Z$-dimensional unit sphere $\mathcal{Z}:=\mathcal{S}^Z$.
This normalization to the unit sphere is commonly done in retrieval to obtain faster retrieval and a slight performance boost~\citep{Arandjelovic2015netvlad, Radenovic2017gem}.  \cref{fig:model_overview} illustrate our Bayesian mapping from image to latent space.

To obtain an approximate posterior over the weights $\theta$ we rely on the Laplace approximation (LA).
We first motivate the post-hoc LA as this is the simplest. We then extend this approach to online LA which marginalizes the Laplace approximation during training to improve model fit.
In \cref{sec:probabilistic_view}, we prove that the contrastive loss is a valid unnormalized log-posterior on the compact spherical space.
The proof draws inspiration from electrostatics, and the main idea is to define a PDF for the attracting and repelling terms. We then show that the logarithm of the product of these PDFs is equivalent (up to some constant) to the contrastive loss (details are postponed to \cref{sec:probabilistic_view}). Since the contrastive loss is a valid log-posterior, we can proceed with LA\@.\looseness=-1



\begin{figure}
    \centering
    \includegraphics[width=\linewidth]{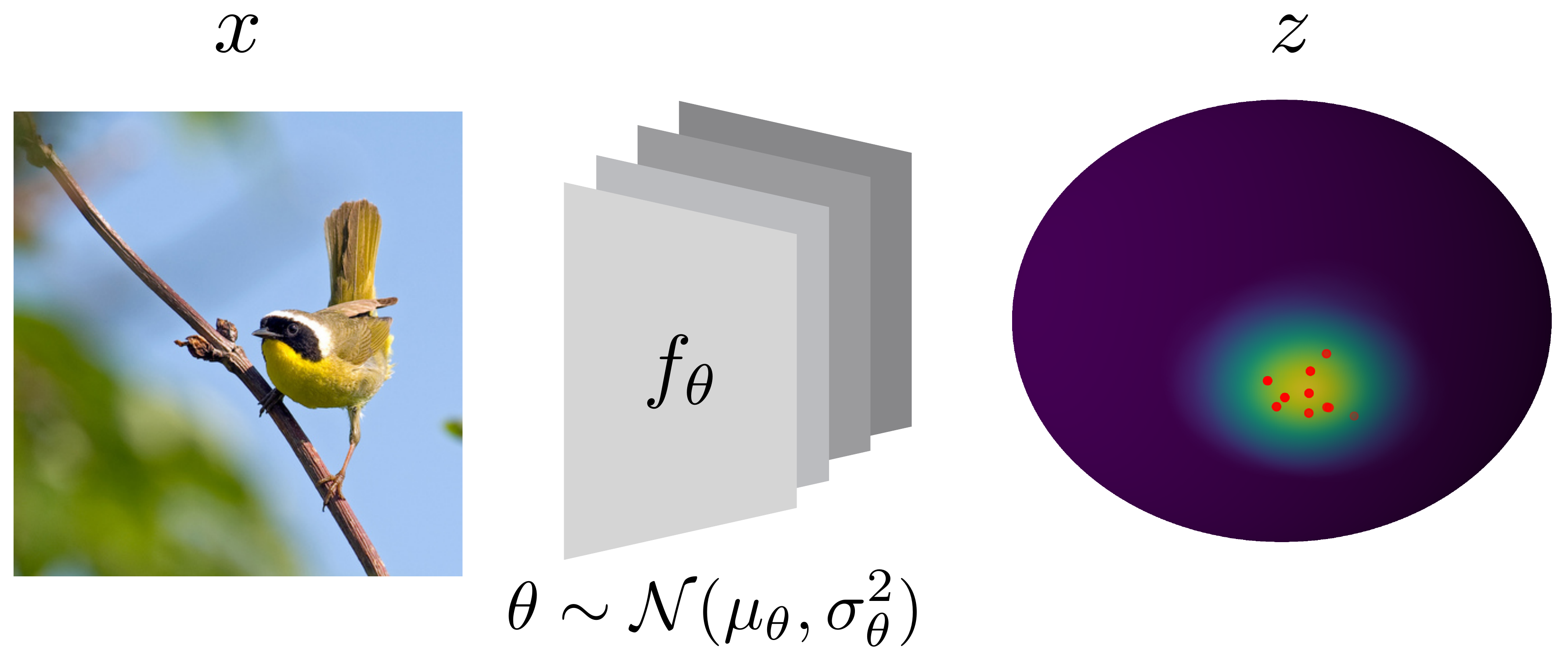}
    \caption{\textbf{Model overview}. We learn a distribution over parameters, such that we embed an image through sampled encoders $f_\theta$ to points $z_i$ (red dots) in a latent space $\mathcal{Z}$. We reduce these latent samples to a single measure of uncertainty by estimating the parameters of a von Mises-Fisher distribution.  
    }
    \label{fig:model_overview}
\end{figure}

\textbf{A post-hoc Laplace approximation} is found by first training a standard deterministic network through gradient steps with the contrastive loss to find the \emph{maximum a posteriori} (MAP) parameters $\theta^{*}$. Since we are in a local optimum, the first-order term in the second-order Taylor expansion (\cref{eq:taylor_laplace}) vanishes, and we can define the parameter distribution as
\begin{equation}
    p(\theta|\mathcal{D})
    =
    \mathcal{N}\left(
    \theta\Big|\theta^{*},
    \left(\nabla^2_\theta \mathcal{L}_{\text{con}}\left(\theta^{*};\mathcal{D}\right)
        +
        \sigma_{\text{prior}}^{-2}\mathbb{I}\right)^{-1} \right).
\end{equation}
The advantage of post-hoc LA is that the training procedure does not change, and already trained neural networks can be made Bayesian.
In practice, however, stochastic gradient-based training does not locate isolated minima of the loss landscape, but rather ends up exploring regions near local minima.
The Hessian (and hence the posterior covariance) can change significantly during this exploration, and the post-hoc LA can become unstable.


\textbf{Online Laplace approximations} \citep{miani_2022_neurips} avoids this instability by marginalizing the LA during training with Monte Carlo EM.
This helps the training recover a solution $\theta^*$ where the Hessian reflects the loss landscape.
Specifically, at each step $t$ during training we keep in memory a Gaussian distribution on the parameters $q^t(\theta)=\mathcal{N}(\theta|\theta_t,H_{\theta_{t}}^{-1})$.
The parameters are updated through an expected gradient step
\begin{equation}
    \theta_{t+1}
    =
    \theta_t
    + \lambda
    \mathbb{E}_{\theta\sim q^t}
    [\nabla_\theta \mathcal{L}_{\text{con}}(\theta;\mathcal{D})]
\end{equation}
and a discounted Laplace update
\begin{equation}
    H_{\theta_{t+1}}
    =
    (1-\alpha) H_{\theta_t}
    +
    \nabla^2_\theta \mathcal{L}_{\text{con}}(\theta;\mathcal{D}),
\end{equation}
where $\alpha$ describes an exponential moving average, similar to momentum-like training. The initialization follows the isotropic prior $q^0(\theta) = \mathcal{N}(\theta|0,\sigma^2_{\text{prior}}\mathbb{I})$.

In practice, the Hessian scales quadratically in memory wrt. the number of parameters. To mitigate this, we approximate this Hessian by its diagonal \citep{lecun1989optimal,denker1990transforming}.


\textbf{Hessian of the contrastive loss.}
Both post-hoc and online LA require the Hessian of the contrastive loss $\nabla^2_\theta \mathcal{L}_{\text{con}}(\theta;\mathcal{D})$.
The Hessian is commonly approximated with the Generalized Gauss-Newton (GGN) approximation~\citep{foresee1997ggn, daxberger2021laplaceredux, dangel2020backpack, software:stochman}. 
The GGN decomposes the loss into $\mathcal{L} = g \circ f$, where $g$ is usually chosen as the loss function and $f$ the model function, and only $f$ is linearized.

However, in our case, this decomposition is non-trivial. Recall that the last layer of our network is an $\ell_2$ normalization layer, which projects embeddings onto a hyper-sphere. This normalization layer can either be viewed as part of the model $f$ (linearized normalization layer) or part of the loss $g$ (non-linearized normalization layer). We show in \cref{sec:derivatives_euclidean} that the former can be interpreted as using the \textit{Euclidean} distance and in \cref{sec:derivatives_arccos} that the latter as using the \textit{Arccos} distance for the contrastive loss (\cref{eq:contrastive}). We highlight that these share the zero- and first-order terms for normalized embeddings but, due to the GGN linearization, not the second-order derivatives. The Euclidean interpretation leads to simpler derivatives and interpretations, and we will therefore use it for our derivations. We emphasize that the Arccos is theoretically a more accurate approximation, because the $\ell_2$-layer is not linearized, and we provide derivations in \cref{sec:derivatives}. 



The GGN matrix for contrastive loss with the \textit{Euclidean} interpretation is given by 
\begin{align}\label{eq:ggn-euclidean}
     & \nabla^2_\theta \mathcal{L}_{\text{con}}(\theta; \mathcal{I}) = \sum_{ij \in \mathcal{I}} H^{ij}_\theta = \sum_{ij \in \mathcal{I}_{p}} H^{ij}_{\theta} + \sum_{ij \in \mathcal{I}_{n}} H^{ij}_{\theta}                                                                                   \\
     & \stackrel{\textsc{ggn}}{\approx} \sum_{ij \in \mathcal{I}_{p}} {J^{ij}_\theta}^\top \underbrace{\left(\begin{smallmatrix*}[r] 1 & -1 \\ -1 & 1 \end{smallmatrix*}\right)}_{:=H_p}
     J^{ij}_\theta + \sum_{ij \in \mathcal{I}_{n}} {J^{ij}_\theta}^\top \underbrace{\left(\begin{smallmatrix*}[r] -1 & 1 \\ 1 & -1 \end{smallmatrix*}\right)}_{:=H_n}
     J^{ij}_\theta, \nonumber
\end{align}
where $J^{ij}_\theta = \left(J_\theta f_\theta(x_i)^\top, J_\theta f_\theta(x_j)^\top \right)^\top$, with $J_\theta$ being the Jacobian wrt. the parameters, and where $H_p$ and $H_n$ are the Hessian of the contrastive loss wrt. the model output for positive and negative pairs. Notice that the first sum runs over positive pairs and the second sum runs over negative pairs \textit{within the margin}. Negative pairs outside the margin do not contribute to the Hessian, and can therefore be ignored to reduce the computational load~(\cref{sec:contrastive_loss}).

The eigenvalues of the Hessian wrt. to the output are $(0, 2)$ and $(-2, 0)$ for the positive $H_p$ and negative $H_n$ terms, so we are not guaranteed to have a positive definite Hessian, $H_{\theta}$.
To avoid covariances with negative eigenvalues, we propose three solutions to ensure a positive definite Hessian. Proofs are in \cref{sec:hessian_approximations}.




\begin{figure}
    \centering
    \includegraphics[width=1\linewidth]{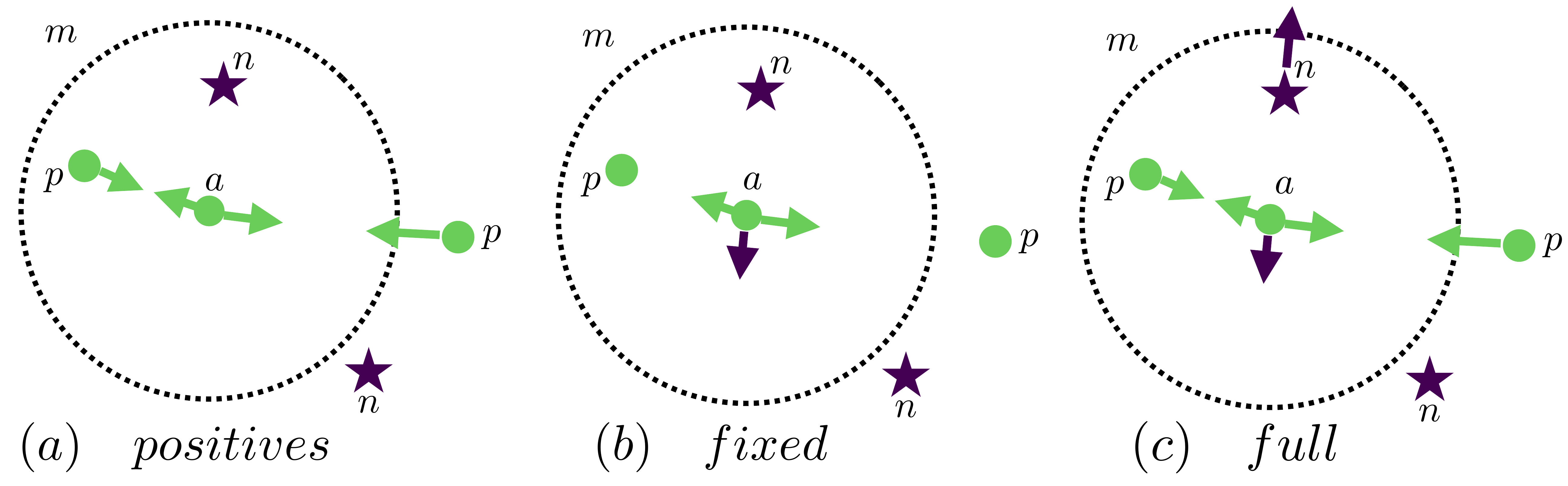}
    \caption{\textbf{Hessian approximations.} To ensure a positive definite Hessian approximation we propose three approximations. In (a) only the positives $p$ contribute to the Hessian as the negatives $n$ are ignored. In (b) we consider one point at a time, e.g., only the anchor $a$ contributes. In (c) we consider all interactions.}
    \vspace{-1\baselineskip}
    \label{fig:Hessian_approximations}
\end{figure}

\textbf{Ensuring positive definiteness of the Hessian}. We do not want to be restricted in the choice of the prior, so we must ensure that $\nabla^2_\theta\mathcal{L}_{\text{con}}(\theta^{*};\mathcal{D})$ is positive definite itself. Differently from the standard convex losses, this is not ensured by the GGN approximation~\cite{immer2020bnnlocallocalization}. 
Our main insight is that we can ensure a positive definite Hessian $H_\theta$ by only manipulating the Hessians $H_p$ and $H_n$ in \cref{eq:ggn-euclidean}. 

\textit{1. Positive: The repelling term is ignored, such that only positive pairs contribute to the Hessian.}
\begin{equation}\label{eq:pos}
    H_p = \begin{pmatrix*}[r]
        1 & -1 \\ -1 & 1
    \end{pmatrix*},
    \qquad
    H_n =\begin{pmatrix*}[r]
        0 & 0 \\ 0 & 0
    \end{pmatrix*}
\end{equation}

\textit{2. Fixed: The cross derivatives are ignored.} 
\begin{equation}\label{eq:fix}
    H_p = \begin{pmatrix*}[r]
        1 & 0 \\ 0 & 1
    \end{pmatrix*},
    \qquad
    H_n =\begin{pmatrix*}[r]
        -1 & 0 \\ 0 & -1
    \end{pmatrix*}
\end{equation}

\textit{3. Full: Nothing is ignored, but rather positive definiteness is ensured with ReLU, $\max(0, \nabla^2_\theta \mathcal{L}_{\text{con}}(\theta;\mathcal{D}))$, on the Hessian of the loss wrt. the parameters. 
}
\begin{equation}\label{eq:full}
    H_p = \begin{pmatrix*}[r]
        1 & -1 \\ -1 & 1
    \end{pmatrix*},
    \qquad
    H_n =\begin{pmatrix*}[r]
        -1 & 1 \\ 1 & -1
    \end{pmatrix*}
\end{equation}


The \textit{positive} approximation is inspired by~\citet{shi2019probabilistic}, which also only uses positive pairs to train their uncertainty module. The gradient arrows in \cref{fig:Hessian_approximations} (a) illustrate that negative pairs are neglected when computing the Hessian of the contrastive loss. The
\textit{fixed} approximation considers one data point at the time, assuming the other one is fixed. Thus, given a pair of data points, this can be interpreted as first moving one data point, and then the second (rather than both at the same time). We formalize this in \cref{sec:hessian_approximations}. \cref{fig:Hessian_approximations} (b) illustrate this idea when all points except $a$ are fixed. 
Lastly, we propose the \textit{full} Hessian of the contrastive loss (\cref{fig:Hessian_approximations} (c)) and ensure a positive definiteness by computing the ReLU of the Hessian. We note that this approximation assumes a diagonal Hessian. We experimentally determine, which of these three methods to ensure a positive definiteness leads to better performance (\cref{sec:ablation}).


\textbf{Intuition of Hessian approximations.} With \cref{fig:intuition} we try to provide some intuition on these three approximations and their impact in a clustered versus cluttered latent space. Positive points will increase the precision (magnitude of the Hessian) while negative points inside the margin will decrease the precision (as indicated with the arrows in \cref{fig:intuition}).
Negative points outside the margin will not affect the gradient or the Hessian.
This is the desired behavior as a perfect clustered latent space (\cref{fig:intuition} b) will have high precision (low uncertainty), whereas in a cluttered latent space (\cref{fig:intuition} a) will have many negatives within the margin, which will decrease the precision (large uncertainty).
Therefore, the positive approximation (\cref{eq:pos}) will provide the lowest uncertainty, as the negatives are excluded.
In practice, we cannot compute the Hessian for all pairs, as this scales quadratically in time and memory with the dataset size. Furthermore, most pairs, namely the negatives outside the margin, do not contribute to the Hessian, so it is wasteful to compute their Hessian. Therefore, we make use of a common trick in metric learning, namely hard negative mining~\cite{musgrave2020metric}.

\begin{figure}
    \centering
    \begin{subfigure}{0.49\linewidth}
        \centering
        \includegraphics[width=\linewidth]{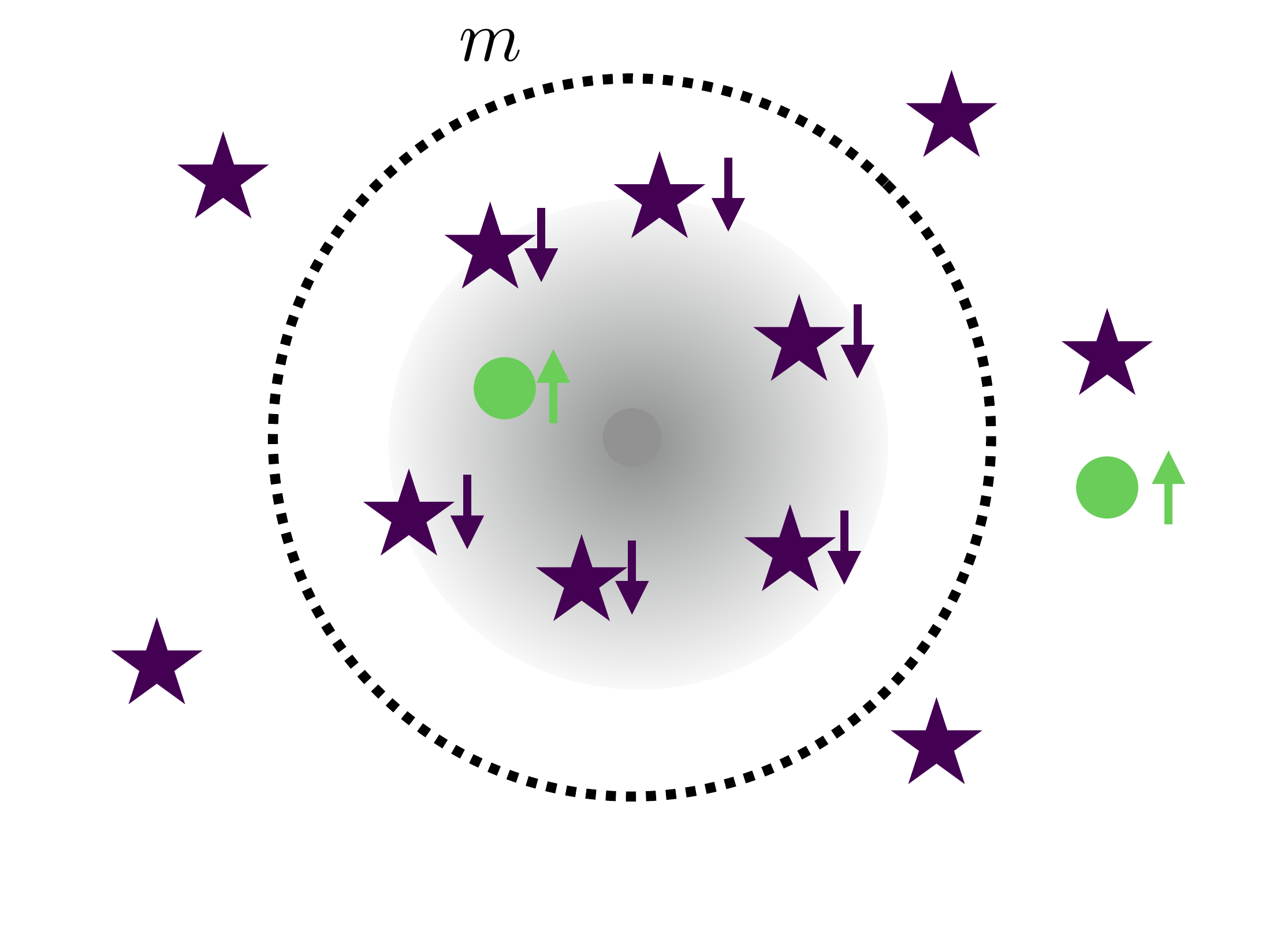}
        \vspace{-2.0\baselineskip}
        \caption{Cluttered}
        \label{fig:sub1}
    \end{subfigure}%
    \begin{subfigure}{0.49\linewidth}
        \centering
        \includegraphics[width=\linewidth]{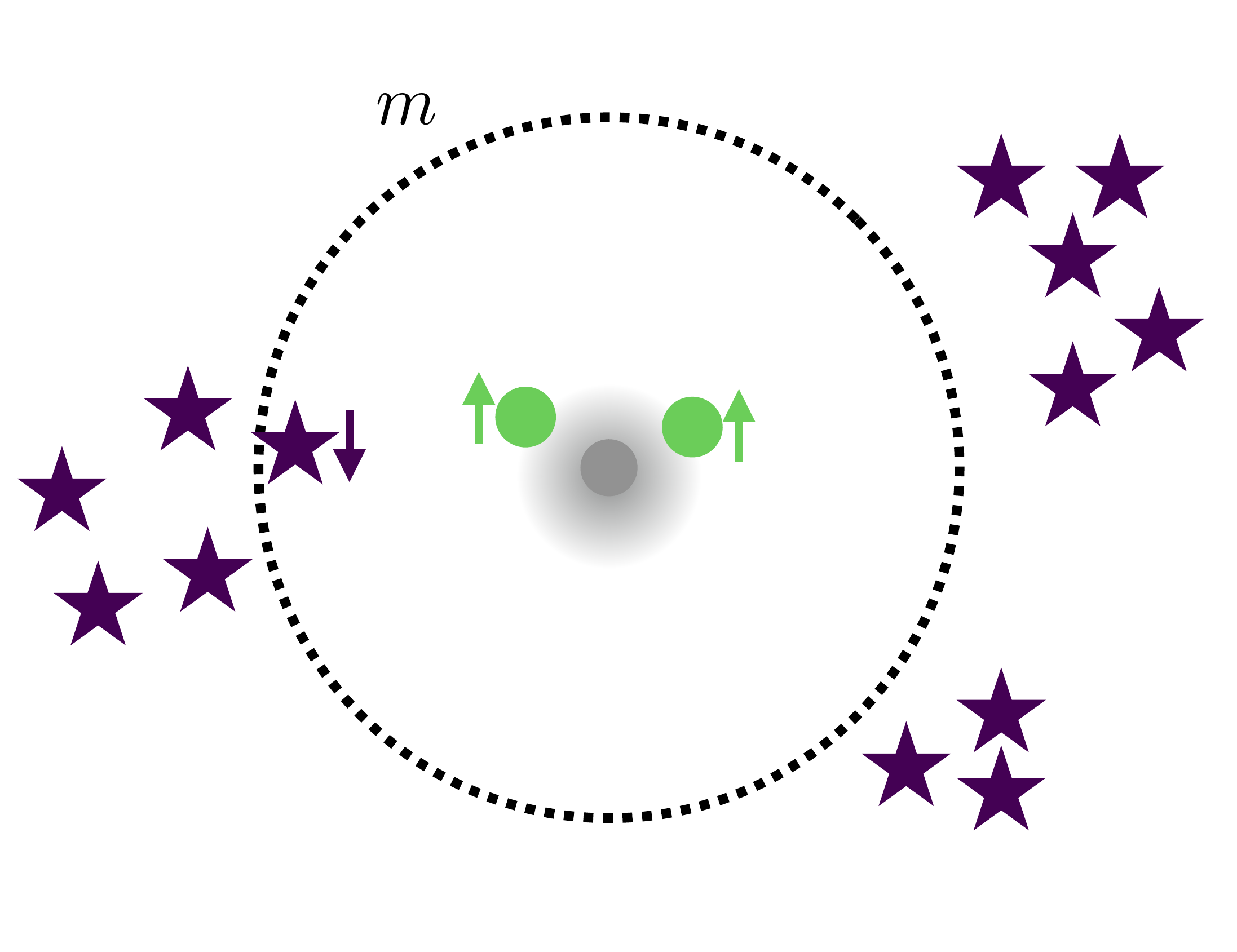}
        \vspace{-2.0\baselineskip}
        \caption{Clustered}
        \label{fig:sub2}
    \end{subfigure}
    \vspace{-0.6\baselineskip}
    \caption{\textbf{Intuition of $\mathbf{H_p}$ and $\mathbf{H_n}$.} Illustration of a cluttered latent space, where observations from different classes are close, and a clustered latent space with distinct clusters with observations from the same class. Intuitively, negative examples $\star$ inside the margin decreases the precision (higher variance) with $H_n$, and positive points $\bullet$ will increase the precision (lower variance) with $H_p$.}
    \vspace{-1\baselineskip}
    \label{fig:intuition}
\end{figure}

\textbf{Hard negative mining leads to biased Hessian approximation.} Metric losses operate locally, hence non-zero losses will only occur for negative examples that are close to the anchor in the latent space. Presenting the model for randomly sampled data will in practice rarely result in negative pairs that are close, which leads to extremely long training times. Therefore, \textit{hard negative mining} is often used to construct batches that have an over-representation of hard negative pairs (negative examples that are close to the anchor). We use similar mining when computing the Hessian, which leads to a biased estimate of the Hessian. We thus introduce a scaling parameter $0 \leq w_n \leq 1$ to obtain
\begin{equation}
    \hat{H}_\theta = (1-w_n) J_\theta^\top H_p J_\theta + w_n  J_\theta^\top H_n J_\theta,
\end{equation}
which corrects the biased estimate (see \cref{appendix:mining} for proof). Our positive approximation (\cref{eq:pos}) can be seen as the extreme case with $w_n = 0$, whereas our full approximation \cref{eq:full} ($w_n = 0.5$) corresponds to unbiased sampling.

Having an estimate of the Hessian in place and three methods to ensure it is positively definite, we can perform both post-hoc and online LA to obtain a distribution over the weights. We can sample from this weight posterior and embed an input image via each sampled network to obtain multiple samples in the latent space (see \cref{fig:model_overview}). We reduce these samples to a single measure of uncertainty by estimating the parameters of a von~Mises-Fisher distribution. 

\textbf{Why the von~Mises-Fisher distribution?} 
The uncertainty deduced from LA is usually computed by the variance of the samples~\cite{daxberger2021laplaceredux}. This assumes that the samples follow a Gaussian distribution. However, for $\ell_2$-normalized networks, we assume that all the probability mass lies on a $Z$-dimensional hyper-sphere. The von~Mises-Fisher distribution describes such distribution, and it is parametrized with a directional mean $\mu$ and a scalar concentration parameter $\kappa$, which can be interpreted as the inverse of an isotropic covariance $\kappa = 1 / \sigma^2$, i.e., small $\kappa$ means high uncertainty and large $\kappa$ means low uncertainty.

There exist several methods to estimate $\kappa$. We opt for the simplest and most computationally efficient~\cite{sra2011vonmises} (see \cref{sec:normalization-vmf}). Prior works~\citep{shi2019probabilistic, taha2019dropout, taha2019ensembles} have treated an estimated latent distribution as a Gaussian, although all probability mass lies on the unit sphere. This is insufficient, as samples from the Gaussian distribution will not lie on the unit sphere. In the experimental work, we correct this by projecting the samples onto the unit sphere.

\begin{figure*}
  \centering
  \begin{subfigure}{\dimexpr.5\linewidth-0.160in\relax}
    \centering
    \includegraphics[width=1\linewidth]{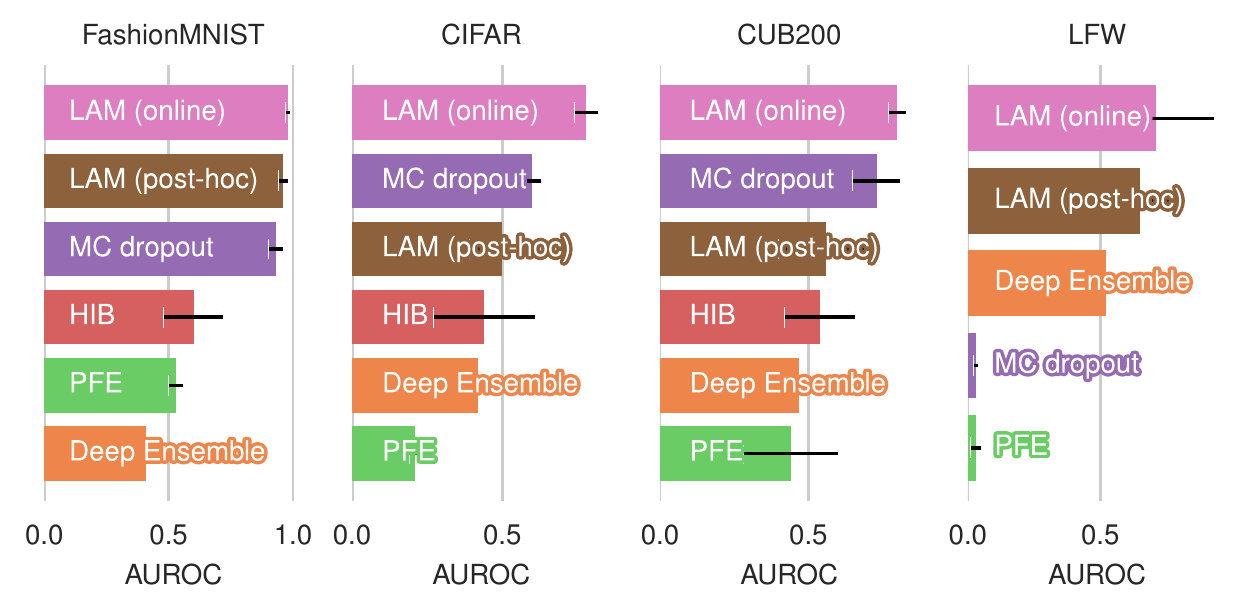}
    \vspace{-1.5\baselineskip}
    \caption{Out-of-distribution detection.}
    \label{fig:overview-auroc}
  \end{subfigure}
  \hspace{\the\columnsep}
  \begin{subfigure}{\dimexpr.5\linewidth-0.160in\relax}
    \centering
    \includegraphics[width=1\linewidth]{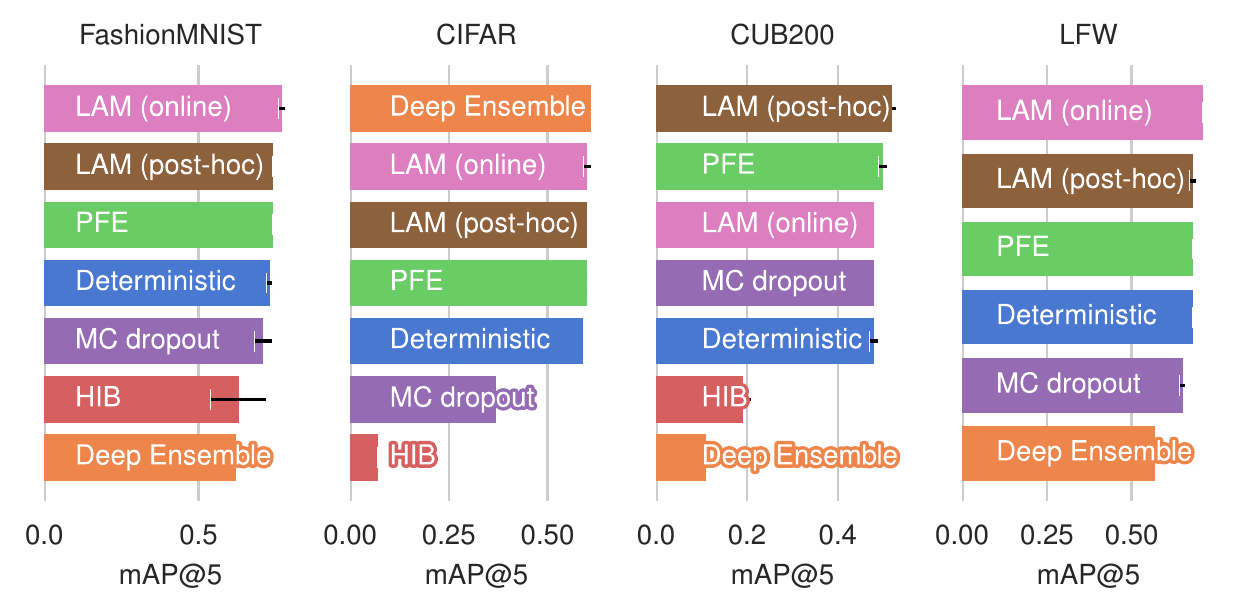}
    \vspace{-1.5\baselineskip}
    \caption{Predictive performance.}
    \label{fig:overview-map}
  \end{subfigure}
  \caption{\textbf{Summary of experimental results.} LAM consistently outperforms existing methods on \textsc{OoD} detection, as measured by AUROC, and matches or surpasses in predictive performance measured by mAP@$k$. Error bar shows $\pm$ one standard deviation measured across five runs.}
  \vspace{-1.2\baselineskip}
  \label{fig:overview}
\end{figure*}

\section{Experiments}

We benchmark our method against strong probabilistic retrieval models. Probabilistic Face Embeddings (PFE) \citep{shi2019probabilistic} and Hedge Image Embedding (HIB) \citep{oh2018modeling} perform amortized inference and thus estimate the mean and variance of latent observation. We also compare against MC Dropout~\citep{gal2015mcdropout} and Deep Ensemble~\citep{Lakshminarayanan2016deepensembles}, two approximate Bayesian methods, which have successfully been applied in image retrieval~\citep{taha2019dropout, taha2019ensembles}.

We compare the models' \textit{predictive performance} with the recall (recall@$k$) and mean average precision (mAP@$k$) among the $k$ nearest neighbors \citep{warburg2020btl, musgrave2020metric, Arandjelovic2015netvlad}. We evaluate the models' abilities to \textit{interpolate} and \textit{extrapolate} uncertainties by measuring the Area Under the Sparsification Curve (\textsc{AUSC}) and Expected Calibration Error (\textsc{ECE}) on in-distribution (ID) data, and the Area Under Receiver Operator Curve (\textsc{AUROC}) and Area Under Precision-Recall Curve (\textsc{AUPRC}) on out-of-distribution (\textsc{OoD}) data. We provide more details on these metrics in \cref{sec:metrics}.

We extend StochMan \citep{software:stochman} with the Hessian backpropagation for the contrastive loss. The training code is implemented in PyTorch \citep{paszke2017automatic} and is available on GitHub\footnote{Code: \url{https://github.com/FrederikWarburg/bayesian-metric-learning}}. \cref{sec:experimental_details} provides more details on the experimental setup.

\textbf{Experimental Summary.} We begin with a short summary of our experimental results. Across five datasets, three different network architectures, and three different sizes of the latent space (ranging from $3$ to $2048$), we find that LAM has well-calibrated uncertainties, reliably detects \textsc{OoD} examples, and achieves state-of-the-art predictive performance. \cref{fig:overview-auroc} shows that the uncertainties from online LAM reliably identify \textsc{OoD} examples. Online LAM outperforms other Bayesian methods, such as post-hoc LAM and MC dropout, on this task, which in turn clearly improves upon amortized methods that rely on a neural network to extrapolate uncertainties. \cref{fig:overview-map} shows that LAM consistently matches or outperforms existing image retrieval methods in terms of predictive performance. We find that the fixed Hessian approximation with the Arccos distance performs the best, especially on higher dimensional data. 

\textbf{Ablation: Positive Definiteness of the Hessian.}\label{sec:ablation} We experimentally study which method to ensure a positive definite Hessian has the best performance measured in both predictive performance (mAP@$5$) and uncertainty quantification (\textsc{AUROC} , \textsc{AUSC}).
We found that all these methods perform similarly on simple datasets and low dimensional hyper-spheres, but the fixed approximation with Arccos distance performs better on more challenging datasets and higher dimensional hyper-spheres.
We present results on one of these more challenging datasets, namely the LFW~\cite{LFWTech} face recognition dataset with the CUB200~\cite{WahCUB_200_2011} bird dataset as an \textsc{OoD} dataset.
We use a ResNet50~\cite{he2016deep} with a GeM pooling layer \cite{Radenovic2017gem} and a $2048$ dimensional embedding and diagonal, last-layer LA~\cite{daxberger2021laplaceredux}.

\cref{tab:ablation_ood} shows the performance for post-hoc and online LA with fixed, positive, or full Hessian approximation using either Euclidean or Arccos distance. Across all metrics, the online LA with Arccos distance and the fixed Hessian approximation performs similarly or the best. We proceed to benchmark this method against several strong probabilistic baselines on closed-set retrieval and a more challenging open-set retrieval.

\begin{table}[]
    \centering
        \caption{\textbf{Ablation on Hessian approximation and GGN decomposition.} Online LA with the fixed approximation and Arccos distance performs best. Error bars show one std. deviation across five runs.\looseness=-1 }
\resizebox{\linewidth}{!}{
    \begin{tabular}{cllll}
    \toprule
    & &   mAP@$5$ $\uparrow$   &   \textsc{AUROC}  $\uparrow$   &   \textsc{AUSC}  $\uparrow$\\ \midrule 
    \multirow{6}{*}{\rotatebox{90}{Post-hoc}}
&  Euclidean fix & \second 0.70 $\pm$ 0.0 & 0.57 $\pm$ 0.25 & 0.44 $\pm$ 0.01 \\
 &  Euclidean pos & \second 0.70 $\pm$ 0.0 & 0.58 $\pm$ 0.23 & 0.45 $\pm$ 0.01 \\
 &  Euclidean full & \second 0.70 $\pm$ 0.0 & 0.56 $\pm$ 0.26 & 0.44 $\pm$ 0.01 \\
 &  Arccos fix & \third 0.69 $\pm$ 0.0 & 0.53 $\pm$ 0.20 & 0.46 $\pm$ 0.02 \\
 &  Arccos pos & \second 0.70 $\pm$ 0.0 & 0.29 $\pm$ 0.11 & \third 0.48 $\pm$ 0.01 \\
 &  Arccos full & \third 0.69 $\pm$ 0.0 & 0.55 $\pm$ 0.18 & 0.45 $\pm$ 0.01 \\ \hline
 \multirow{6}{*}{\rotatebox{90}{Online}}
 &  Euclidean fix & 0.63 $\pm$ 0.01 & \second 0.77 $\pm$ 0.04 & 0.31 $\pm$ 0.02 \\
 &  Euclidean pos & \second 0.70 $\pm$ 0.0 & 0.38 $\pm$ 0.10 & 0.47 $\pm$ 0.01 \\
 &  Euclidean full & 0.67 $\pm$ 0.01 & 0.59 $\pm$ 0.04 & 0.42 $\pm$ 0.01 \\
 &  Arccos fix & \first \tabf{0.71 $\pm$ 0.0} & \first \tabf{0.78 $\pm$ 0.18} & \second 0.50 $\pm$ 0.03 \\
 &  Arccos pos & \second 0.70 $\pm$ 0.0 & 0.23 $\pm$ 0.03 & 0.46 $\pm$ 0.00 \\
 &  Arccos full & \first \tabf{0.71 $\pm$ 0.0} & \third 0.70 $\pm$ 0.12 & \first \tabf{0.51 $\pm$ 0.01} \\
 \bottomrule
    \end{tabular}
    }
    \vspace{-1.5\baselineskip}
    \label{tab:ablation_ood}
\end{table}


\textbf{Closed-Set Retrieval.} \textsc{OoD} capabilities are critical for identifying distributional shifts, outliers, and irregular user inputs, which might hinder the propagation of erroneous decisions in an automated system.
We evaluate \textsc{OoD} performance on the commonly used benchmarks \citep{nilisnick2019knownothing}, where we use (1) FashionMNIST \citep{xiao2017fashionmnist} as \textsc{ID} and MNIST \citep{lecun1998mnist} as \textsc{OoD}, and (2) CIFAR10 \cite{Krizhevsky09learningmultiple} as \textsc{ID} and SVHN \cite{netzer2011reading} as \textsc{OoD}.
We use, respectively, a standard $2$- or $3$-layer relu convolutional network followed by a single linear layer on which we compute LA with a diagonal Hessian.

\begin{table*}
    \centering
    \caption{\textbf{Closed-set results.} LAM matches or outperforms existing methods in terms of predictive performance. It produces reliable uncertainties \textsc{ID} and \textsc{OoD} on two standard datasets FashionMNIST and CIFAR10. Confidence intervals show one standard deviation computed across five runs. 
    }
    \label{tab:ooddetection}
    \small
    \begin{tabular}{cl|lll|ll|ll}
        \toprule
         &                & \multicolumn{3}{c|}{\textsc{Image retrieval}} & \multicolumn{2}{c|}{\textsc{OoD}} & \multicolumn{2}{c}{\textsc{Calibration}}                                                                                                     \\
         &                & mAP@1 $\uparrow$                                        & mAP@5 $\uparrow$                            & mAP@10 $\uparrow$                                  & \textsc{AUROC}  $\uparrow$                 & \textsc{AUPRC}  $\uparrow$                  & \textsc{AUSC}   $\uparrow$                 & \textsc{ECE}  $\downarrow$                    \\ \midrule
        \multirow{7}{*}{\rotatebox{90}{FashionMNIST}}
         & Deterministic  & \second 0.78 $\pm$ 0.01                               & \third  0.73 $\pm$ 0.01                   & \second 0.72 $\pm$ 0.01                          & ---                    & ---                    & ---                    & ---                    \\
         & Deep Ensemble  & 0.69                                          & 0.62                              & 0.59                                     & 0.41                   & 0.46                   & 0.61                   & 0.04                   \\
         & PFE            & \second 0.78 $\pm$ 0.00                               & \second 0.74 $\pm$ 0.00                   & \second 0.72 $\pm$ 0.00                          & 0.53 $\pm$ 0.03        & 0.46 $\pm$ 0.01        & 0.65 $\pm$ 0.01        & 0.26 $\pm$ 0.02        \\
         & HIB            & 0.69 $\pm$ 0.08                               & 0.63 $\pm$ 0.09                   & 0.61 $\pm$ 0.09                          & 0.60 $\pm$ 0.12        & 0.60 $\pm$ 0.11        & 0.65 $\pm$ 0.08        & \third 0.54 $\pm$ 0.08        \\
         & MC dropout     & \third 0.76 $\pm$ 0.03                               & 0.71 $\pm$ 0.03                   & \third  0.70 $\pm$ 0.03                          & \third 0.93 $\pm$ 0.03        & \third 0.93 $\pm$ 0.03        & \third  0.84 $\pm$ 0.06        & \second 0.03 $\pm$ 0.04        \\
         & LAM (post-hoc) & \second 0.78 $\pm$ 0.00                               & \second 0.74 $\pm$ 0.00                   & \second 0.72 $\pm$ 0.00                          & \second 0.96 $\pm$ 0.02        & \second 0.96 $\pm$ 0.02        & \second 0.86 $\pm$ 0.01        & \second 0.03 $\pm$ 0.00        \\
         & LAM (online)   & \first \tabf{0.81 $\pm$ 0.00}                        & \first \tabf{0.77 $\pm$ 0.01}            & \first \tabf{0.76 $\pm$ 0.01}                   & \first \tabf{0.98 $\pm$ 0.01} & \first \tabf{0.98 $\pm$ 0.01} & \first \tabf{0.89 $\pm$ 0.01} & \first \tabf{0.02 $\pm$ 0.00} \\
        \midrule
        \multirow{7}{*}{\rotatebox{90}{Cifar10}}
         & Deterministic  & \first \tabf{0.66 $\pm$ 0.00}                        & \third 0.59 $\pm$ 0.00                   & \second 0.58 $\pm$ 0.00                          & ---                    & ---                    & ---                    & ---                    \\
         & Deep Ensemble  & \first \tabf{0.66}                                   & \first \tabf{0.61}                       & \first \tabf{0.59}                              & 0.42                   & 0.67                   & \third 0.72                   & \second 0.02                   \\
         & MC dropout     & \second 0.46 $\pm$ 0.01                               & 0.37 $\pm$ 0.01                   & 0.34 $\pm$ 0.01                          & \second 0.60 $\pm$ 0.03        & \second 0.76 $\pm$ 0.02        & 0.61 $\pm$ 0.01        & 0.05 $\pm$ 0.00        \\
         & HIB            & \third 0.11 $\pm$ 0.01                               & 0.07 $\pm$ 0.00                   & 0.05 $\pm$ 0.00                          & 0.44 $\pm$ 0.17        & \third 0.70 $\pm$ 0.1         & 0.29 $\pm$ 0.03        & \third 0.04 $\pm$ 0.02        \\
         & PFE            & \first \tabf{0.66 $\pm$ 0.00}                        & \second 0.60 $\pm$ 0.00                   & \second 0.58 $\pm$ 0.00                          & 0.21 $\pm$ 0.02        & 0.56 $\pm$ 0.01        & 0.56 $\pm$ 0.01        & 0.11 $\pm$ 0.01        \\
         & LAM (post-hoc) & \first \tabf{0.66 $\pm$ 0.00}                        & \second 0.60 $\pm$ 0.00                   & \second 0.58 $\pm$ 0.00                          & \third 0.50 $\pm$ 0.11        & 0.69 $\pm$ 0.07        & \second 0.81 $\pm$ 0.01        & 0.23 $\pm$ 0.01        \\ 
         & LAM (online)   & \first \tabf{0.66 $\pm$ 0.01}                        & \second 0.60 $\pm$ 0.00                   & \third 0.57 $\pm$ 0.01                          & \first \tabf{0.78 $\pm$ 0.04} & \first \tabf{0.85 $\pm$ 0.03} & \first \tabf{0.83 $\pm$ 0.01} & \first \tabf{0.01 $\pm$ 0.00} \\ 
        \bottomrule
    \end{tabular}
    \vspace{-1.0\baselineskip}
\end{table*}

\textit{FashionMNIST (ID) vs MNIST (\textsc{OoD}).} \cref{tab:ooddetection} shows that both PFE and post-hoc LAM have a similar predictive performance to the deterministic model.
This is not surprising, as both methods are initialized with the deterministic parameters, and then uncertainties are learned (PFE) or deduced (post-hoc LAM) with frozen weights.
The awareness of uncertainties during training, grants the online LAM slightly higher predictive performance.

PFE uses amortized inference to predict variances. This works reasonably within distribution but does not work well for \textsc{OoD} detection. This is because a neural network is trusted to extrapolate far away from the data distribution. \cref{tab:ooddetection} shows that MC dropout, LAM (post-hoc), and LAM (online) assign high uncertainty to observations outside the training distribution.
\cref{fig:ece} shows that both post-hoc and online LAM are near perfectly calibrated, which leads to very low \textsc{ECE}  measure (\cref{tab:ooddetection}).

\textit{Cifar10 (ID) vs SVHN (\textsc{OoD})} is a slightly harder setting.
\cref{tab:ooddetection} yields similar conclusions as before; Bayesian approaches such as MC dropout, LAM (post-hoc), and LAM (online) are much better at detecting \textsc{OoD} examples than neural amortized methods such as PFE\@.
Online LAM has a similar predictive performance to state-of-the-art while having better \textsc{ID} (lower \textsc{ECE}  and higher \textsc{AUSC}) and \textsc{OoD} (higher \textsc{AUROC}  and \textsc{AUPRC}) performance. 
\cref{fig:ece} shows the calibration plot for CIFAR10.
For this dataset, online LAM has a near-perfect calibration curve.
\cref{fig:auroc} shows the ROC curves for CIFAR10 and highlights that online LAM is better at distinguishing \textsc{ID} and \textsc{OoD} examples.

\begin{figure}
    \centering
    \begin{subfigure}{0.45\linewidth}
        \centering
        \includegraphics[width=\linewidth]{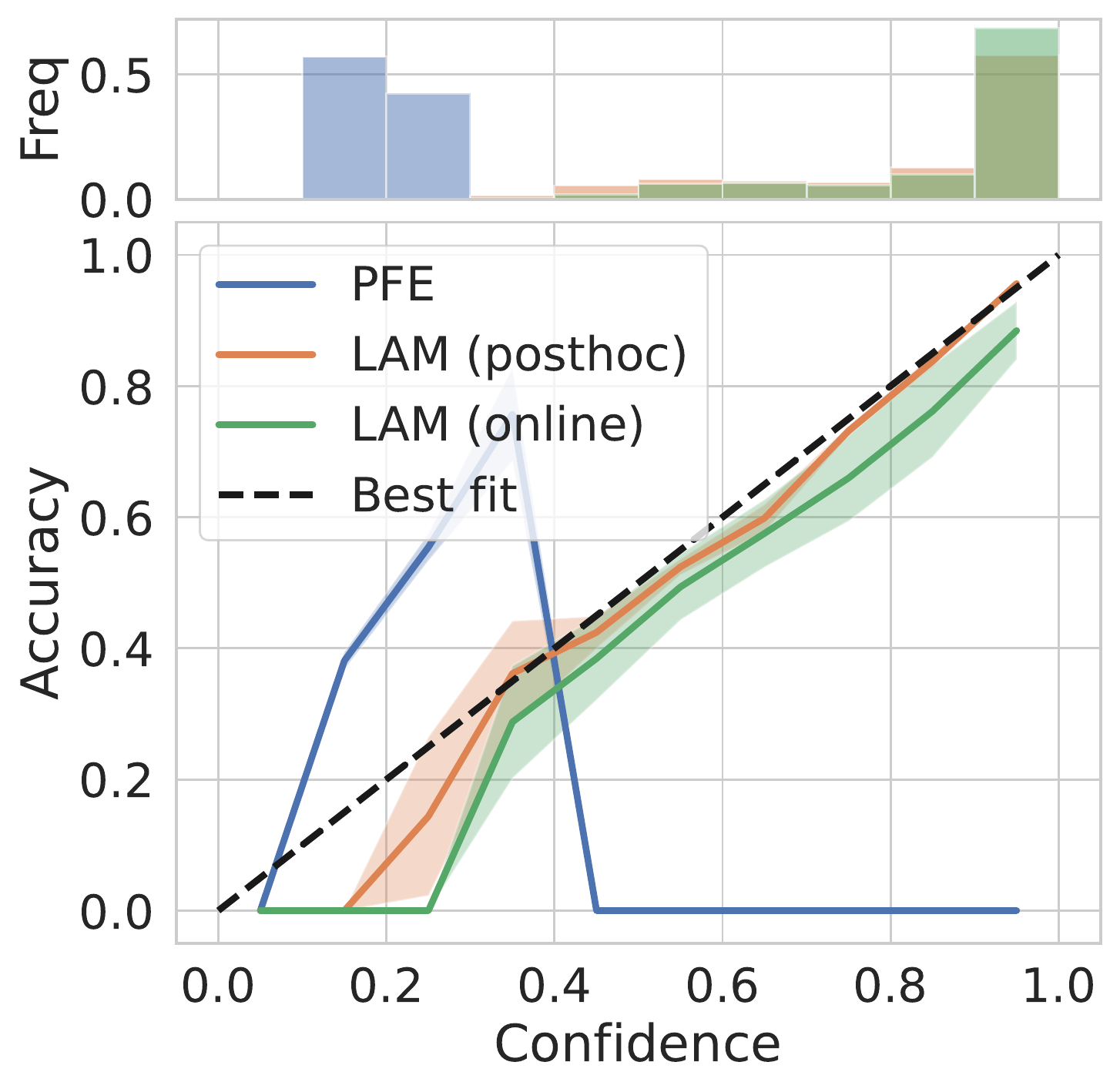}
        \caption{FashionMNIST}
    \end{subfigure} %
    \begin{subfigure}{0.45\linewidth}
        \centering
        \includegraphics[width=\linewidth]{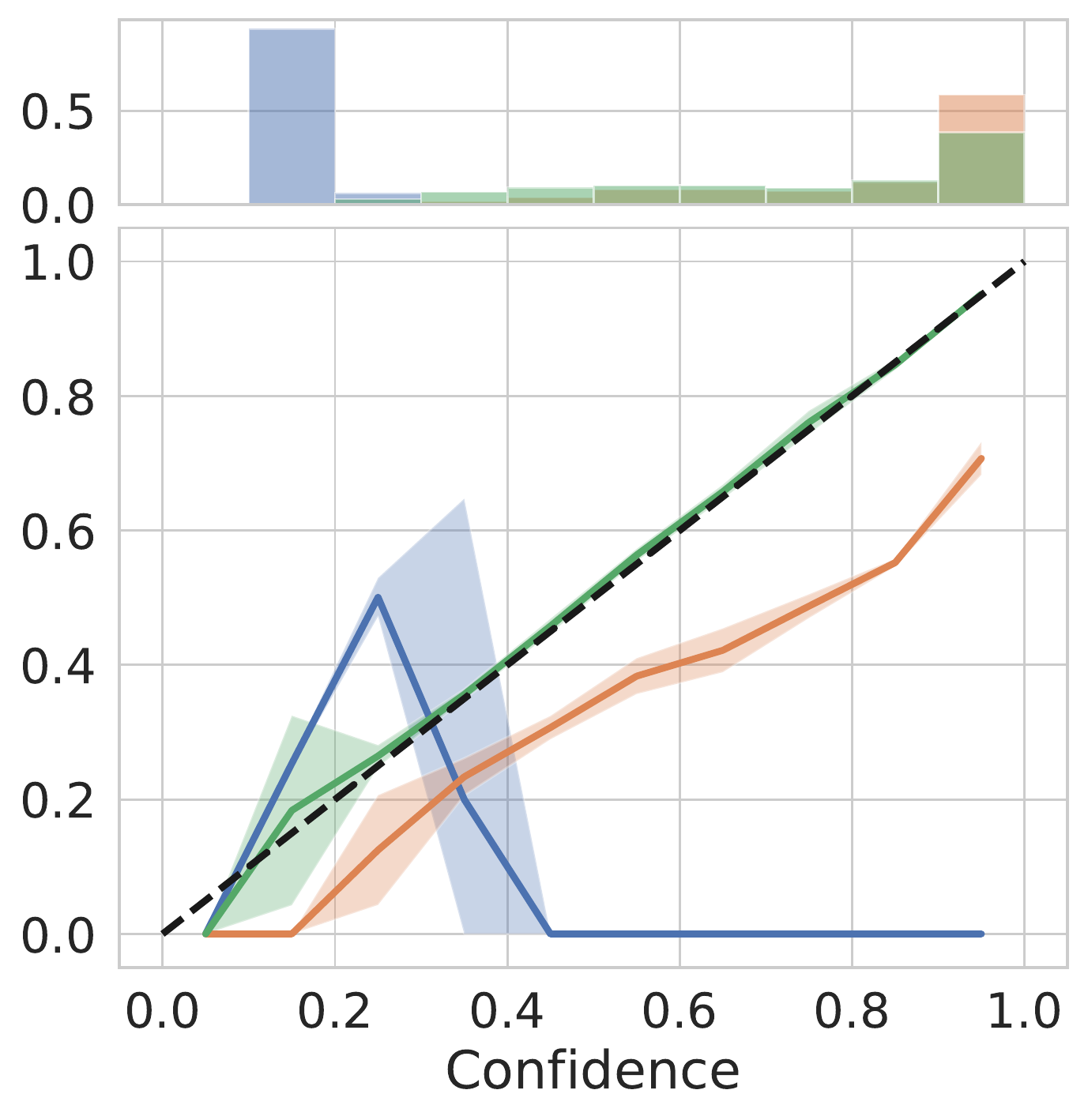}
        \caption{CIFAR10}
    \end{subfigure}
    \caption{\textbf{Calibration Curves.} LAM is near perfectly calibrated on FashionMNIST and CIFAR10.}
    \label{fig:auroc}
\end{figure}

\begin{figure}
    \centering
    \begin{subfigure}{0.49\linewidth}
        \includegraphics[width=\linewidth]{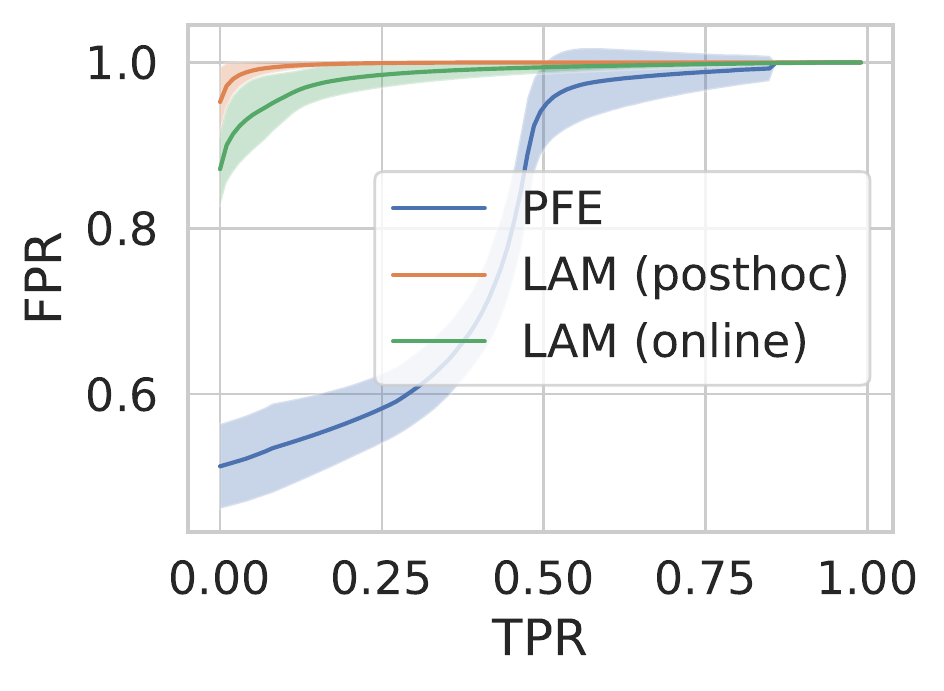}
        \caption{FashionMNIST}
    \end{subfigure} %
    \begin{subfigure}{0.49\linewidth}
        \includegraphics[width=\linewidth]{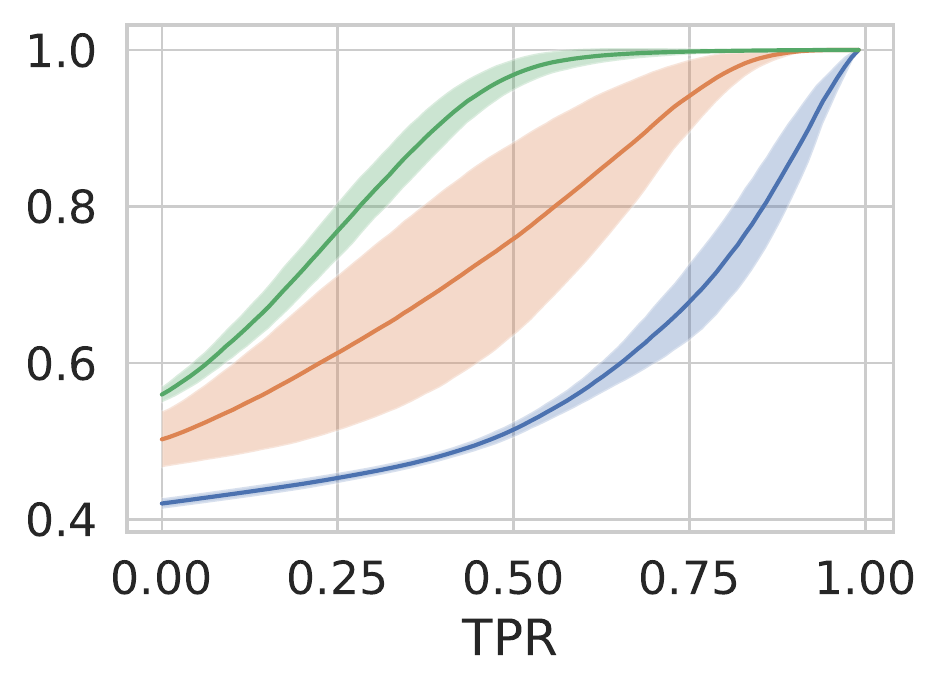}
        \caption{CIFAR10}
    \end{subfigure}
    \caption{\textbf{Receive Operator Curves} LAM assign high uncertainty to \textsc{OoD} observations.}
    \label{fig:ece}
\end{figure}

\begin{table*}[t]
    \centering
    \caption{\textbf{Open-set results.} LAM matches or outperforms existing methods in terms of predictive performance and produces state-of-the-art uncertainty quantification for challenging zero-shot metric learning datasets LFW and CUB200. Confidence intervals show one standard deviation computed across five runs.}
    \label{tab:cub200}
    \small
    \begin{tabular}{cl|lll|ll|l}
        \toprule
         &                & \multicolumn{3}{c|}{\textsc{Image retrieval}} & \multicolumn{2}{c|}{\textsc{OoD}} & \multicolumn{1}{c}{\textsc{ID}}                                                                            \\
         &                & mAP@1 $\uparrow$                              & mAP@5 $\uparrow$                  & mAP@10 $\uparrow$               & \textsc{AUROC}  $\uparrow$       & \textsc{AUPRC}  $\uparrow$       & \textsc{AUSC}  $\uparrow$        \\ \midrule  
        \multirow{7}{*}{\rotatebox{90}{CUB200}}
         & Deterministic  & \second 0.62 $\pm$ 0.01                               & \third 0.48 $\pm$ 0.01                   & \third 0.42 $\pm$ 0.01                 & ---                    & ---                                             \\
         & Deep Ensemble  & 0.21                                          & 0.11                              & 0.07                            & 0.47                   & 0.55                   & 0.21                   \\
         & PFE            & \second 0.62 $\pm$ 0.01                               & \second 0.5 $\pm$ 0.01                    & \second 0.43 $\pm$ 0.01                 & 0.44 $\pm$ 0.16        & 0.5 $\pm$ 0.08         & \third 0.61 $\pm$ 0.02        \\
         & HIB            & 0.33 $\pm$ 0.04                               & 0.19 $\pm$ 0.02                   & 0.14 $\pm$ 0.02                 & 0.54 $\pm$ 0.12        & 0.61 $\pm$ 0.1         & 0.31 $\pm$ 0.07        \\
         & MC dropout     & \third  0.61 $\pm$ 0.00                               & \third 0.48 $\pm$ 0.00                   & \third 0.42 $\pm$ 0.00                 & \second 0.73 $\pm$ 0.08        & \second 0.68 $\pm$ 0.07        & \second 0.63 $\pm$ 0.01        \\
         & LAM (post-hoc) & \first \tabf{0.65 $\pm$ 0.01}                        & \first \tabf{0.52 $\pm$ 0.01}            & \first \tabf{0.45 $\pm$ 0.01}          & \third 0.56 $\pm$ 0.16        & \third 0.61 $\pm$ 0.11        & \first \tabf{0.66 $\pm$ 0.03} \\
         & LAM (online)   & \third 0.61 $\pm$ 0.00                               & \third 0.48 $\pm$ 0.00                   & \third 0.42 $\pm$ 0.00                 & \first \tabf{0.80 $\pm$ 0.03} & \first \tabf{0.75 $\pm$ 0.03} & \second 0.63 $\pm$ 0.01        \\
        \midrule
        \multirow{6}{*}{\rotatebox{90}{LFW}}
         & Deterministic  & \second ${0.44 \pm 0.00}$                             & \second ${0.68 \pm 0.00}$                 & \second ${0.65 \pm 0.00}$               & ---                    & ---                    & ---                    \\ 
         & Deep Ensemble  & $0.36$                                        & $0.57$                            & $0.54$                          & \third $0.52$                 & \third  $0.64$                 & $0.33$                 \\ 
         & PFE            & \second ${0.44 \pm 0.00}$                             & \second ${0.68 \pm 0.00}$                 & \second ${0.65 \pm 0.00}$               & $0.03 \pm 0.02$        & $0.41 \pm 0.0$         & \second $0.49 \pm 0.01$        \\ 
         & MC dropout     & \third $0.42 \pm 0.00$                               & \third $0.65 \pm 0.01$                   & \third $0.63 \pm 0.01$                 & $0.03 \pm 0.01$        & $0.41 \pm 0.0$         & \third $0.46 \pm 0.01$        \\ 
         & LAM (post-hoc) & \second $0.44 \pm 0.01$                               & \second $0.68 \pm 0.01$                   & \second $0.65 \pm 0.00$                 & \second $0.65 \pm 0.14$        &  \second $0.72 \pm 0.11$        & $0.45 \pm 0.03$        \\ 
         & LAM (online)   & \first \tabf 0.46 $\pm$ 0.00                         & \first \tabf 0.71 $\pm$ 0.00             & \first \tabf 0.69 $\pm$ 0.00           & \first \tabf 0.71 $\pm$ 0.22  & \first \tabf 0.78 $\pm$ 0.17  & \first \tabf 0.50 $\pm$ 0.02  \\
        \bottomrule
    \end{tabular}
\end{table*}

\begin{figure*}
\centering
    \adjustbox{minipage=1.3em,valign=m}{\rotatebox{90}{PFE}}
    \begin{subfigure}[t]{\dimexpr.5\linewidth-2.0em\relax}
        \centering
        \includegraphics[width=.95\linewidth,valign=m]{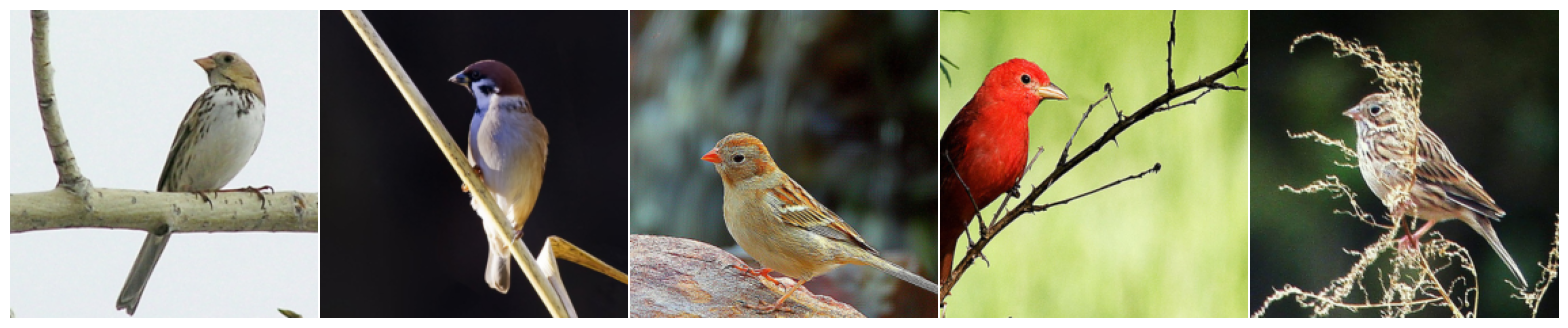}
    \end{subfigure}
    $\ldots$
    \adjustbox{minipage=1.3em,valign=m}{}
    \begin{subfigure}[t]{\dimexpr.5\linewidth-2.0em\relax}
        \centering
        \includegraphics[width=.95\linewidth,valign=m]{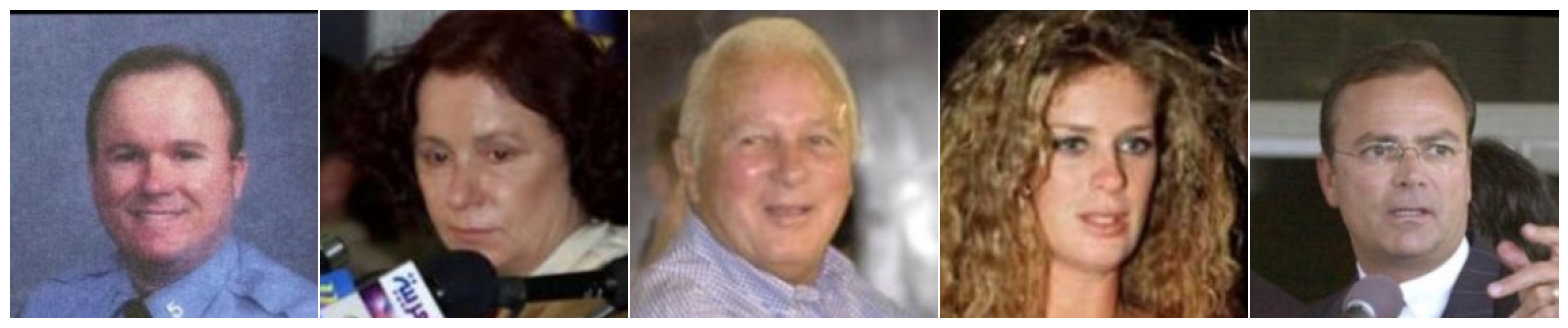}
    \end{subfigure}

    \adjustbox{minipage=1.3em,valign=m}{\rotatebox{90}{LAM (p)}}
    \begin{subfigure}[t]{\dimexpr.5\linewidth-2.0em\relax}
        \centering
        \includegraphics[width=.95\linewidth,valign=m]{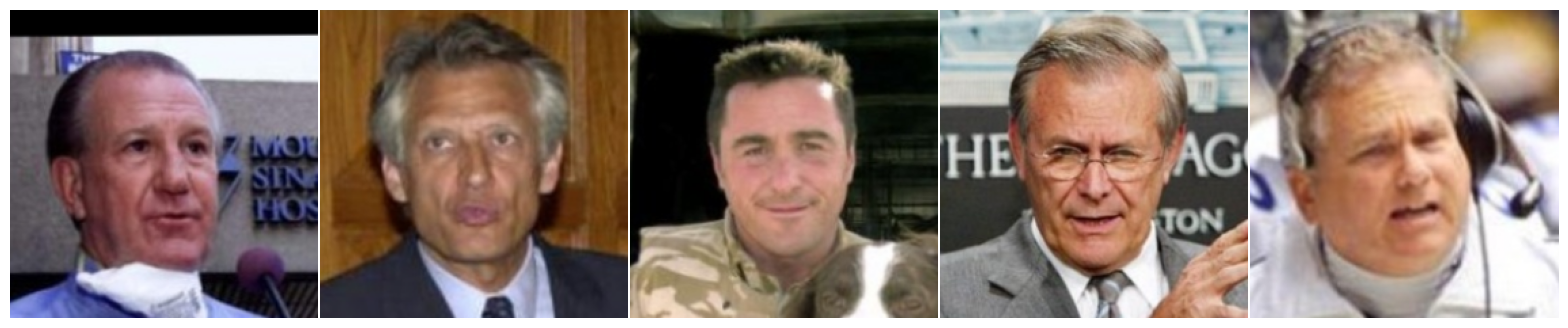}
    \end{subfigure}
    $\ldots$
    \adjustbox{minipage=1.3em,valign=m}{}
    \begin{subfigure}[t]{\dimexpr.5\linewidth-2.0em\relax}
        \centering
        \includegraphics[width=.95\linewidth,valign=m]{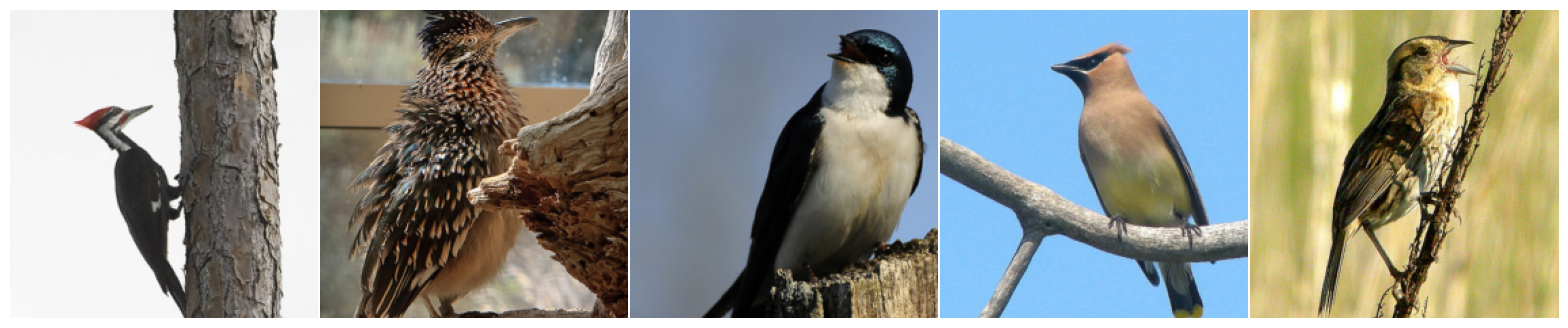}
    \end{subfigure}

    \adjustbox{minipage=1.3em,valign=m}{\rotatebox{90}{LAM (o)}}
    \begin{subfigure}[t]{\dimexpr.5\linewidth-2.0em\relax}
        \centering
        \includegraphics[width=.95\linewidth,valign=m]{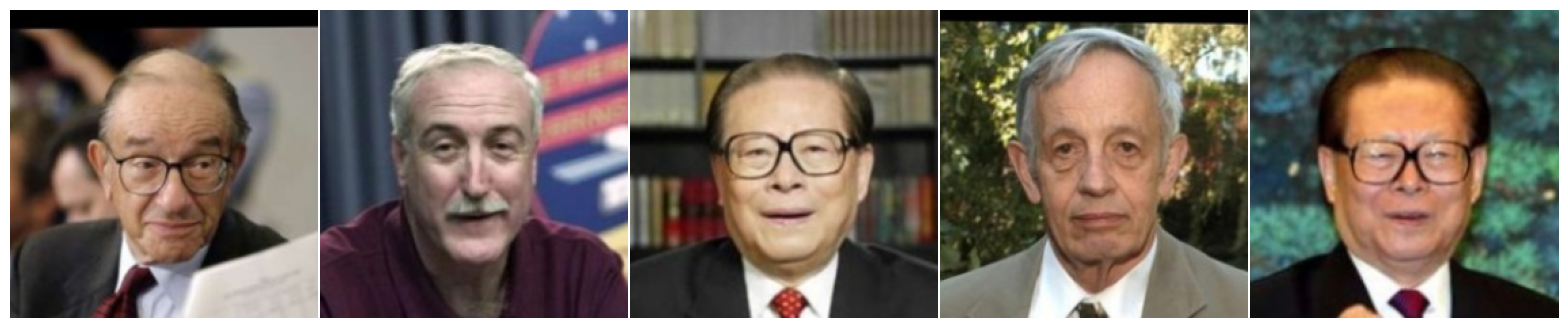}
    \end{subfigure}
    $\ldots$
    \adjustbox{minipage=1.3em,valign=m}{}
    \begin{subfigure}[t]{\dimexpr.5\linewidth-2.0em\relax}
        \centering
        \includegraphics[width=.95\linewidth,valign=m]{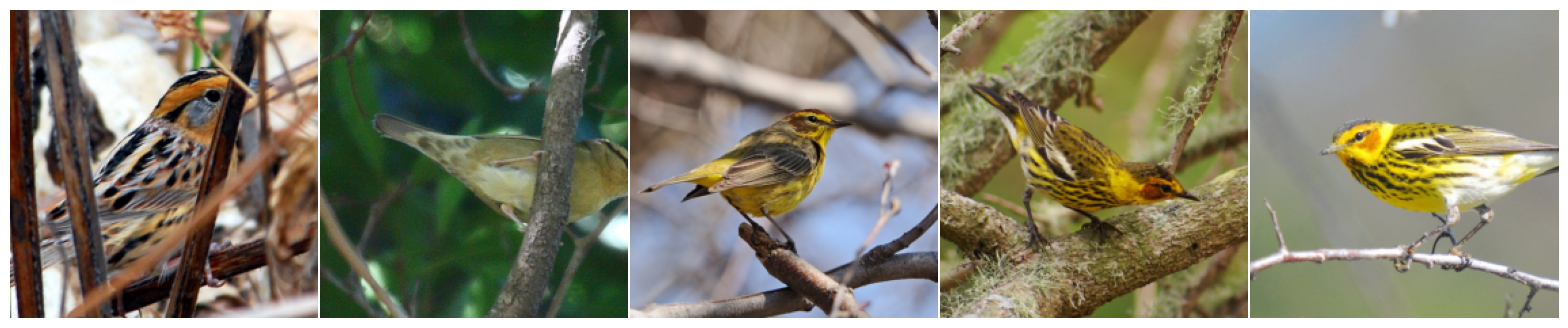}
    \end{subfigure}
    \hspace{1.2em}
    \begin{subfigure}[t]{\dimexpr1.0\linewidth-2.0em\relax}
        \begin{tikzpicture}[node distance=.5\linewidth-1.5em, thick]
            \node (1) {};
            \node (2) [right of=1] {};
            \draw[-latex, thick, shorten <=2pt, shorten >=2pt]    (2) to node [font=\small, preaction={fill, white}] {Lower uncertainty} (1);
        \end{tikzpicture}
        \hspace{1.2em}
        \begin{tikzpicture}[node distance=.5\linewidth-1.5em, thick]
            \node (1) {};
            \node (2) [right of=1] {};
            \draw[-latex, thick, shorten <=2pt, shorten >=2pt]    (1) to node [font=\small, preaction={fill, white}] {Higher uncertainty} (2);
        \end{tikzpicture}
    \end{subfigure}
    \caption{\textbf{Images with lowest and highest variance} for PFE, post-hoc LAM, and online LAM across LFW (ID) and CUB200 (\textsc{OoD}) datasets. LAM associates high uncertainty to \textsc{OoD} examples, whereas PFE predicts higher uncertainties for \textsc{ID} images. Shows the best-performing PFE and online LAM across five runs.}
    \label{fig:lfw_high_low_uq_images}
\end{figure*}
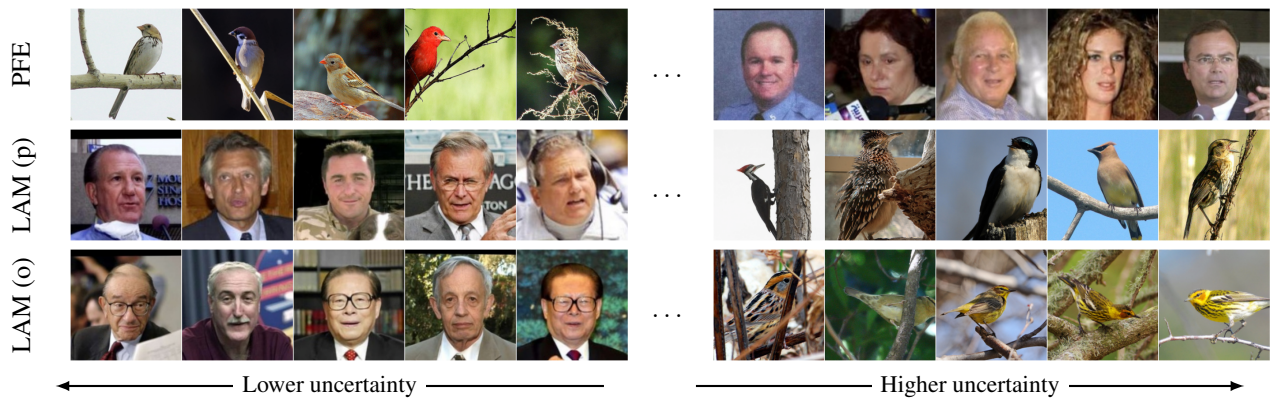

\textbf{Open-Set Retrieval.} A key advantage of metric learning methods is that they easily cope with a large number of classes and new classes can be added seamlessly.
We therefore evaluate LAM's performance on challenging open-set retrieval, where none of the classes in the test set are available during training.
We first test with CUB200 \citep{WahCUB_200_2011} as \textsc{ID} and CAR196 \citep{krause2013cars} as \textsc{OoD} similarly to \citet{warburg2020btl}, and second, test with LFW~\cite{LFWTech} as \textsc{ID} and CUB200 as \textsc{OoD}.
We use a ResNet50~\cite{he2016deep} with a GeM pooling layer \cite{Radenovic2017gem} and a $2048$ dimensional embedding and diagonal, last-layer LA~\cite{daxberger2021laplaceredux}.
\looseness=-1

\textit{CUB200 (ID) vs CARS196 (\textsc{OoD}).}
The CUB-200-2011 dataset~\citep{WahCUB_200_2011} has $200$ bird species captured from different perspectives and in different environments. We follow the zero-shot train/test split \cite{musgrave2020metric}. In this zero-shot setting, the trained models have not seen any of the bird species in the test set, and the learned features must generalize well across species. 
\cref{tab:cub200} shows that LAM matches or surpasses the predictive performance of all other methods. LAM (post-hoc) achieves state-of-the-art predictive performance, while LAM (online) matches the predictive performance of the deterministic trained model while achieving state-of-the-art \textsc{AUROC}  and \textsc{AUPRC}  for \textsc{OoD} detection.

\textit{LFW (ID) vs CUB200 (\textsc{OoD}).}
Face recognition is another challenging metric learning task with many applications in security and surveillance. The goal is to retrieve images of the same person as in the query image. 
\cref{tab:cub200} shows that online LAM outperforms existing methods both in terms of predictive performance and uncertainty quantification.
\cref{fig:lfw_high_low_uq_images} shows that PFE assigns higher uncertainty to images from the \textsc{ID} dataset (faces) than those from the \textsc{OoD} dataset (birds).
In contrast, both online and post-hoc LAM better associate high variance to \textsc{OoD} examples, while PFE predicts high variance to \textsc{ID} examples. Furthermore, online LAM seems to assign the highest variance to images in which the background is complex and thus camouflages the birds.

\subsection{Visual Place recognition}
Lastly, we evaluate LAM on the challenging task of visual place recognition, which has applications that span from human trafficking investigation~\cite{stylianou2019hotels50k} to the long-term operation of autonomous robots~\cite{davison2007monoslam}. Our focus is the latter, where the goal is to retrieve images taken within a radius of $25$ meters from a query image. The high number of unique places and varying visual appearance of each location -- including weather, dynamic, structural, view-point, seasonal, and day/night changes -- makes visual place recognition a very challenging metric learning problem. Reliable uncertainties and reliable out-of-distribution behavior are important to avoid incorrect loop-closure, which can deteriorate the autonomous robots' location estimate. We evaluate on MSLS~\citep{Warburg_2020_CVPR}, which is the largest and most diverse place recognition dataset currently available comprised of $1.6M$ images from $30$ cities spanning six continents.  We use the standard train/test split, training on $24$ cities and testing on six other cities. We use the same model as in open-set retrieval.

\begin{table*}
    \centering
     \caption{\textbf{Results on MSLS.} LAM yields state-of-the-art uncertainties and matches the predictive performance of deterministic trained models. We evaluate on both the validation set and the official challenge set \cite{Warburg_2020_CVPR}.}
     \label{tab:msls}
    \resizebox{\linewidth}{!}{%
    \begin{tabular}{l|lllll|l|lllll|l}
    \toprule
    & \multicolumn{6}{c}{\textbf{Validation Set}} & \multicolumn{6}{|c}{\textbf{Challenge Set}} \\
    & R@1$\uparrow$    &   R@5$\uparrow$    &   R@10$\uparrow$  &   M@5$\uparrow$    &   M@10$\uparrow$  &   AUSC$\uparrow$ &  R@1$\uparrow$    &   R@5$\uparrow$    &   R@10$\uparrow$  &   M@5$\uparrow$    &   M@10$\uparrow$  &   AUSC$\uparrow$    \\ \midrule
    Deterministic  & 
    \first \tabf 0.77 & \first \tabf 0.88 & \first \tabf 0.90 &\first \tabf 0.61 & \first \tabf 0.56 & --- & 
    \first \tabf 0.58 & \first \tabf 0.74 & \first \tabf 0.78 & \first \tabf 0.45 & \second 0.43 & ---    \\
    MC Dropout     & 
    \third 0.75 & \second 0.87 & \third 0.87 & \third 0.59 & \third 0.54 & \first \tabf 0.77 & 
    \third 0.55 & \second 0.71 & \second 0.76 & \second 0.43 & \third 0.41 & \third 0.57\\
    PFE            & 
    \first \tabf 0.77 & \first \tabf 0.88 & \first \tabf 0.90 & \first \tabf 0.61 & \first \tabf 0.56 & \third 0.73 & 
    \first \tabf 0.58 & \first \tabf 0.74 & \first \tabf 0.78 & \first \tabf 0.45 & \first \tabf 0.44 & \third 0.57 \\
    LAM (post-hoc) & 
    \second 0.76 & \third 0.86 & \second 0.89 & \second 0.60 & \second 0.55 & \second 0.74 & 
    \first \tabf 0.58 & \first \tabf 0.74 & \first \tabf 0.78 & \first \tabf 0.45 & \first \tabf 0.44 & \second 0.59 \\
    LAM (online)   & 
    \second 0.76 & \second 0.87 & \first \tabf 0.90 & \second 0.60 & \first \tabf 0.56 & \first \tabf 0.77 & 
    \second 0.57 & \first \tabf 0.74 & \first \tabf 0.78 & \first \tabf 0.45 & \second 0.43 & \first \tabf 0.63 \\
    \bottomrule
     \end{tabular}
     }
 \end{table*}

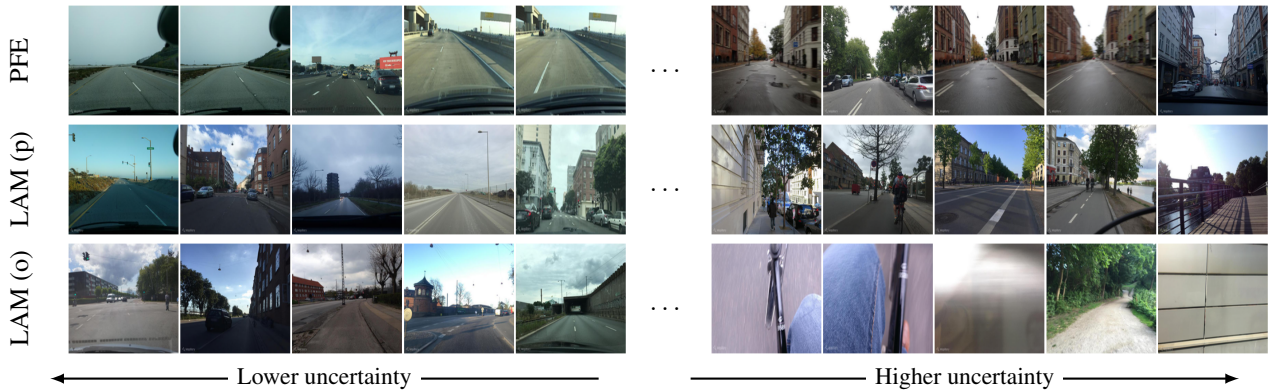
\begin{figure*}
\centering
    \adjustbox{minipage=1.3em,valign=m}{\rotatebox{90}{PFE}}
    \begin{subfigure}[t]{\dimexpr.5\linewidth-2.0em\relax}
        \centering
        \includegraphics[width=.95\linewidth,valign=m]{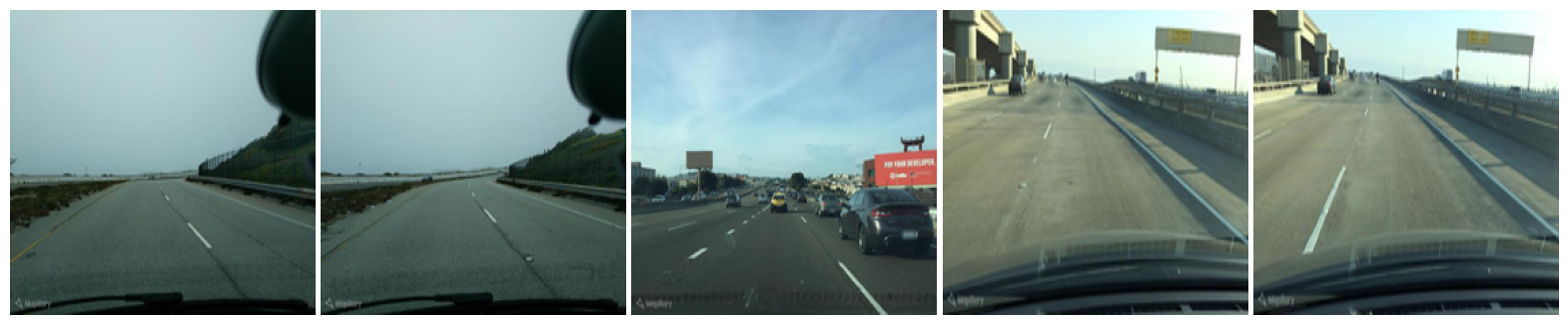}
    \end{subfigure}
    $\ldots$
    \adjustbox{minipage=1.3em,valign=m}{}
    \begin{subfigure}[t]{\dimexpr.5\linewidth-2.0em\relax}
        \centering
        \includegraphics[width=.95\linewidth,valign=m]{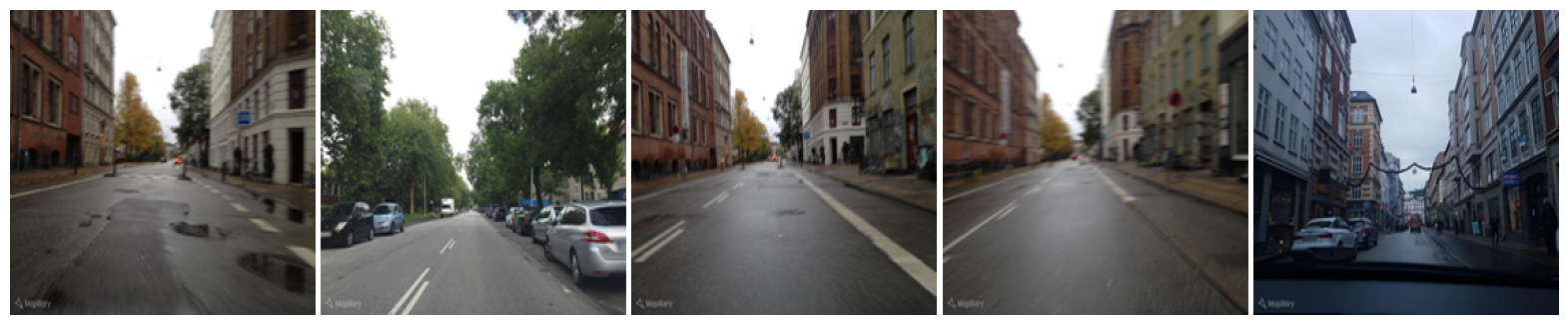}
    \end{subfigure}
    \adjustbox{minipage=1.3em,valign=m}{\rotatebox{90}{LAM (p)}}
    \begin{subfigure}[t]{\dimexpr.5\linewidth-2.0em\relax}
        \centering
        \includegraphics[width=.95\linewidth,valign=m]{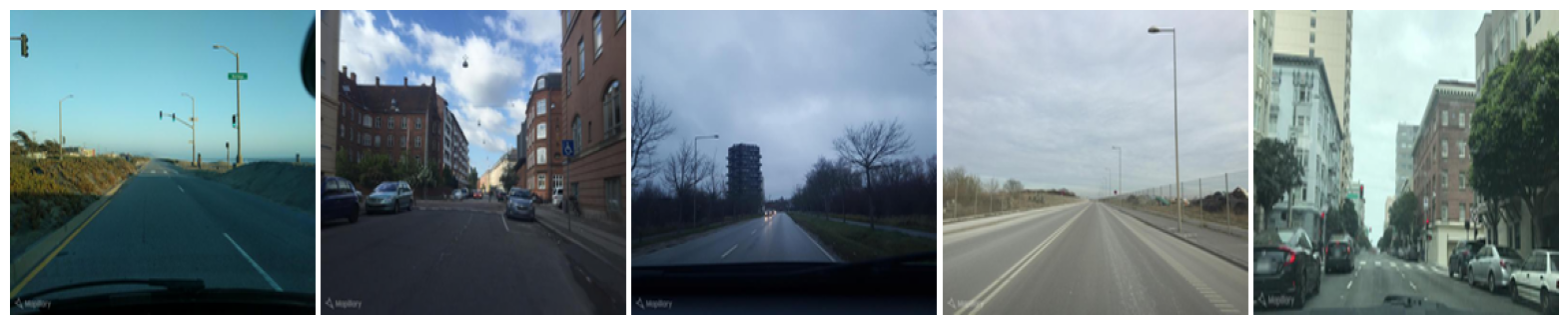}
    \end{subfigure}
    $\ldots$
    \adjustbox{minipage=1.3em,valign=m}{}
    \begin{subfigure}[t]{\dimexpr.5\linewidth-2.0em\relax}
        \centering
        \includegraphics[width=.95\linewidth,valign=m]{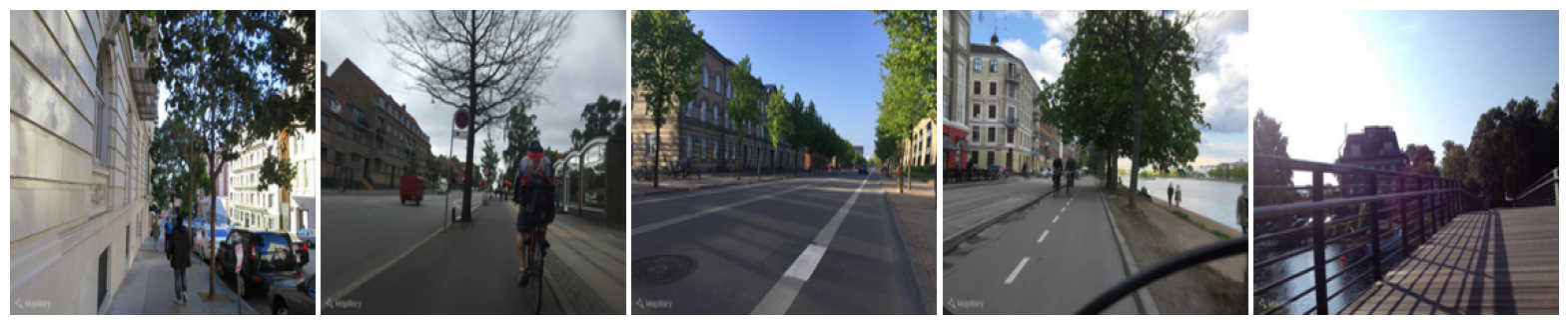}
    \end{subfigure}

    \adjustbox{minipage=1.3em,valign=m}{\rotatebox{90}{LAM (o)}}
    \begin{subfigure}[t]{\dimexpr.5\linewidth-2.0em\relax}
        \centering
        \includegraphics[width=.95\linewidth,valign=m]{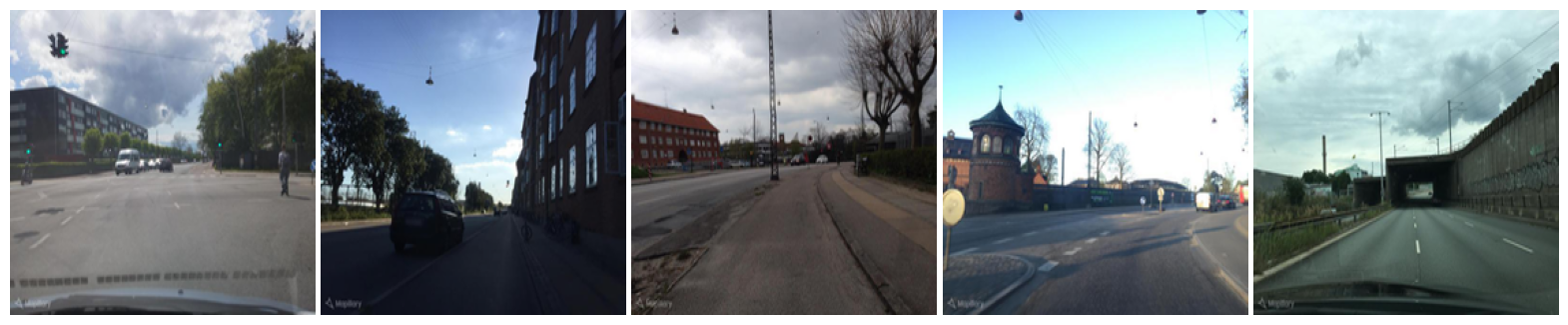}
    \end{subfigure}
    $\ldots$
    \adjustbox{minipage=1.3em,valign=m}{}
    \begin{subfigure}[t]{\dimexpr.5\linewidth-2.0em\relax}
        \centering
        \includegraphics[width=.95\linewidth,valign=m]{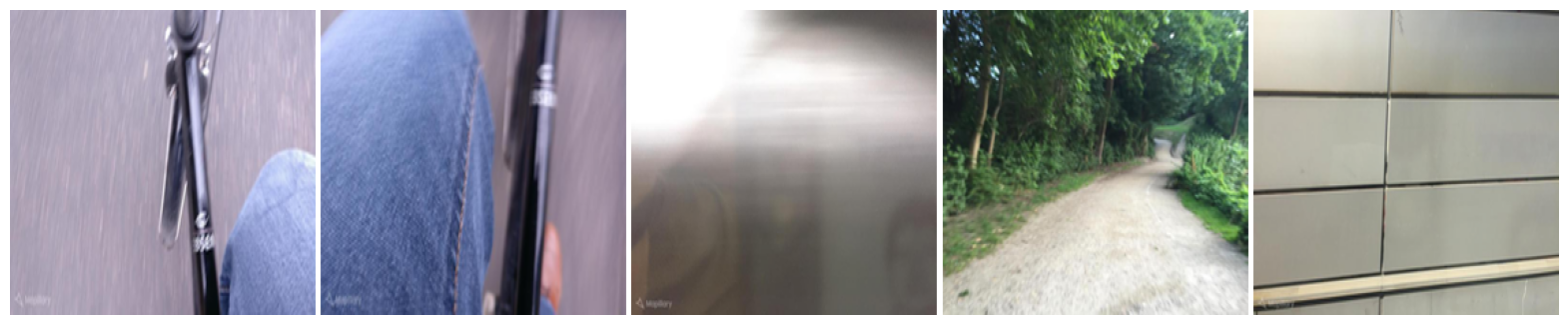}
    \end{subfigure}
    \hspace{1.2em}
    \begin{subfigure}[t]{\dimexpr1.0\linewidth-2.0em\relax}
        \begin{tikzpicture}[node distance=.5\linewidth-1.5em, thick]
            \node (1) {};
            \node (2) [right of=1] {};
            \draw[-latex, thick, shorten <=2pt, shorten >=2pt]    (2) to node [font=\small, preaction={fill, white}] {Lower uncertainty} (1);
        \end{tikzpicture}
        \hspace{1.2em}
        \begin{tikzpicture}[node distance=.5\linewidth-1.5em, thick]
            \node (1) {};
            \node (2) [right of=1] {};
            \draw[-latex, thick, shorten <=2pt, shorten >=2pt]    (1) to node [font=\small, preaction={fill, white}] {Higher uncertainty} (2);
        \end{tikzpicture}
    \end{subfigure}
    \vspace{-1mm}
    \caption{\textbf{Images with lowest and highest variance} for PFE, post-hoc LAM, and online LAM across MSLS validation set. LAM reliably associates high uncertainty to images that are blurry, are captured facing the pavement, or contain vegetation. These images do not contain features that are descriptive of a specific place, making them especially challenging to geographically locate. 
    }
    \vspace{-0.5\baselineskip}
    \label{fig:msls_qualitative}
\end{figure*}

\cref{tab:msls} shows that online LAM yields state-of-the-art uncertainties for visual place recognition measured with \textsc{AUSC}, while matching the predictive performance of the alternative probabilistic and deterministic methods on both the MSLS validation and the challenge set. \cref{fig:ausc_msls} shows the sparsification curves on the challenge set. Both online and posthoc LAM have monotonically increasing sparsification curves, implying that when we remove the most uncertain observations, the predictive performance increase. This illustrates that LAM produces reliable uncertainties for this challenging open-set retrieval task. \cref{fig:msls_qualitative} shows the queries associated with the highest and lowest uncertainty. LAM predicts high uncertainty to images with are blurry, captured facing into the pavement, or contain mostly vegetation. These images do not have features that are descriptive of a specific place, making them hard to geographically locate.

\begin{figure}
    \centering
    \includegraphics[width=0.9\linewidth]{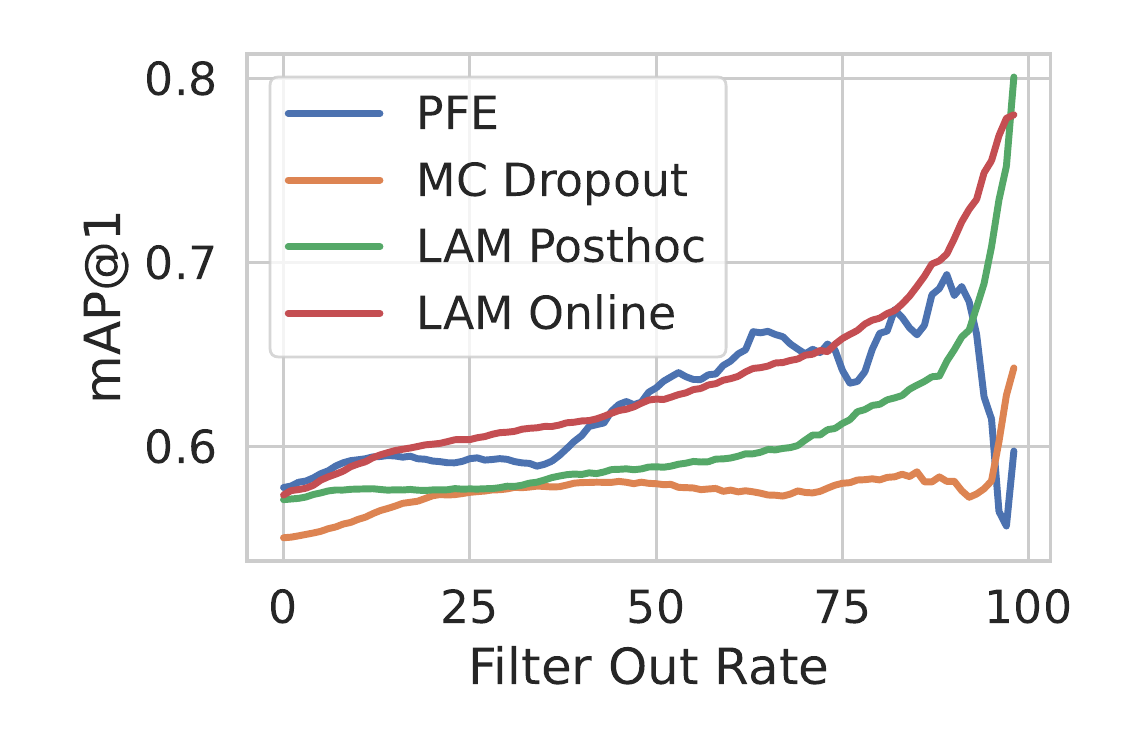}
    \vspace{-0.6cm}
    \caption{\textbf{Sparsification curve.} Online and posthoc LAM's sparsification curves monotonically increase, illustrating that they reliably associate higher uncertainty to harder observations. }
    \label{fig:ausc_msls}
\end{figure}

\textbf{Limitations.} Similar to other Bayesian methods, LAM relies on $n$ samples to obtain uncertainties. This makes inferences $n$ times slower.
Computing the Hessian at every step during online LAM also makes training time slower.
To combat long training times and high memory usage, we use a last-layer LAM and thus only estimate and sample for a weight posterior of the last layer.
The last-layer LAM training time is $3$ hours for online LAM vs $2.3$ hours for deterministic contrastive loss on LFW, and $30$ minutes vs $15$ minutes loss on CUB200 on an NVIDIA RTX A5000.

\section{Conclusion}

In this paper, we have introduced a Bayesian encoder for metric learning, the Laplacian Metric Learner (LAM), which uses the Laplace approximation.
We prove that the contrastive loss is indeed a valid unnormalized log-posterior, and develop three Hessian approximations, which ensures a positive definite covariance matrix. We propose a novel decomposition of the Generalized Gauss-Newton approximation that improves Hessian approximations of $\ell_2$-normalized networks.
Empirically, we demonstrate that LAM consistently produces well-calibrated uncertainties, reliably detects out-of-distribution examples, and achieves state-of-the-art predictive performance on both closed-set and challenging open-set image retrieval tasks.

\newpage

\textbf{Acknowledgement.} This work was supported by research grants (15334, 42062) from VILLUM FONDEN. This project received funding from the European Research Council (ERC) under the European Union's Horizon 2020 research and innovation programme (grant agreement 757360). The work was partly funded by the Novo Nordisk Foundation through the Center for Basic Machine Learning Research in Life Science (NNF20OC0062606). The authors acknowledge the Pioneer Centre for AI, DNRF grant P1.\looseness=-1

\bibliography{icml2023}

\begin{thebibliography}{53}
\providecommand{\natexlab}[1]{#1}
\providecommand{\url}[1]{\texttt{#1}}
\expandafter\ifx\csname urlstyle\endcsname\relax
  \providecommand{\doi}[1]{doi: #1}\else
  \providecommand{\doi}{doi: \begingroup \urlstyle{rm}\Url}\fi

\bibitem[Arandjelović et~al.(2016)Arandjelović, Gronat, Torii, Pajdla, and
  Sivic]{Arandjelovic2015netvlad}
Arandjelović, R., Gronat, P., Torii, A., Pajdla, T., and Sivic, J.
\newblock Netvlad: {CNN} architecture for weakly supervised place recognition.
\newblock In \emph{Proceedings of the IEEE Conference on Computer Vision and
  Pattern Recognition}, pp.\  5297--5307, 2016.

\bibitem[Blundell et~al.(2015)Blundell, Cornebise, Kavukcuoglu, and
  Wierstra]{blundell2015bayesbybackprop}
Blundell, C., Cornebise, J., Kavukcuoglu, K., and Wierstra, D.
\newblock Weight uncertainty in neural networks, 2015.

\bibitem[Chang et~al.(2020)Chang, Lan, Cheng, and Wei]{chang2020data}
Chang, J., Lan, Z., Cheng, C., and Wei, Y.
\newblock Data uncertainty learning in face recognition.
\newblock In \emph{Proceedings of the IEEE Conference on Computer Vision and
  Pattern Recognition}, pp.\  5710--5719, 2020.

\bibitem[Christensen et~al.(2022)Christensen, Warburg, Jia, and
  Belongie]{ebert2022narratives}
Christensen, P.~E., Warburg, F., Jia, M., and Belongie, S.
\newblock Searching for structure in unfalsifiable claims, 2022.

\bibitem[Chun et~al.(2021)Chun, Oh, De~Rezende, Kalantidis, and
  Larlus]{chun2021crossmodal}
Chun, S., Oh, S.~J., De~Rezende, R.~S., Kalantidis, Y., and Larlus, D.
\newblock Probabilistic embeddings for cross-modal retrieval.
\newblock In \emph{Proceedings of the IEEE Conference on Computer Vision and
  Pattern Recognition}, pp.\  8415--8424, 2021.

\bibitem[Dangel et~al.(2020)Dangel, Kunstner, and Hennig]{dangel2020backpack}
Dangel, F., Kunstner, F., and Hennig, P.
\newblock Back{PACK}: Packing more into backprop.
\newblock In \emph{International Conference on Learning Representations
  (ICLR)}, 2020.

\bibitem[Davison et~al.(2007)Davison, Reid, Molton, and
  Stasse]{davison2007monoslam}
Davison, A.~J., Reid, I.~D., Molton, N.~D., and Stasse, O.
\newblock Monoslam: Real-time single camera slam.
\newblock \emph{IEEE Transactions on Pattern Analysis and Machine
  Intelligence}, 29\penalty0 (6):\penalty0 1052--1067, 2007.
\newblock \doi{10.1109/TPAMI.2007.1049}.

\bibitem[Daxberger et~al.(2021)Daxberger, Kristiadi, Immer, Eschenhagen, Bauer,
  and Hennig]{daxberger2021laplaceredux}
Daxberger, E., Kristiadi, A., Immer, A., Eschenhagen, R., Bauer, M., and
  Hennig, P.
\newblock Laplace redux--effortless {B}ayesian deep learning.
\newblock In \emph{{N}eur{IPS}}, 2021.

\bibitem[Denker \& LeCun(1990)Denker and LeCun]{denker1990transforming}
Denker, J. and LeCun, Y.
\newblock Transforming neural-net output levels to probability distributions.
\newblock \emph{Advances in Neural Information Processing Systems}, 1990.

\bibitem[Detlefsen et~al.(2019)Detlefsen, J{\o}rgensen, and
  Hauberg]{detlefsen2019reliable}
Detlefsen, N.~S., J{\o}rgensen, M., and Hauberg, S.
\newblock Reliable training and estimation of variance networks.
\newblock In \emph{Advances in Neural Information Processing Systems}, 2019.

\bibitem[Detlefsen et~al.(2021)Detlefsen, Pouplin, Feldager, Geng, Kalatzis,
  Hauschultz, Duque, Warburg, Miani, and Hauberg]{software:stochman}
Detlefsen, N.~S., Pouplin, A., Feldager, C.~W., Geng, C., Kalatzis, D.,
  Hauschultz, H., Duque, M.~G., Warburg, F., Miani, M., and Hauberg, S.
\newblock Stoch{M}an, 2021.
\newblock URL \url{https://github.com/MachineLearningLifeScience/stochman/}.

\bibitem[Foresee \& Hagan(1997)Foresee and Hagan]{foresee1997ggn}
Foresee, F.~D. and Hagan, M.~T.
\newblock Gauss-{N}ewton approximation to {B}ayesian learning.
\newblock In \emph{Proceedings of International Conference on Neural Networks
  (ICNN'97)}, volume~3, pp.\  1930--1935. IEEE, 1997.

\bibitem[Gal \& Ghahramani(2016)Gal and Ghahramani]{gal2015mcdropout}
Gal, Y. and Ghahramani, Z.
\newblock Dropout as a {B}ayesian approximation: Representing model uncertainty
  in deep learning.
\newblock In \emph{international conference on machine learning}, pp.\
  1050--1059. PMLR, 2016.

\bibitem[Gershman \& Goodman(2014)Gershman and Goodman]{gershman2014amortized}
Gershman, S. and Goodman, N.
\newblock Amortized inference in probabilistic reasoning.
\newblock In \emph{Proceedings of the annual meeting of the cognitive science
  society}, volume~36, 2014.

\bibitem[Hadsell et~al.(2006)Hadsell, Chopra, and
  LeCun]{hadsell2006dimensionality}
Hadsell, R., Chopra, S., and LeCun, Y.
\newblock Dimensionality reduction by learning an invariant mapping.
\newblock In \emph{2006 IEEE Computer Society Conference on Computer Vision and
  Pattern Recognition (CVPR'06)}, volume~2, pp.\  1735--1742. IEEE, 2006.

\bibitem[He et~al.(2016)He, Zhang, Ren, and Sun]{he2016deep}
He, K., Zhang, X., Ren, S., and Sun, J.
\newblock Deep residual learning for image recognition.
\newblock In \emph{Proceedings of the IEEE conference on computer vision and
  pattern recognition}, pp.\  770--778, 2016.

\bibitem[Huang et~al.(2007)Huang, Ramesh, Berg, and Learned-Miller]{LFWTech}
Huang, G.~B., Ramesh, M., Berg, T., and Learned-Miller, E.
\newblock Labeled faces in the wild: A database for studying face recognition
  in unconstrained environments.
\newblock Technical Report 07-49, University of Massachusetts, Amherst, October
  2007.

\bibitem[Immer et~al.(2021)Immer, Korzepa, and
  Bauer]{immer2020bnnlocallocalization}
Immer, A., Korzepa, M., and Bauer, M.
\newblock Improving predictions of {B}ayesian neural nets via local
  linearization.
\newblock In \emph{International Conference on Artificial Intelligence and
  Statistics}, pp.\  703--711. PMLR, 2021.

\bibitem[Kenney \& Keeping(1951)Kenney and Keeping]{kenney1951mathematics}
Kenney, J.~F. and Keeping, E.
\newblock Mathematics of statistics, vol. ii, 1951.

\bibitem[Kingma \& Welling(2013)Kingma and Welling]{kingma2013vae}
Kingma, D.~P. and Welling, M.
\newblock Auto-encoding variational {B}ayes.
\newblock \emph{arXiv preprint arXiv:1312.6114}, 2013.

\bibitem[Krause et~al.(2013)Krause, Stark, Deng, and Fei-Fei]{krause2013cars}
Krause, J., Stark, M., Deng, J., and Fei-Fei, L.
\newblock 3d object representations for fine-grained categorization.
\newblock In \emph{2013 IEEE International Conference on Computer Vision
  Workshops}, pp.\  554--561, 2013.

\bibitem[Krizhevsky(2009)]{Krizhevsky09learningmultiple}
Krizhevsky, A.
\newblock Learning multiple layers of features from tiny images, 2009.

\bibitem[Lakshminarayanan et~al.(2017)Lakshminarayanan, Pritzel, and
  Blundell]{Lakshminarayanan2016deepensembles}
Lakshminarayanan, B., Pritzel, A., and Blundell, C.
\newblock Simple and scalable predictive uncertainty estimation using deep
  ensembles.
\newblock \emph{Advances in Neural Information Processing Systems}, 30, 2017.

\bibitem[Laplace(1774)]{laplace1774memoire}
Laplace, P.~S.
\newblock M{\'e}moire sur la probabilit{\'e} des causes par les
  {\'e}v{\'e}nements.
\newblock \emph{M{\'e}moire de l'Acad{\'e}mie Royale des Sciences}, 1774.

\bibitem[LeCun et~al.(1989)LeCun, Denker, and Solla]{lecun1989optimal}
LeCun, Y., Denker, J., and Solla, S.
\newblock Optimal brain damage.
\newblock \emph{Advances in Neural Information Processing Systems}, 1989.

\bibitem[LeCun et~al.(1998)LeCun, Bottou, Bengio, and Haffner]{lecun1998mnist}
LeCun, Y., Bottou, L., Bengio, Y., and Haffner, P.
\newblock Gradient-based learning applied to document recognition.
\newblock \emph{Proceedings of the IEEE}, 86\penalty0 (11):\penalty0
  2278--2324, 1998.

\bibitem[MacKay(1992)]{mackay1992laplace}
MacKay, D. J.~C.
\newblock A practical {B}ayesian framework for backpropagation networks.
\newblock \emph{Neural Computation}, 4\penalty0 (3):\penalty0 448--472, 1992.

\bibitem[Maddox et~al.(2019)Maddox, Izmailov, Garipov, Vetrov, and
  Wilson]{Maddox2019swag}
Maddox, W.~J., Izmailov, P., Garipov, T., Vetrov, D.~P., and Wilson, A.~G.
\newblock A simple baseline for {B}ayesian uncertainty in deep learning.
\newblock \emph{Advances in Neural Information Processing Systems}, 32, 2019.

\bibitem[Miani et~al.(2022)Miani, Warburg, Moreno-Mu{\~n}oz, Detlefsen, and
  Hauberg]{miani_2022_neurips}
Miani, M., Warburg, F., Moreno-Mu{\~n}oz, P., Detlefsen, N.~S., and Hauberg, S.
\newblock Laplacian autoencoders for learning stochastic representations.
\newblock In \emph{Advances in Neural Information Processing Systems}, 2022.

\bibitem[Musgrave et~al.(2020)Musgrave, Belongie, and Lim]{musgrave2020metric}
Musgrave, K., Belongie, S., and Lim, S.-N.
\newblock A metric learning reality check.
\newblock In \emph{European Conference on Computer Vision}, pp.\  681--699.
  Springer, 2020.

\bibitem[Nalisnick et~al.(2018)Nalisnick, Matsukawa, Teh, Gorur, and
  Lakshminarayanan]{nilisnick2019knownothing}
Nalisnick, E., Matsukawa, A., Teh, Y.~W., Gorur, D., and Lakshminarayanan, B.
\newblock Do deep generative models know what they don't know?
\newblock \emph{arXiv preprint arXiv:1810.09136}, 2018.

\bibitem[Netzer et~al.(2011)Netzer, Wang, Coates, Bissacco, Wu, and
  Ng]{netzer2011reading}
Netzer, Y., Wang, T., Coates, A., Bissacco, A., Wu, B., and Ng, A.~Y.
\newblock Reading digits in natural images with unsupervised feature learning,
  2011.

\bibitem[Oh et~al.(2018)Oh, Murphy, Pan, Roth, Schroff, and
  Gallagher]{oh2018modeling}
Oh, S.~J., Murphy, K., Pan, J., Roth, J., Schroff, F., and Gallagher, A.
\newblock Modeling uncertainty with hedged instance embedding.
\newblock \emph{arXiv preprint arXiv:1810.00319}, 2018.

\bibitem[Paszke et~al.(2017)Paszke, Gross, Chintala, Chanan, Yang, DeVito, Lin,
  Desmaison, Antiga, and Lerer]{paszke2017automatic}
Paszke, A., Gross, S., Chintala, S., Chanan, G., Yang, E., DeVito, Z., Lin, Z.,
  Desmaison, A., Antiga, L., and Lerer, A.
\newblock Automatic differentiation in {PyTorch}.
\newblock In \emph{NIPS-W}, 2017.

\bibitem[Radenović et~al.(2018)Radenović, Tolias, and Chum]{Radenovic2017gem}
Radenović, F., Tolias, G., and Chum, O.
\newblock Fine-tuning {CNN} image retrieval with no human annotation.
\newblock \emph{IEEE transactions on pattern analysis and machine
  intelligence}, 41\penalty0 (7):\penalty0 1655--1668, 2018.

\bibitem[Rezende et~al.(2014)Rezende, Mohamed, and
  Wierstra]{rezende2014stochastic}
Rezende, D.~J., Mohamed, S., and Wierstra, D.
\newblock Stochastic backpropagation and approximate inference in deep
  generative models.
\newblock In \emph{International conference on machine learning}, pp.\
  1278--1286. PMLR, 2014.

\bibitem[Schroff et~al.(2015)Schroff, Kalenichenko, and
  Philbin]{Schroff2015Facenet}
Schroff, F., Kalenichenko, D., and Philbin, J.
\newblock Facenet: A unified embedding for face recognition and clustering.
\newblock In \emph{Proceedings of the IEEE Conference on Computer Vision and
  Pattern Recognition}, pp.\  815--823, 2015.

\bibitem[Shi \& Jain(2019)Shi and Jain]{shi2019probabilistic}
Shi, Y. and Jain, A.~K.
\newblock Probabilistic face embeddings.
\newblock In \emph{Proceedings of the IEEE/CVF International Conference on
  Computer Vision}, pp.\  6902--6911, 2019.

\bibitem[Song \& Soleymani(2019)Song and Soleymani]{song2019crossmodal}
Song, Y. and Soleymani, M.
\newblock Polysemous visual-semantic embedding for cross-modal retrieval.
\newblock In \emph{Proceedings of the IEEE/CVF Conference on Computer Vision
  and Pattern Recognition}, pp.\  1979--1988, 2019.

\bibitem[Sra(2012)]{sra2011vonmises}
Sra, S.
\newblock A short note on parameter approximation for von {M}ises-{F}isher
  distributions: and a fast implementation of i s (x).
\newblock \emph{Computational Statistics}, 27\penalty0 (1):\penalty0 177--190,
  2012.

\bibitem[Stylianou et~al.(2019)Stylianou, Xuan, Shende, Brandt, Souvenir, and
  Pless]{stylianou2019hotels50k}
Stylianou, A., Xuan, H., Shende, M., Brandt, J., Souvenir, R., and Pless, R.
\newblock Hotels-50k: {A} global hotel recognition dataset.
\newblock \emph{CoRR}, 2019.

\bibitem[Sun et~al.(2020)Sun, Zhao, Chen, Schroff, Adam, and
  Liu]{sun2019prob4pose}
Sun, J.~J., Zhao, J., Chen, L.-C., Schroff, F., Adam, H., and Liu, T.
\newblock View-invariant probabilistic embedding for human pose.
\newblock In \emph{European Conference on Computer Vision}, pp.\  53--70.
  Springer, 2020.

\bibitem[Taha et~al.(2019{\natexlab{a}})Taha, Chen, Misu, Shrivastava, and
  Davis]{taha2019ensembles}
Taha, A., Chen, Y.-T., Misu, T., Shrivastava, A., and Davis, L.
\newblock Unsupervised data uncertainty learning in visual retrieval systems.
\newblock \emph{arXiv preprint arXiv:1902.02586}, 2019{\natexlab{a}}.

\bibitem[Taha et~al.(2019{\natexlab{b}})Taha, Chen, Yang, Misu, and
  Davis]{taha2019dropout}
Taha, A., Chen, Y.-T., Yang, X., Misu, T., and Davis, L.
\newblock Exploring uncertainty in conditional multi-modal retrieval systems.
\newblock \emph{arXiv preprint arXiv:1901.07702}, 2019{\natexlab{b}}.

\bibitem[{The Week}(2018)]{theweek2018facial}
{The Week}.
\newblock Is facial recognition technology racist?, 2018.
\newblock URL
  \url{https://www.theweek.co.uk/95383/is-facial-recognition-racist}.

\bibitem[Vilnis \& McCallum(2014)Vilnis and McCallum]{vilnis2014word}
Vilnis, L. and McCallum, A.
\newblock Word representations via {G}aussian embedding.
\newblock \emph{arXiv preprint arXiv:1412.6623}, 2014.

\bibitem[Wah et~al.(2011)Wah, Branson, Welinder, Perona, and
  Belongie]{WahCUB_200_2011}
Wah, C., Branson, S., Welinder, P., Perona, P., and Belongie, S.
\newblock Caltech-{UCSD} birds 200, 2011.

\bibitem[Wang \& Deng(2021)Wang and Deng]{Wang2018deepface}
Wang, M. and Deng, W.
\newblock Deep face recognition: A survey.
\newblock \emph{Neurocomputing}, 429:\penalty0 215--244, March 2021.

\bibitem[Warburg et~al.(2020)Warburg, Hauberg, Lopez-Antequera, Gargallo,
  Kuang, and Civera]{Warburg_2020_CVPR}
Warburg, F., Hauberg, S., Lopez-Antequera, M., Gargallo, P., Kuang, Y., and
  Civera, J.
\newblock Mapillary street-level sequences: A dataset for lifelong place
  recognition.
\newblock In \emph{Proceedings of the IEEE/CVF Conference on Computer Vision
  and Pattern Recognition (CVPR)}, June 2020.

\bibitem[Warburg et~al.(2021)Warburg, J{\o}rgensen, Civera, and
  Hauberg]{warburg2020btl}
Warburg, F., J{\o}rgensen, M., Civera, J., and Hauberg, S.
\newblock Bayesian triplet loss: Uncertainty quantification in image retrieval.
\newblock In \emph{Proceedings of the IEEE/CVF International Conference on
  Computer Vision}, pp.\  12158--12168, 2021.

\bibitem[Wilber et~al.(2015)Wilber, Kwak, Kriegman, and
  Belongie]{wilber2015snack}
Wilber, M., Kwak, I.~S., Kriegman, D., and Belongie, S.
\newblock Learning concept embeddings with combined human-machine expertise.
\newblock In \emph{Proceedings of the IEEE International Conference on Computer
  Vision}, pp.\  981--989, 2015.

\bibitem[Xiao et~al.(2017)Xiao, Rasul, and Vollgraf]{xiao2017fashionmnist}
Xiao, H., Rasul, K., and Vollgraf, R.
\newblock Fashion-{MNIST}: a novel image dataset for benchmarking machine
  learning algorithms, 2017.

\bibitem[Xu et~al.(2021)Xu, Zhang, Li, Du, ichi Kawarabayashi, and
  Jegelka]{xu2021neural}
Xu, K., Zhang, M., Li, J., Du, S.~S., ichi Kawarabayashi, K., and Jegelka, S.
\newblock How neural networks extrapolate: From feedforward to graph neural
  networks, 2021.

\end{thebibliography}
\bibliographystyle{icml2023}

\newpage
\appendix
\onecolumn

\section{Notation}

Let $\mathcal{X}$ be the data space, $\mathcal{Z}$ be the latent space, $\Theta$ be the parameter space. For a given neural network architecture $f$ and a parameter vector $\theta\in\Theta$, $f_\theta$ is a function
\begin{equation}
f_\theta: \mathcal{X} \rightarrow \mathcal{Z}
\end{equation}
The discrete set $\mathcal{C}=\{c_i\}_{i=1,\dots,C}$ of classes induces a partition of the points in the data space. A dataset is a collection of \emph{data point}-class pairs $(x,c)\in\mathcal{X}\times\mathcal{C}$.
\begin{equation}
\mathcal{D} = \{(x_i,c_i)\}_{i=1,\dots,D}
\end{equation}
and it is itself partitioned by the classes into $C$ sets
\begin{equation}
    \mathcal{D}_c
    =
    \{(x',c')\in\mathcal{D} | c' = c\}
    \subset \mathcal{D}
    \qquad
    \forall c\in\mathcal{C}.
\end{equation}
We assume these sets always to contain at least one element $\mathcal{D}_c\not=\emptyset$ and we use the notation $|\mathcal{D}_c|$ to refer to their cardinalities. In the metric learning setting, instead of enforcing properties of a single data point, the goal is to enforce relations between data points. Thus we will often make use of pairs $\pp_{ij}=((x_i,c_i),(x_j,c_j))\in\mathcal{D}^2$, specifically we use the terms
\begin{align}
    \pp_{ij} & \text{ is \emph{positive} pair if } c_i=c_j \\
    \pp_{ij} & \text{ is \emph{negative} pair if } c_i\not=c_j.
\end{align}

To ease the later notation we will consider the trivial pairs composed by a data point with itself. We then define the positive set and negative set, and compute their cardinalities, respectively
\begin{align}
    \mathcal{D}^2_{\text{pos}}
    &
    := \{\pp\in\mathcal{D}^2 \text{ such that }p\text{ is positive}\}
    \qquad\qquad
    | \mathcal{D}^2_{\text{pos}}| =
    \sum_{c\in\mathcal{C}} |\mathcal{D}_c|^2 
    \\
    \mathcal{D}^2_{\text{neg}}
    &
    := 
    \{\pp\in\mathcal{D}^2 \text{ such that }p\text{ is negative}\}
    \qquad\quad\,\,
    | \mathcal{D}^2_{\text{neg}}|
    =
    \sum_{\substack{c_i,c_j\in\mathcal{C}^2 \\ \text{s.t. } c_i \not= c_j}}
    |\mathcal{D}_{c_i}| |\mathcal{D}_{c_j}|
\end{align}

A common trick in metric learning is to introduce a \emph{margin} $m\in\mathbb{R}^+$. 
The margin induces a further partition of the $\mathcal{D}^2_{\text{neg}}$ set into pairs with close or far embeddings. 
For any pair $\pp_{ij}~=~((x_i,c_i),(x_j,c_j))$ we describe the pair as being
\begin{align}
    \text{\emph{inside} the margin if } \|f_\theta(x_i)-f_\theta(x_j)\| \leq m \\
    \text{\emph{outside} the margin if } \|f_\theta(x_i)-f_\theta(x_j)\| > m
\end{align}
irrespective of the classes $c_i$ and $c_j$.  We can make use of this definition 
to consider only the pairs that have a non-zero loss. 
From this logic, it is convenient to define the set $\mathcal{D}^2_{\text{neg inside}}:=\{\pp \in\mathcal{D}^2_{\text{neg}} \text{ inside the margin}\}$.

The \emph{target}, or label, is the value that encodes the information we want to learn. In classical settings, we have one scalar for each data point: a class for classification, a value for regression. In our setting, we consider a target $y_{ij}\in\mathbb{R}$ for every pair $\pp_{ij}\in\mathcal{D}^2$. With this notation in mind, we can continue to define the contrastive loss.


\section{Contrastive Loss}\label{sec:contrastive_loss}
Recall the definition of the contrastive loss~\cite{hadsell2006dimensionality} 
\begin{equation}\label{eq:contrastive_loss_bad_definition}
    \mathcal{L}_{\mathrm{con}}(\theta) 
     = 
    \frac{1}{2}\|f_\theta(x_a)-f_\theta(x_p)\|^2 
    +
    \frac{1}{2}\max\left(0, m- \|f_\theta(x_a)-f_\theta(x_n)\|^2\right),
\end{equation}
where $f$ is a network with parameters $\theta$ and $x_p$ is a data point the same class as the anchor $x_a$ and different from the negative data point $x_n$. The loss over the whole dataset is then informally defined as $\mathcal{L}_{\mathrm{con}}(\theta;\mathcal{D})=\mathbb{E}_{\mathcal{D}}[\mathcal{L}_{\mathrm{con}}(\theta)]$, where the expectation is taken over tuples $(x_a,x_p,x_n)$ satisfying the positive and negative constrains. This definition is intuitive and compact, but not formal enough to show that the contrastive loss is in fact an unnormalized log posterior. The main issue is the explicit notation for positive $x_p$ and negative $x_n$, which at first glance seems innocent, but for later derivations becomes cumbersome. Instead, we express the loss in an equivalent but more verbose way.

In order to understand that we are doing nothing more than a change in notation, it is first convenient to express \cref{eq:contrastive_loss_bad_definition} as
\begin{equation}\label{eq:per_obs_bad}
\mathcal{L}_{\text{con}}(\theta) 
     = 
    \left\{
    \begin{array}{ll}
        \frac{1}{2}\|f_\theta(x_i)-f_\theta(x_j)\|^2  & \text{ if the pair $\pp_{ij}$ is positive} \\
        0 & \text{ if the pair $\pp_{ij}$ is negative and outside the margin} \\
        m-\frac{1}{2}\|f_\theta(x_i)-f_\theta(x_j)\|^2 & \text{ if the pair $\pp_{ij}$ is negative and inside the margin}
    \end{array}
    \right.
    ,
\end{equation}

which reveals the three cases of the \emph{per-observation} contrastive loss. Notice that, up to a neglectable additive constant, all cases shares the same form of being scalar multiples of a distance between the two embeddings $z_i=f_\theta(x_i)$ and $z_j=f_\theta(x_j)$. This allows us to combine the three scenarios into a single-case \emph{per-observation} contrastive loss, parametrized by a scalar $y\in\mathbb{R}$, as
\begin{align}\label{eq:per_obs}
\mathcal{L}_{y}:
 & \quad\mathcal{Z}^2\,\,\longrightarrow\mathbb{R} \nonumber\\
 & \,\,z_i,z_j \longmapsto \frac{1}{2}y \|z_i-z_j\|^2 
\end{align}
and we later entrust the distinction between the three scenarios, as in \cref{eq:per_obs_bad}, to the scalar $y$. 
Although it is important to define this loss for a general target $y\in\mathbb{R}$ (we will interpret it as a Von-Mises-Fisher concentration parameter $\kappa$ in \cref{sec:probabilistic_view}), practically we will make use of the specific instances $y=y_{ij}\in\mathbb{R}$ with the target values defined for every data indexes $i$ and $j$ as
\begin{equation}\label{eq:target_definition}
    y_{ij}
    :=
    \left\{
    \begin{array}{ll}
        \frac{1}{| \mathcal{D}^2_{\text{pos}}|} & \text{ if the pair $\pp_{ij}$ is positive} \\
        0 & \text{ if the pair $\pp_{ij}$ is negative and outside the margin} \\
        -\frac{1}{| \mathcal{D}^2_{\text{neg}}|} & \text{ if the pair $\pp_{ij}$ is negative and inside the margin}
    \end{array}
    \right.
    .
\end{equation}
We define the contrastive loss over the entire dataset $\mathcal{L}(\,\,;\mathcal{D}):\Theta\rightarrow\mathbb{R}$ as a sum over all pairs of the \emph{per-observation} contrastive loss (\ref{eq:per_obs}) with targets (\ref{eq:target_definition})
\begin{equation}\label{eq:contrastive_loss_definition}
    \mathcal{L}(\theta;\mathcal{D})
    :=
    \sum_{\pp_{ij}\in\mathcal{D}^2} 
    \mathcal{L}_{y_{ij}}(
        \underbrace{f_\theta(x_i)}_{z_i}
        ,
        \underbrace{f_\theta(x_j)}_{z_j}
    ).
\end{equation}
Notice that we slightly overload the notation $\mathcal{L}(\theta)$ and $\mathcal{L}_{y}(\theta;\mathcal{D})$ for the dataset loss and the per-observation-loss, respectively. 

\textbf{Expanded expression}. In order to better highlight the equivalence, it may be useful to express the loss explicitly using all the previous definitions.
\begin{align}
    \mathcal{L}(\theta;\mathcal{D})
    & =
    \sum_{\pp_{ij}\in\mathcal{D}^2} 
    \mathcal{L}_{y_{ij}}(
        f_\theta(x_i),
        f_\theta(x_j)
    ) \\
    & =
    \sum_{\pp_{ij}\in\mathcal{D}^2} 
    \frac{1}{2}y_{ij}
    \|
        f_\theta(x_i) -
        f_\theta(x_j)
    \|^2 \\
    & =
    \sum_{\pp_{ij}\in\mathcal{D}^2_\text{pos}} 
    \frac{1}{2}
    \frac{1}{| \mathcal{D}^2_{\text{pos}}|}
    \|
        f_\theta(x_i) -
        f_\theta(x_j)
    \|^2
    +
    \sum_{\pp_{ij}\in\mathcal{D}^2_\text{neg inside}} 
    -\frac{1}{2}
    \frac{1}{| \mathcal{D}^2_{\text{neg}}|}
    \|
        f_\theta(x_i) -
        f_\theta(x_j)
    \|^2 \\
    & =
    \frac{1}{| \mathcal{D}^2_{\text{pos}}|}
    \sum_{\pp_{ij}\in\mathcal{D}^2_\text{pos}} 
    \frac{1}{2}
    \|
        f_\theta(x_i) -
        f_\theta(x_j)
    \|^2
    -
    \frac{1}{| \mathcal{D}^2_{\text{neg}}|}
    \sum_{\pp_{ij}\in\mathcal{D}^2_\text{neg inside}} 
    \frac{1}{2}
    \|
        f_\theta(x_i) -
        f_\theta(x_j)
    \|^2
\end{align}

The scaling factor $\frac{1}{| \mathcal{D}^2_{\text{pos}}|}$ for positives and $\frac{1}{| \mathcal{D}^2_{\text{neg}}|}$ for negatives, together with the sum over pairs, leads to a per-type average. In this sense, we can informally say that $\mathcal{L}(\theta)=\mathbb{E}[\mathcal{L}_{\text{con}}(\theta)]$ under a distribution that assigns equal probabilities of a pair being positive or negative. We highlight that scaling is key in order to ensure positive definiteness in \cref{prop:positive_definiteness}.

\subsection{Minibatch}\label{appendix:mining}

In the previous Section, we defined the contrastive loss for the entire dataset (\ref{eq:contrastive_loss_definition}). However, in practice it is approximated with minibatching. This approximation induces a bias that we can correct by simply adjusting the target $y_{ij}$ as in \cref{eq:minibatch_target_definition}.

\textbf{Minibatching recap.} When dealing with a huge number of identically distributed data, it is common to assume that a big enough arbitrary subset will follow the same distribution and thus have the same properties, specifically we assume the expected value to be similar $\mathbb{E}_{\mathcal{S}}[f(s)]\approx\mathbb{E}_{\mathcal{S}'}[f(s)]$.
The idea of minibatching relies on this to approximate a sum over a set $\mathcal{S}$ with a \emph{scaled} sum over a subset $\mathcal{S'}\subseteq\mathcal{S}$, 
\begin{equation}\label{eq:minibatching_logic}
    \sum_{s\in\mathcal{S}} f(s)
    \approx
    \frac{|\mathcal{S}|}{|\mathcal{S}'|} 
    \sum_{s\in\mathcal{S'}} f(s)
\end{equation}

where $\mathcal{S}=\mathcal{D}$ and $\mathcal{S'}$ is the set of data points in the minibatch. 

\textbf{Minibatching contrastive.} Conversely, in the constrastive setting, we need to consider a subset of the pairs (rather than single observations), i.e. $\mathcal{S}=\mathcal{D}^2$.
In practice this subset will not be representative of the positive and negative ratio, thus we need a different scaling to account for that. This process can be viewed as taking two independent minibatches at the same time, one representative of the positives and one of the negatives. This intuition is formalized by the following Definition and Proposition. 

\textbf{Definition.} Consider a minibatch set of pairs $\mathcal{B}\subseteq\mathcal{D}^2$ and its partition $\mathcal{B}=\mathcal{B}_{\text{pos}}\cup\mathcal{B}_{\text{neg}}$ in positives $\mathcal{B}_{\text{pos}}\subseteq\mathcal{D}^2_{\text{pos}}$ and negatives $\mathcal{B}_{\text{neg}}\subseteq\mathcal{D}^2_{\text{neg}}$, then we can define a new scaled target as
\begin{equation}\label{eq:minibatch_target_definition}
    y^{\mathcal{B}}_{ij}
    :=
    \left\{
    \begin{array}{ll}
        \frac{|\mathcal{B}|}{| \mathcal{D}^2| |\mathcal{B}_\text{pos}|}
        & \text{ if the pair $\pp_{ij}$ is positive} \\
        0 & \text{ if the pair $\pp_{ij}$ is negative and outside the margin} \\
        -\frac{|\mathcal{B}|}{| \mathcal{D}^2| |\mathcal{B}_\text{neg}|}
        & \text{ if the pair $\pp_{ij}$ is negative and inside the margin}
    \end{array}
    \right.
\end{equation}
and so we can properly approximate the loss by minibatching positives and negatives independently

\begin{proposition}
    Assume that the positives in the batch are representative of the positives in the whole dataset, and similarly for the negatives, i.e. $\mathbb{E}_{\mathcal{D}^2_{\text{pos}}}[\mathcal{L}_y]\approx\mathbb{E}_{\mathcal{B}_{\text{pos}}}[\mathcal{L}_y]$ and $\mathbb{E}_{\mathcal{D}^2_{\text{neg}}}[\mathcal{L}_y]\approx\mathbb{E}_{\mathcal{B}_{\text{neg}}}[\mathcal{L}_y]$. Then the loss, as defined in \cref{eq:contrastive_loss_definition}, can be approximated by using the target in \cref{eq:minibatch_target_definition} with
    \begin{equation}
        \mathcal{L}(\theta;\mathcal{D})
        \approx
        \frac{|\mathcal{D}^2|}{|\mathcal{B}|}
        \sum_{\pp_{ij}\in\mathcal{B}}
            \mathcal{L}_{y^{\mathcal{B}}_{ij}}(f_\theta(x_i),f_\theta(x_j))
    \end{equation}
\end{proposition}

\begin{proof}
The equality is proven by applying the logic of \cref{eq:minibatching_logic} two times independently, once for the positive pairs with $\mathcal{B}_{\text{pos}}=\mathcal{S}'\subseteq\mathcal{S}=\mathcal{D}^2_{\text{pos}}$ and once for the negatives with $\mathcal{B}_{\text{neg}}=\mathcal{S}'\subseteq\mathcal{S}=\mathcal{D}^2_{\text{neg}}$ and then rearranging the terms
\begin{align}
    \mathcal{L}(\theta;\mathcal{D})
    & =
    \sum_{\pp_{ij}\in\mathcal{D}^2} \mathcal{L}_{y_{ij}}(f_\theta(x_i),f_\theta(x_j)) \\
    & =
    \sum_{\pp_{ij}\in\mathcal{D}^2_{\text{pos}}} \mathcal{L}_{y_{ij}}(f_\theta(x_i),f_\theta(x_j))
    +
    \sum_{\pp_{ij}\in\mathcal{D}^2_{\text{neg}}} \mathcal{L}_{y_{ij}}(f_\theta(x_i),f_\theta(x_j)) \\
    & \approx
    \frac{|\mathcal{D}^2_{\text{pos}}|}{|\mathcal{B}_{\text{pos}}|}
        \sum_{\pp_{ij}\in\mathcal{B}_{\text{pos}}}
        \mathcal{L}_{y_{ij}}(f_\theta(x_i),f_\theta(x_j))
    +
    \frac{|\mathcal{D}^2_{\text{neg}}|}{|\mathcal{B}_{\text{neg}}|}
        \sum_{\pp_{ij}\in\mathcal{B}_{\text{neg}}}
        \mathcal{L}_{y_{ij}}(f_\theta(x_i),f_\theta(x_j)) \\
    & =
    \frac{|\mathcal{D}^2|}{|\mathcal{B}|}
    \left(
        \frac{|\mathcal{D}^2_{\text{pos}}|}{|\mathcal{D}^2|}
        \frac{|\mathcal{B}|}{|\mathcal{B}_{\text{pos}}|}
            \sum_{\pp_{ij}\in\mathcal{B}_{\text{pos}}}
            \mathcal{L}_{y_{ij}}(f_\theta(x_i),f_\theta(x_j))
        +
        \frac{|\mathcal{D}^2_{\text{neg}}|}{|\mathcal{D}^2|}
        \frac{|\mathcal{B}|}{|\mathcal{B}_{\text{neg}}|}
            \sum_{\pp_{ij}\in\mathcal{B}_{\text{neg}}}
            \mathcal{L}_{y_{ij}}(f_\theta(x_i),f_\theta(x_j))
    \right) \\
    & =
    \frac{|\mathcal{D}^2|}{|\mathcal{B}|}
    \left(
        \frac{\text{pos}^{\%}_{\mathcal{D}^2}}{\text{pos}^{\%}_{\mathcal{B}}}
            \sum_{\pp_{ij}\in\mathcal{B}_{\text{pos}}}
            \mathcal{L}_{y_{ij}}(f_\theta(x_i),f_\theta(x_j))
        +
        \frac{\text{neg}^{\%}_{\mathcal{D}^2}}{\text{neg}^{\%}_{\mathcal{B}}}
            \sum_{\pp_{ij}\in\mathcal{B}_{\text{neg}}}
            \mathcal{L}_{y_{ij}}(f_\theta(x_i),f_\theta(x_j))
    \right) \\
    & =
    \frac{|\mathcal{D}^2|}{|\mathcal{B}|}
    \left(
        \sum_{\pp_{ij}\in\mathcal{B}_{\text{pos}}}
            \mathcal{L}_{y^{\mathcal{B}}_{ij}}(f_\theta(x_i),f_\theta(x_j))
        +
        \sum_{\pp_{ij}\in\mathcal{B}_{\text{neg}}}
            \mathcal{L}_{y^{\mathcal{B}}_{ij}}(f_\theta(x_i),f_\theta(x_j))
    \right) \\
    & =
    \frac{|\mathcal{D}^2|}{|\mathcal{B}|}
    \sum_{\pp_{ij}\in\mathcal{B}}
        \mathcal{L}_{y^{\mathcal{B}}_{ij}}(f_\theta(x_i),f_\theta(x_j))
\end{align}
where $\text{pos}^{\%}$ and $\text{neg}^{\%}$ are the percentage of respectively positives and negatives in a given set, indicated in the subscript. 
\end{proof}

We highlight that this scaling is linear, and thus is reflected in both first and second-order derivatives. This will later become important for scaling the hessian. 

\section{Normalization layer and Von~Mises–Fisher}
\label{sec:normalization-vmf}

It is common in metric learning to add a normalization layer at the end of the neural network architecture. This, besides improving performances, has an interesting geometric interpretation. Moreover, happens to be fundamental in order to interpret the loss in a probabilistic way, and consequently to apply the Laplace approximation in a meaningful way.

Adding an $l^2$-normalization layer is a practical way to enforce that all outputs lie on the unit sphere. In other words, we are assuming our latent manifold to be
\begin{equation}
    \mathcal{Z}:=S^Z  \subset \mathbb{R}^{Z+1}.
\end{equation}

We can generalize the concept of Gaussian to the unit sphere. Start from a normal distribution with isotropic covariance $\sigma^2 \mathbb{I}$ and mean $\mu$. Conditioning on $\|z\|=1$ leads to a distribution on $S^Z$, the so-called von~Mises–Fisher distribution $\mathcal{N}^S$. More succinctly, the restriction of any isotropic multivariate normal density to the unit hypersphere, gives a von~Mises-Fisher density, up to normalization. 
\begin{equation}\label{eq:von_mises_fisher_density}
    \mathcal{N}^S(z|\mu,\kappa) \sim \mathcal{P}(z) = c_{\kappa} e^{-\kappa\frac{\|z-\mu\|^2}{2}} 
\end{equation}
where, importantly, the normalization constant $c_{\kappa}\in\mathbb{R}$ only depends on $\kappa$ and not on $\mu$. 
The von Mises-Fisher distribution is parametrized with a directional mean $\mu$ and a scalar concentration parameter $\kappa$, which can be interpreted as the inverse
of an isotropic covariance $\kappa=\frac{1}{\sigma^2}$.

Being both $\mu$ and $z$ of the unitary norm, the norm $\|z-\mu\|^2$ can be more efficiently computed through the scalar product $\langle z,\mu\rangle$, that is the reason it is also called Arccos distance. Moreover, we highlight that on the unit sphere, the equivalence 
\begin{equation}
\|z-\mu\|^2 = 2 - \langle z,\mu\rangle = 4 - \|z+\mu\|^2
\end{equation}
holds. This will be key in proving the equivalence of the probabilistic setting from \cref{eq:equiv_repulsive_term_line1} to \cref{eq:equiv_repulsive_term_line2}.

\textbf{Parameters estimators}. Having access to $N$ samples drawn from a Gaussian distribution with unknown mean and variance, it is common to compute an empirical estimate of such parameters. Various estimators exists, each satisfying different properties like being unbiased~\cite{kenney1951mathematics}.

In a similar fashion, for Von-Mises-Fisher distributions we make use of two empirical estimators~\citet{sra2011vonmises} of the parameters $\mu$ and $\kappa=\frac{1}{\sigma^2}$. We compute the empirical mean direction
\begin{equation}
    \bar{\mu} = \frac{\mu}{\bar{R}}, \qquad \bar{R} = \|\mu\|, \qquad \mu = \frac{1}{N} \sum_{i=0}^{N} z_i,
\end{equation}
from these samples and the approximate concentration parameter
\begin{equation}
    \bar{\kappa} = \frac{\bar{R}(D - \bar{R}^2)}{1 - \bar{R}^2}.
\end{equation}


\section{Probabilistic view}\label{sec:probabilistic_view}

In the metric learning framework, a dataset $\mathcal{D}$ plays two very different roles at the same time. This happens in the same spirit as, in electrostatics, a charged particle at the same time \emph{generates} an electric field and \emph{interact} with the existing electric field. This subtle difference, although being neglectable in the classic non-Bayesian setting, is conceptually important in the probabilistic setting. Thus, for now, we assume to be given a \emph{generating} dataset $\mathcal{D}_G$ and an \emph{interacting} dataset $\mathcal{D}_I$. Only after the derivations, we will set them to be equal. 

For some fixed value $\kappa>0$, each data point $x\in\mathcal{X}$ induces two Von Mises-Fisher distributions on the latent space
\begin{align}
    &\mathcal{P}^{\shortrightarrow\shortleftarrow}(z|x,\theta)
    \sim \mathcal{N}^{S}(z|f_\theta(x), \kappa) \\
    &\mathcal{P}^{\shortleftarrow\shortrightarrow}(z|x,\theta)
    \sim \mathcal{N}^{S}(z| -f_\theta(x), \kappa)
\end{align}
one centered in the embedding point $f_\theta(x)$ and one centered in the antipodal point $-f_\theta(x)$. 
It is convenient to explicit the densities of these two distributions, according to \cref{eq:von_mises_fisher_density}
\begin{align}
    \mathcal{P}^{\shortrightarrow\shortleftarrow}(z|x,\theta)
    &=
    c_\kappa \exp\left({-\kappa\frac{\|z-f_\theta(x)\|^2}{2}}\right) \\
    \label{eq:repulsive_term_density}
    \mathcal{P}^{\shortleftarrow\shortrightarrow}(z|x,\theta)
    &=
    c_\kappa \exp\left({-\kappa\frac{\|z+f_\theta(x)\|^2}{2}}\right).
\end{align}

\begin{figure}[h]
    \centering
    \includegraphics[width=0.7\linewidth]{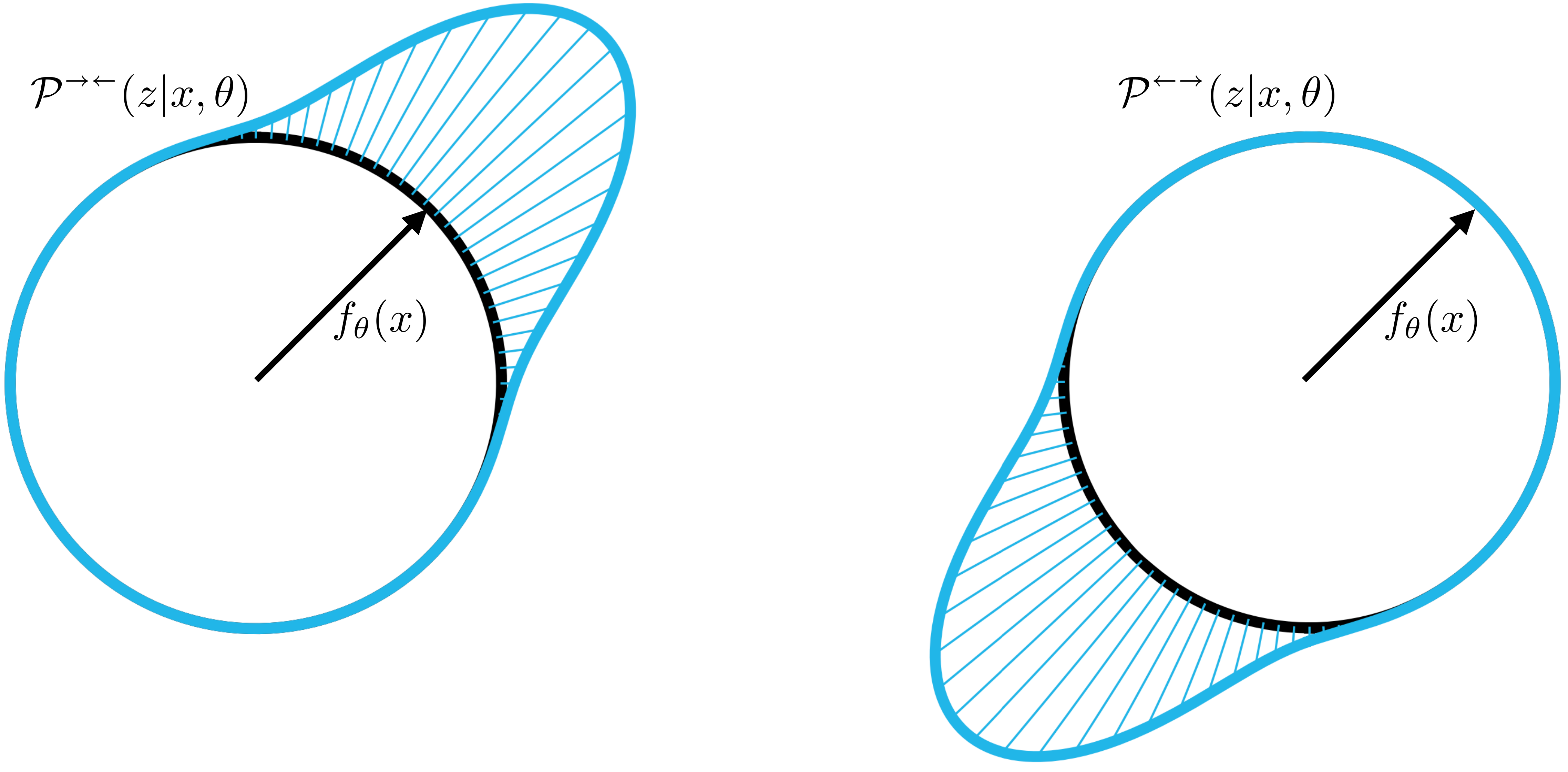}
    \caption{Representation of the Von Mises-Fisher densities $\mathcal{P}^{\shortrightarrow\shortleftarrow}$ and $\mathcal{P}^{\shortleftarrow\shortrightarrow}$ in the one-dimensional unit circle $\mathcal{Z}=\mathcal{S}^1$. Intuitively, the distribution corresponding to an attractive force is supported near the embedding $z=f_\theta(x)$, while the distribution corresponding to a repulsive force is supported far from the embedding, i.e. close to the antipodal $z=-f_\theta(x)$. The higher the value of $\kappa$ is, the higher the precision, the lower the variance and the narrower the support.}
    \label{fig:von_mises_fisher}
\end{figure}

The \emph{generating} dataset $\mathcal{D}_G$ induces a distribution $\mathcal{P}^c_{\mathcal{Z}}\in\Delta(\mathcal{Z)}$ on the latent space for each class $c\in\mathcal{C}$ defined by
\begin{equation}\label{eq:induced_distribution_definition_new}
    \mathcal{P}^c_{\mathcal{Z}}(z|\mathcal{D}_G,\theta)
    =
    \frac{c_\kappa}{|\mathcal{D}_G|} \prod_{\substack{(x_j,c_j)\in\mathcal{D}_G \\ \text{s.t. } c_j = c}}
    \mathcal{P}^{\shortrightarrow\shortleftarrow}(z|x_j,\theta)
    \prod_{\substack{(x_j,c_j)\in\mathcal{D}_G \\ \text{s.t. } c_j\not= c}}
    \mathcal{P}^{\shortleftarrow\shortrightarrow}(z|x_j,\theta)
\end{equation}
that is a product of one \emph{attractive} term $\mathcal{P}^{\shortrightarrow\shortleftarrow}$ for all anchor points of the same class, times a \emph{repulsive} term $\mathcal{P}^{\shortleftarrow\shortrightarrow}$ for all anchor points of different classes. 
In the same spirit as the product of Gaussian densities is a Gaussian density, we can show 
that $\mathcal{P}^c_{\mathcal{Z}}$ is a Von Mises-Fisher distribution itself, the precision of which is the sum of the precisions. This implies that the normalization constant only depends on $\kappa$ for this distribution as well.

Having access to a likelihood in the latent space, the \emph{interacting} dataset likelihood is then classically defined as the product
\begin{equation}
    \mathcal{P}(\mathcal{D}_I|\mathcal{D}_G, \theta)
    =
    \prod_{(x_i,c_i)\in\mathcal{D}_I}
    \mathcal{P}^{c_i}_{\mathcal{Z}}(z|\mathcal{D}_G,\theta)\big|_{z=f_\theta(x_i)}
\end{equation}
Using the definition in \cref{eq:induced_distribution_definition_new} and rearranging the terms, we can rewrite this likelihood in extended form as
\begin{equation}
    \mathcal{P}(\mathcal{D}_I|\mathcal{D}_G, \theta)
    =
    \bar{c}_\kappa \prod_{\substack{\pp_{ij}\in\mathcal{D}_I\times\mathcal{D}_G \\ \text{s.t. } \pp_{ij} \text{ is positive}}}
    \mathcal{P}^{\shortrightarrow\shortleftarrow}(z|x_j,\theta)\big|_{z=f_\theta(x_i)}
    \prod_{\substack{\pp_{ij}\in\mathcal{D}_I\times\mathcal{D}_G \\ \text{s.t. } \pp_{ij} \text{ is negative}}}
    \mathcal{P}^{\shortleftarrow\shortrightarrow}(z|x_j,\theta)\big|_{z=f_\theta(x_i)}
\end{equation}

Setting the generating and interacting dataset to be the same, we can define the constrastive learning dataset likelihood as
\begin{equation}
    \mathcal{P}(\mathcal{D}|\theta) := \mathcal{P}(\mathcal{D}_I|\mathcal{D}_G, \theta)\big|_{\mathcal{D}_I=\mathcal{D}_G=\mathcal{D}}
\end{equation}

\subsection{Equivalence of the two settings}
The contrastive term for a single pair is equivalent to the negative log-likelihood of a von~Mises-Fisher distribution, up to additive constants. Specifically, a positive target $y=\kappa>0$ is related to the attractive term $\mathcal{P}^{\shortrightarrow\shortleftarrow}$
\begin{align}
    \mathcal{L}_{\kappa}(f_\theta(x_e),f_\theta(x_a))
    &=
    \frac{1}{2}\kappa\|f_\theta(x_e)-f_\theta(x_a)\|^2  \\
    &=
    \log(c_\kappa)  -\log \mathcal{P}^{\shortrightarrow\shortleftarrow}(f_\theta(x_e)|x_a,\theta)
\end{align}
while a negative target $y=-\kappa<0$ is related to the repulsive term $\mathcal{P}^{\shortleftarrow\shortrightarrow}$
\begin{align}
    \label{eq:equiv_repulsive_term_line1}
    \mathcal{L}_{-\kappa}(f_\theta(x_e),f_\theta(x_a))
    &=
    -\frac{1}{2}\kappa\|f_\theta(x_e)-f_\theta(x_a)\|^2  \\
    \label{eq:equiv_repulsive_term_line2}
    &=
    -2\kappa + \frac{1}{2}\kappa\|f_\theta(x_e)+f_\theta(x_a)\|^2  \\
    &=
    -2\kappa + \log(c_\kappa) -\log \mathcal{P}^{\shortleftarrow\shortrightarrow}(f_\theta(x_e)|x_a,\theta).
\end{align}

This equivalence is reflected in the equivalence between the contrastive loss and the dataset negative log-likelihood, up to an additive constant. Expanding the definition of dataset likelihood we have
\begin{align}
    \mathcal{P}(\mathcal{D}|\theta)
    & =
    \prod_{(x_i,c_i)\in\mathcal{D}}
    \mathcal{P}^{c_i}_{\mathcal{Z}}(f_\theta(x_i)|\mathcal{D},\theta) \\
    & =
    \bar{c}_\kappa \prod_{\substack{\pp_{ij}\in\mathcal{D}^2 \\ \text{s.t. } \pp_{ij} \text{ is positive}}}
    \mathcal{P}^{\shortrightarrow\shortleftarrow}(z|x_j,\theta)\big|_{z=f_\theta(x_i)}
    \prod_{\substack{\pp_{ij}\in\mathcal{D}^2\\ \text{s.t. } \pp_{ij} \text{ is negative}}}
    \mathcal{P}^{\shortleftarrow\shortrightarrow}(z|x_j,\theta)\big|_{z=f_\theta(x_i)} \\
    & =
    \bar{c}_\kappa c_\kappa^{|\mathcal{D}|} \prod_{\substack{\pp_{ij}\in\mathcal{D}^2 \\ \text{s.t. } \pp_{ij} \text{ is positive}}}
    \exp\left({-\kappa\frac{\|f_\theta(x_i)-f_\theta(x_j)\|^2}{2}}\right)
    \prod_{\substack{\pp_{ij}\in\mathcal{D}^2\\ \text{s.t. } \pp_{ij} \text{ is negative}}}
    \exp\left({-\kappa\frac{\|f_\theta(x_i)+f_\theta(x_j)\|^2}{2}}\right)
\end{align}

Considering the log-likelihood
\begin{align}
    \log \mathcal{P}(\mathcal{D}|\theta)
    & =
    \sum_{(x_i,c_i)\in\mathcal{D}}
    \log \mathcal{P}^{c_i}_{\mathcal{Z}}(f_\theta(x_i)|\mathcal{D},\theta) \\
    & =
    \log(\bar{c}_\kappa)
    +
    \sum_{\substack{\pp_{ij}\in\mathcal{D}^2 \\ \text{s.t. } \pp_{ij} \text{ is positive}}}
    \log \mathcal{P}^{\shortrightarrow\shortleftarrow}(z|x_j,\theta)\big|_{z=f_\theta(x_i)}
    +
    \sum_{\substack{\pp_{ij}\in\mathcal{D}^2\\ \text{s.t. } \pp_{ij} \text{ is negative}}}
    \log \mathcal{P}^{\shortleftarrow\shortrightarrow}(z|x_j,\theta)\big|_{z=f_\theta(x_i)} \\
    & =
    \mathrm{const}(\kappa)
    -
    \sum_{\substack{\pp_{ij}\in\mathcal{D}^2 \\ \text{s.t. } \pp_{ij} \text{ is positive}}}
    \mathcal{L}_{\kappa}(f_\theta(x_i),f_\theta(x_j))
    -
    \sum_{\substack{\pp_{ij}\in\mathcal{D}^2\\ \text{s.t. } \pp_{ij} \text{ is negative}}}
    \mathcal{L}_{-\kappa}(f_\theta(x_i),f_\theta(x_j)) \\
    & =
    \mathrm{const}(\kappa)
    - \mathcal{L}(\theta;\mathcal{D})
\end{align}
thus the loss is equal, up to an additive constant, to the negative log-likelihood and we can proceed with the Bayesian interpretation. We highlight that $\mathrm{const}(\kappa)$ is not dependent on $\theta$, and so is neglected by the derivatives
\begin{align}
    \nabla_\theta \log \mathcal{P}(\mathcal{D}|\theta)
    & =
    - \nabla_\theta \mathcal{L}(\theta;\mathcal{D}) \\
    \nabla^2_\theta \log \mathcal{P}(\mathcal{D}|\theta)
    & =
    - \nabla^2_\theta \mathcal{L}(\theta;\mathcal{D}).
\end{align}
This shows that contrastive loss minimization is indeed a likelihood maximization. We highlight that this proof assumes margin $m>2$, i.e. all negatives are inside the margin (because on the manifold $\mathcal{S}^Z$ the maximum distance is 2). Following the exact same structure, a proof can be done with arbitrary $m$. The only difference would be in the definition \cref{eq:repulsive_term_density}, instead of a Von Mises-Fisher it would be a discontinuous density and the normalization constants would depend on both $\kappa$ and $m$.



\section{Bayesian Metric learning}
Having defined the dataset likelihood $\mathcal{P}(\mathcal{D}|\theta)$ conditioned to the parameter $\theta$, the parameter likelihood $\mathcal{P}(\theta|\mathcal{D})$ conditioned to the dataset is defined according to Bayes rule as
\begin{equation}
    \mathcal{P}(\theta|\mathcal{D})
    =
    \frac{\mathcal{P}(\mathcal{D}|\theta) \mathcal{P}(\theta)}{\mathcal{P}(\mathcal{D})}.
\end{equation}

The aim of metric learning is to maximise such likelihood with respect to the parameter $\theta$, specifically, it aims at finding the Maximum A Posteriori
\begin{equation}
    \theta^{\text{MAP}} \in \argmax_{\theta\in\Theta} \mathcal{P}(\theta|\mathcal{D}).
\end{equation}
We highlight that the argmax is the same in log-scale
\begin{equation}
    \argmax_{\theta\in\Theta} \mathcal{P}(\theta|\mathcal{D})
    =
    \argmax_{\theta\in\Theta} \log \mathcal{P}(\theta|\mathcal{D})
    =
    \argmax_{\theta\in\Theta} 
    \log \mathcal{P}(\mathcal{D}|\theta) + \log \mathcal{P}(\theta)
\end{equation}
where the log-prior takes the form of a standard $l^2$ regularize (weight decay), and, as shown before, maximizing the log-likelihood $\log p(\mathcal{D}|\theta)$ is equivalent to minimizing of the contrastive loss $\mathcal{L}(\theta;\mathcal{D})$. 

However in the Bayesian setting, we do not seek a maximum, but we seek an expression of the whole distribution $\mathcal{P}(\theta|\mathcal{D})$, for which we use the notation $q(\theta)$. We do that by maximizing the dataset likelihood $\mathcal{P}(\mathcal{D})$ by integrating out $\theta$ with $\mathcal{P}(\mathcal{D})=\mathbb{E}_{\theta\sim q}[\mathcal{P}(\mathcal{D} | \theta)]$. This maximization is then defined by
\begin{align}
    q^*(\theta)
    :=
    \mathcal{P}^*(\theta|\mathcal{D})
    \in 
    \argmax_{\mathcal{P}(\theta|\mathcal{D})\in\mathcal{G}(\Theta)} 
        \E_{\theta\sim \mathcal{P}(\theta|\mathcal{D})}
        [\mathcal{P}(\mathcal{D} | \theta)]
\end{align}



Notice that, limited by the ability to parametrize distributions, we aim at finding such a maximum on $q$ over only some subspace of distributions on $\Theta$. To this end, we consider the Laplace approximation, which is one way of choosing a subspace of the parameter distribution. In Laplace post-hoc we restrict ourselves to the space of Gaussians $\mathcal{G}(\Theta)\subset\Delta(\Theta)$ centered in $\theta^{\text{MAP}}$. This is a strong assumption since there are (1) no guarantees of $p(\theta|\mathcal{D})$ being Gaussian and (2) no guarantees of the distribution to be centered in $\theta^{\text{MAP}}$. Online Laplace lifts the latter assumption, and only assumes the parameters to be Gaussian distributed. 

With the general Bayesian framework in mind, we now proceed to derive post-hoc and online Laplace.

\subsection{Laplace post-hoc}

The Bayes rule
\begin{equation}
    \mathcal{P}(\theta|\mathcal{D})
    =
    \frac{\mathcal{P}(\mathcal{D}|\theta) \mathcal{P}(\theta)}{\mathcal{P}(\mathcal{D})}
\end{equation}
implies that
\begin{equation}
    \nabla^2_\theta \log \mathcal{P}(\theta|\mathcal{D})
    =
    \nabla^2_\theta \log \mathcal{P}(\mathcal{D}|\theta) 
    +
    \nabla^2_\theta \log \mathcal{P}(\theta).
\end{equation}

Assuming an isotropic Gaussian prior $\mathcal{P}(\theta)\sim\mathcal{N}(\theta|0,\sigma^2_{\text{prior}}\mathbb{I})$ implies $\nabla^2_\theta \log \mathcal{P}(\theta)=-\sigma_{\text{prior}}^{-2}\mathbb{I}$ and we have
\begin{align}
    \nabla^2_\theta \log \mathcal{P}(\theta|\mathcal{D})
    & =
    \nabla^2_\theta \log \mathcal{P}(\mathcal{D}|\theta) 
    -
    \sigma_{\text{prior}}^{-2}\mathbb{I} \\
    & =
    -\nabla^2_\theta \mathcal{L}(\theta;\mathcal{D}) 
    -
    \sigma_{\text{prior}}^{-2}\mathbb{I}.
\end{align}
Thus, we have two options:
\begin{itemize}
    \item IF $\mathcal{P}(\theta|\mathcal{D})$ is a Gaussian it holds $\forall \theta^*\in\Theta$
\begin{equation}
    \mathcal{P}(\theta|\mathcal{D})
    \sim
    \mathcal{N}(\theta|\mu=\theta^{\text{MAP}},\Sigma = (\nabla^2_\theta \mathcal{L}(\theta^*;\mathcal{D}) 
    +
    \sigma_{\text{prior}}^{-2}\mathbb{I})^{-1})
\end{equation}

    \item ELSE we can do a second order Taylor approximation of $\log \mathcal{P}(\theta|\mathcal{D})$ around $\theta^{\text{MAP}}$ and we have the approximation
\begin{equation}
    \mathcal{P}(\theta|\mathcal{D})
    \sim
    \mathcal{N}(\theta|\mu=\theta^{\text{MAP}},\Sigma = (\nabla^2_\theta \mathcal{L}(\theta^{\text{MAP}};\mathcal{D}) 
    +
    \sigma_{\text{prior}}^{-2}\mathbb{I})^{-1})
\end{equation}
\end{itemize}

\subsection{Laplace online}
At every step $t$ we have some Gaussian on the parameter space
\begin{equation}
    q^t(\theta)
    \sim
    \mathcal{N}(\theta|\mu=\theta_t,\Sigma=(H_t)^{-1})
\end{equation}
Where the values $\theta_t$ and $H_t$ are iteratively defined as
\begin{equation}\label{eq:updaterule_theta_point}
    \theta_{t+1}
    =
    \theta_t
    +
    \lambda
    \nabla_\theta \mathcal{L}(\theta_t;\mathcal{D}) 
    \qquad
    \theta_0 = \mu_{\text{prior}}
\end{equation}
and
\begin{equation}\label{eq:updaterule_precision_point}
    H_{t+1}
    =
    (1-\alpha) H_t
    +
    \nabla^2_\theta \mathcal{L}(\theta_t;\mathcal{D}) 
    \qquad
    H_0 = \sigma_{\text{prior}}^{-2}\mathbb{I}
\end{equation}
For some learning rate $\lambda$ and memory factor $\alpha$. This means that $q^0(\theta)$ is actually the prior distribution, which is updated with the first and second order derivatives of the loss. The updates can be improved by computing the derivatives not only in the single point $\theta_t$, but rather on the expected value with $\theta$ following the distribution $q^t$. This leads to the update rules
\begin{equation}\label{eq:updaterule_theta_expected}
    \theta_{t+1}
    =
    \theta_t
    +
    \lambda
    \mathbb{E}_{\theta\sim q^t(\theta)}
    [\nabla_\theta \mathcal{L}(\theta;\mathcal{D}) 
    ]
    \qquad
    \theta_0 = \mu_{\text{prior}}
\end{equation}
and
\begin{equation}\label{eq:updaterule_precision_expected}
    H_{t+1}
    =
    (1-\alpha) H_t
    +
    \mathbb{E}_{\theta\sim q^t(\theta)}
    [\nabla^2_\theta \mathcal{L}(\theta;\mathcal{D}) 
    ]
    \qquad
    H_0 = \sigma_{\text{prior}}^{-2}\mathbb{I}
\end{equation}


\section{Derivatives}\label{sec:derivatives}

In order to perform the Laplace based learning, we need to compute the second-order derivative of the loss with respect to the parameter $\theta$. Let us start by fixing a target $y\in\mathbb{R}$ 
and two data points $x_1,x_2\in\mathcal{X}$ and compute the second order derivative of one contrastive term $\mathcal{L}_{y}(f_\theta(x_1),f_\theta(x_2))$. This term can be viewed as a composition of functions, graphically represented as
\begin{equation}
    \begin{array}{ccc}
        x_1 & \overset{\displaystyle f_\theta}
            {\xrightarrow{\hspace*{1.5cm}}} & z_1\\
        x_2 & \underset{\displaystyle f_\theta}
            {\xrightarrow{\hspace*{1.5cm}}} & z_2
    \end{array}
    \overset{\displaystyle \mathcal{L}_y}
            {\xrightarrow{\hspace*{1.5cm}}}
    l
\end{equation}

To make the derivation cleaner is it useful to define, for a given function $f_\theta:\mathcal{X}\rightarrow\mathcal{Z}$, and auxiliary function $\ff_\theta:\mathcal{X}^2\rightarrow\mathcal{Z}^2$ defined by $\ff_\theta(x_1,x_2) := (f_\theta(x_1),f_\theta(x_2))$. In this way the graphical representation is
\begin{equation}
    \begin{pmatrix}
        x_1 \\ x_2
    \end{pmatrix}
    \overset{\displaystyle \ff_\theta}
            {\xrightarrow{\hspace*{1.5cm}}}
    \begin{pmatrix}
        z_1 \\ z_2
    \end{pmatrix}
    \overset{\displaystyle \mathcal{L}_y}
            {\xrightarrow{\hspace*{1.5cm}}}
    l
\end{equation}
and we can directly apply the chain rule. Before doing so, it is convenient to expand some derivatives to express the $\mathcal{Z}^2$-size matrix as two $\mathcal{Z}$-size submatrixes. The Jacobian of $\ff$ evaluated in $(x_1,x_2)$, is an operator from the tangent space $T_\theta\Theta$ to the tangent space $T_{(z_1,z_2)}\mathcal{Z}^2$, which can be written in block matrix form as
\begin{equation}
    J_\theta \ff_\theta(x_1,x_2)
    =
    \begin{pmatrix}
        J_\theta f_\theta(x_1) \\
        J_\theta f_\theta(x_2)
    \end{pmatrix}
\end{equation}
and, similarly, the hessian of $\mathcal{L}_y(z_1,z_2)$ can be written in block form as
\begin{equation}
    \nabla^2_{(z_1,z_2)} \mathcal{L}_y(z_1,z_2)
    =
    \begin{pmatrix}
        \nabla^2_{z_1} \mathcal{L}_y(z_1,z_2)
         & \nabla_{z_1}\nabla_{z_2} \mathcal{L}_y(z_1,z_2) \\
        \nabla_{z_2}\nabla_{z_1} \mathcal{L}_y(z_1,z_2)
         & \nabla^2_{z_2} \mathcal{L}_y(z_1,z_2)
    \end{pmatrix} 
\end{equation}
The Hessian of the per-observation Contrastive loss is then 
\begin{align}
    \nabla^2_\theta \mathcal{L}_y(f_\theta(x_1),f_\theta(x_2))
    &=
    \nabla^2_\theta \mathcal{L}_y(\ff_\theta(x_1,x_2)) \\
    & \!\stackrel{\textsc{ggn}}{\approx}
    J_\theta \ff_\theta(x_1,x_2)^\top
    \cdot
    \nabla^2_{(z_1,z_2)} \mathcal{L}_y(z_1,z_2)
    \cdot
    J_\theta \ff_\theta(x_1,x_2) \\
    &=
    \begin{pmatrix}
        J_\theta f_\theta(x_1) \\
        J_\theta f_\theta(x_2)
    \end{pmatrix}^\top
    \begin{pmatrix}
        \nabla^2_{z_1} \mathcal{L}_y(z_1,z_2)
         & \nabla_{z_1}\nabla_{z_2} \mathcal{L}_y(z_1,z_2) \\
        \nabla_{z_2}\nabla_{z_1} \mathcal{L}_y(z_1,z_2)
         & \nabla^2_{z_2} \mathcal{L}_y(z_1,z_2)
    \end{pmatrix} 
    \begin{pmatrix}
        J_\theta f_\theta(x_1) \\
        J_\theta f_\theta(x_2)
    \end{pmatrix}
\end{align}
where $(z_1,z_2)=h_\theta(x_1,x_2)$ is the point where we have to evaluate the derivative wrt to $z$. Consequently, the Hessian of the contrastive loss is
\begin{align}
    \nabla^2_\theta \mathcal{L}(\theta;\mathcal{D})
    & = \sum_{\pp_{ij}\in\mathcal{D}^2} \nabla^2_\theta\mathcal{L}_{y_{ij}}(f_\theta(x_i),f_\theta(x_j)) \\
    & = \sum_{\pp_{ij}\in\mathcal{D}^2}
    \begin{pmatrix}
        J_\theta f_\theta(x_i) \\
        J_\theta f_\theta(x_j)
    \end{pmatrix}^\top
    \begin{pmatrix}
        \nabla^2_{z_i} \mathcal{L}_{y_{ij}}(z_i,z_j)
         & \nabla_{z_j}\nabla_{z_j} \mathcal{L}_{y_{ij}}(z_i,z_j) \\
        \nabla_{z_j}\nabla_{z_i} \mathcal{L}_{y_{ij}}(z_i,z_j)
         & \nabla^2_{z_j} \mathcal{L}_{y_{ij}}(z_i,z_j)
    \end{pmatrix} 
    \begin{pmatrix}
        J_\theta f_\theta(x_i) \\
        J_\theta f_\theta(x_j)
    \end{pmatrix} 
\end{align}

We now proceed to find the derivatives of the per-observation loss $\mathcal{L}_y$ wrt. the Neural Network outputs $z_i=f_\theta(x_i)$ and $z_j=f_\theta(x_j)$. The derivatives varies based on the specific loss term that we consider. In the following we derive the derivatives for Euclidean $\mathcal{L}_y^E$ and Arccos $\mathcal{L}_y^A$ cases.

\textbf{Split choice}. The \textsc{ggn} assumes access to a composition of two functions, the specific split choice affects the result. As said in \cref{sec:normalization-vmf} it is common to have a normalization layer at the end of the Neural Network. This can be schematized as follow
\begin{equation}
    \begin{pmatrix}
        x_1 \\ x_2
    \end{pmatrix}
    \overset{\displaystyle \textsc{nn}}
            {\xrightarrow{\hspace*{1.5cm}}}
    \begin{pmatrix}
        z_1 \\ z_2
    \end{pmatrix}
    \overset{\displaystyle \ell_2\text{-norm}}
            {\xrightarrow{\hspace*{1.9cm}}}
    \begin{pmatrix}
        \frac{z_1}{\|z_1\|} \\ \frac{z_2}{\|z_2\|}
    \end{pmatrix}
    \overset{\displaystyle \text{distance}}
            {\xrightarrow{\hspace*{1.5cm}}}
    l
\end{equation}
and this leads to (at least) two possible split choices, wheter we include the normalization layer in the network left or right function, which we proceed to study further. We highlight that these two split choices can be interpreted as different loss function or, equivalently, as different distance metric: Euclidean or Arccos.

\subsection{Euclidean distance}\label{sec:derivatives_euclidean}
If we consider the $\ell_2$-normalization layer as part of the Neural Network $f$ 
\begin{equation}
    \begin{pmatrix}
        x_1 \\ x_2
    \end{pmatrix}
    \overset{\displaystyle \textsc{nn}}
            {\xrightarrow{\hspace*{1.5cm}}}
    \begin{pmatrix}
        z_1 \\ z_2
    \end{pmatrix}
    \overset{\displaystyle \ell_2\text{-norm}}
            {\xrightarrow{\hspace*{1.9cm}}}
    \begin{pmatrix}
        \frac{z_1}{\|z_1\|} \\ \frac{z_2}{\|z_2\|}
    \end{pmatrix}
\underbrace{
    \underset{\phantomsection}{
        \overset{\displaystyle \text{distance}}
            {\xrightarrow{\hspace*{1.5cm}}}}
}_{\displaystyle\mathcal{L}_y^E}
    l
\end{equation}
then the loss $\mathcal{L}_y=\mathcal{L}_y^E=$ is the \emph{Euclidean} distance defined as
\begin{equation}\label{eq:contrastive_term_with_euclidean_dist}
    \mathcal{L}_y^E(z_i,z_j)
    := \frac{1}{2}y\|z_i - z_j\|^2
\end{equation}
and it holds
\begin{align}
    \nabla^2_{z_1} \mathcal{L}_y^E(z_1,z_2)
    =
    \nabla^2_{z_2} \mathcal{L}_y^E(z_1,z_2)
    &=
    y\mathbb{I} \\
    \nabla_{z_1}\nabla_{z_2} \mathcal{L}_y^E(z_1,z_2)
    =
    \nabla_{z_2}\nabla_{z_1} \mathcal{L}_y^E(z_1,z_2)
    &=
    -y\mathbb{I}
\end{align}
which leads to
\begin{equation}\label{eq:hessian_loss_euclidean_wrt_z}
    \nabla^2_{(z_1,z_2)} \mathcal{L}_y^E(z_1,z_2)
    =
    y\begin{pmatrix}
        \mathbb{I}
         & -\mathbb{I} \\
        -\mathbb{I}
         & \mathbb{I}
    \end{pmatrix}.
\end{equation}

We highlight that the matrix $\begin{pmatrix}
        \mathbb{I}
         & -\mathbb{I} \\
        -\mathbb{I}
         & \mathbb{I}
    \end{pmatrix}$ is positive semi-definite, this will be useful in \cref{prop:positive_definiteness}. 

\textbf{Intuition. } The overall hessian expression can be further simplified by considering the matrix square root.

\begin{align}
\nabla^2_\theta \mathcal{L}(\theta;\mathcal{D})
    & = \sum_{\pp_{ij}\in\mathcal{D}^2} \nabla^2_\theta\mathcal{L}_{y(p_{ij})}(f_\theta(x_i),f_\theta(x_j)) \\
    & = \sum_{\pp_{ij}\in\mathcal{D}^2} y_{ij} 
    \begin{pmatrix}
        J_\theta f_\theta(x_i) \\
        J_\theta f_\theta(x_j)
    \end{pmatrix}^\top
    \begin{pmatrix}
        \mathbb{I}
         & -\mathbb{I} \\
        -\mathbb{I}
         & \mathbb{I}
    \end{pmatrix}
    \begin{pmatrix}
        J_\theta f_\theta(x_i) \\
        J_\theta f_\theta(x_j)
    \end{pmatrix} \\
    & = \frac{1}{2}\sum_{\pp_{ij}\in\mathcal{D}^2} y_{ij} 
    \begin{pmatrix}
        J_\theta f_\theta(x_i) \\
        J_\theta f_\theta(x_j)
    \end{pmatrix}^\top
    \begin{pmatrix}
        \mathbb{I}
         & -\mathbb{I} \\
        -\mathbb{I}
         & \mathbb{I}
    \end{pmatrix}
    \begin{pmatrix}
        \mathbb{I}
         & -\mathbb{I} \\
        -\mathbb{I}
         & \mathbb{I}
    \end{pmatrix}
    \begin{pmatrix}
        J_\theta f_\theta(x_i) \\
        J_\theta f_\theta(x_j)
    \end{pmatrix} \\
    & = \frac{1}{2}\sum_{\pp_{ij}\in\mathcal{D}^2} y_{ij} 
    \begin{pmatrix}
        J_\theta f_\theta(x_i) - J_\theta f_\theta(x_j)\\
        J_\theta f_\theta(x_j) - J_\theta f_\theta(x_i)
    \end{pmatrix} ^\top
    \begin{pmatrix}
        J_\theta f_\theta(x_i) - J_\theta f_\theta(x_j)\\
        J_\theta f_\theta(x_j) - J_\theta f_\theta(x_i)
    \end{pmatrix} \\
    & = \sum_{\pp_{ij}\in\mathcal{D}^2} y_{ij} 
    \begin{pmatrix}
        J_\theta f_\theta(x_i) - J_\theta f_\theta(x_j)
    \end{pmatrix} ^\top
    \begin{pmatrix}
        J_\theta f_\theta(x_i) - J_\theta f_\theta(x_j)
    \end{pmatrix} .
\end{align}

This expression give raise to two interpretations. (1) This formulation draws parallels with the GGN approximation of the MSE loss. (2) It can be viewed as the squared distance of the jacobians product, where the sign $y_{ij}$ determines the sign. This is parallel with the contrastive loss

\begin{align}
\mathcal{L}(\theta;\mathcal{D})
    &= 
    \frac{1}{2}
    \sum_{\pp_{ij}\in\mathcal{D}^2} y_{ij} 
    \| f_\theta(x_i) -  f_\theta(x_j) \|^2 \\
    &= 
    \frac{1}{2}
    \sum_{\pp_{ij}\in\mathcal{D}^2} y_{ij} 
    \begin{pmatrix}
        f_\theta(x_i) -  f_\theta(x_j)
    \end{pmatrix} ^\top
    \begin{pmatrix}
        f_\theta(x_i) -  f_\theta(x_j)
    \end{pmatrix}
\end{align}
with the only difference being the Jacobian operator (and the 2 factor).

\subsection{Arccos distance}\label{sec:derivatives_arccos}
If we consider the $\ell_2$-normalization layer as part of the loss function $\mathcal{L}_y$ 
\begin{equation}
    \begin{pmatrix}
        x_1 \\ x_2
    \end{pmatrix}
    \underset{\strut\phantomsection}{
        \overset{\displaystyle \textsc{nn}}
                {\xrightarrow{\hspace*{1.5cm}}}
    }
    \begin{pmatrix}
        z_1 \\ z_2
    \end{pmatrix}
\underbrace{
    \overset{\displaystyle \ell_2\text{-norm}}
            {\xrightarrow{\hspace*{1.9cm}}}
    \begin{pmatrix}
        \frac{z_1}{\|z_1\|} \\ \frac{z_2}{\|z_2\|}
    \end{pmatrix}
    \overset{\displaystyle \text{distance}}
            {\xrightarrow{\hspace*{1.5cm}}}
}_{\displaystyle \mathcal{L}_y^A}
    l
\end{equation}
then the loss $\mathcal{L}_y=\mathcal{L}_y^A$ is the \emph{Arccos} distance defined as
\begin{equation}\label{eq:contrastive_term_with_arccos_dist}
    \mathcal{L}_y^A(z_i,z_j)
    : = \frac{1}{2}y
    \left\|
    \frac{z_i}{\|z_i\|} - \frac{z_j}{\|z_j\|}
    \right\|^2  
    =
    y \left(
    1 - \left\langle \frac{z_i}{\|z_i\|},\frac{z_j}{\|z_j\|} \right\rangle
    \right)
\end{equation}
and it holds
\begin{align}
\nabla^2_{z_1} \mathcal{L}_y^A(z_1,z_2)
    & =
    \frac{y}{\|z_1\|^2}
    \left(
    \left\langle \frac{z_1}{\|z_1\|},\frac{z_2}{\|z_2\|} \right\rangle \mathbb{I}
    +
    \frac{z_1^\top z_2 + z_2^\top z_1}{\|z_1\|\|z_2\|}
    -
    3 \left\langle \frac{z_1}{\|z_1\|},\frac{z_2}{\|z_2\|} \right\rangle
    \frac{z_1^\top z_1}{\|z_1\|^2}
    \right)
\\
\nabla_{z_2}\nabla_{z_1} \mathcal{L}_y^A(z_1,z_2)
    & =
    \frac{y}{\|z_1\|\|z_2\|}
    \left(
    -\mathbb{I}
    +
    \frac{z_1^\top z_1}{\|z_1\|^2}
    +
    \frac{z_2^\top z_2}{\|z_2\|^2}
    -
    \left\langle \frac{z_1}{\|z_1\|},\frac{z_2}{\|z_2\|} \right\rangle 
    \frac{z_1^\top z_2}{\|z_1\|\|z_2\|}
    \right)
\\
\nabla^2_{z_2} \mathcal{L}_y^A(z_1,z_2)
    & =
    \frac{y}{\|z_2\|^2}
    \left(
    \left\langle \frac{z_1}{\|z_1\|},\frac{z_2}{\|z_2\|} \right\rangle \mathbb{I}
    +
    \frac{z_1^\top z_2 + z_2^\top z_1}{\|z_1\|\|z_2\|}
    -
    3 \left\langle \frac{z_1}{\|z_1\|},\frac{z_2}{\|z_2\|} \right\rangle
    \frac{z_2^\top z_2}{\|z_2\|^2}
    \right)
\end{align}

\subsection{Hessian approximations}\label{sec:hessian_approximations}


In classification and regression tasks, the Generalized Gauss Newton approximation is sufficient to guarantee positive definiteness, this is because the hessian of the loss with respect to the Neural Network output $\nabla^2_{z} \mathcal{L}_y(z)$ is positive definite both for MSE and cross-entropy. The overall loss in those cases is a sum over data points, without subtraction, and thus a positive matrix. For the contrastive loss, on the other hand, positive pairs contribute positively while negative pairs contribute negatively, and thus there is no guarantee in general. Moreover, with the Arccos loss there is not even the guarantee that $\nabla^2_{z} \mathcal{L}_y^A(z)$ is positive definite. Therefore, we consider three approximations of the hessian of the contrastive loss: full, positives, fixed, that ensures it to positive definite.

\textbf{Full.} The first possible approach is to forcefully ensure that the matrix only has positive values. This is formalized by applying an elementwise ReLU, that is
\begin{equation}
    \left[
    \nabla^2_\theta \mathcal{L}(\theta;\mathcal{D})
    \right]_{nm}
    \approx
    \max
    \left(
        0, 
        \left[
        \nabla^2_\theta \mathcal{L}(\theta;\mathcal{D})
        \right]_{nm}
    \right)
    \qquad
    \forall n,m
\end{equation}

\textbf{Positives.}
An alternative approach is to consider only the contribution of the positive pairs
\begin{equation}\label{eq:hessian_approximation_positives}
    \nabla^2_\theta \mathcal{L}(\theta;\mathcal{D})
    \approx 
    \sum_{\pp_{ij}\in\mathcal{D}^2_{\text{pos}}} \nabla^2_\theta\mathcal{L}_{y_{ij}}(f_\theta(x_i),f_\theta(x_j)) 
\end{equation}
neglecting the contribution of the negative pairs $\sum_{\mathcal{D}^2_{\text{neg}}}$. 
This approximation will be far from the truth and there is no theoretical justification. This approximation strategy is inspired by~\citet{shi2019probabilistic}, which only uses positive pairs to train the network responsible for predicting the variance.

\textbf{Fixed.} The third approach is to consider the contrastive term as a function of \emph{one data point at a time}, assuming the other fixed. This idea can be formalized by making use of the stop gradient notation: sg. The per-observation contrastive \cref{eq:per_obs} can be written as
\begin{equation}
    \mathcal{L}_y(z_1,z_2)
    = 
    \frac{1}{2}\mathcal{L}_y(\text{sg}[z_1],z_2)
    +
    \frac{1}{2}\mathcal{L}_y(z_1,\text{sg}[z_2]).
\end{equation}
We highlight that, with this definition, the zero- and first-order derivative does not change, and thus the loss and gradient are exactly the same as for the standard contrastive loss. 
However, the second-order derivative looses the cross term
\begin{equation}
    \nabla^2_{(z_1,z_2)} \mathcal{L}_y(z_1,z_2)
    =
    \begin{pmatrix}
        \nabla^2_{z_1} \mathcal{L}_y(z_1,z_2)
         & 0 \\
        0
         & \nabla^2_{z_2} \mathcal{L}_y(z_1,z_2)
    \end{pmatrix} 
\end{equation}
and the Hessian of the loss can then be more compactly written as
\begin{align}\label{eq:hessian_approximation_fixed}
    \nabla^2_\theta \mathcal{L}(\theta;\mathcal{D})
    & =
    \sum_{\pp_{ij}\in\mathcal{D}^2} \nabla^2_\theta\mathcal{L}_{y_{ij}}(f_\theta(x_i),f_\theta(x_j)) 
    \nonumber
    \\
    \nonumber
    & = \sum_{\pp_{ij}\in\mathcal{D}^2} y_{ij} 
    \begin{pmatrix}
        J_\theta f_\theta(x_i) \\ 
        J_\theta f_\theta(x_j)
    \end{pmatrix}^\top
    \begin{pmatrix}
        \nabla^2_{z_i} \mathcal{L}_{y_{ij}}(z_i,z_j)
         & 0 \\
        0
         & \nabla^2_{z_j} \mathcal{L}_{y_{ij}}(z_i,z_j)
    \end{pmatrix} 
    \begin{pmatrix}
        J_\theta f_\theta(x_i) \\
        J_\theta f_\theta(x_j)
    \end{pmatrix}  \\
    & = 2 \sum_{\pp_{ij}\in\mathcal{D}^2} 
    y_{ij} \, 
    J_\theta f_\theta(x_i) ^\top
    \nabla^2_{z_i} \mathcal{L}_{y_{ij}}(z_i,z_j)
    \,
    J_\theta f_\theta(x_i).
\end{align}


\begin{proposition}
    Consider the \textbf{full} approximation. Using only the diagonal of the hessian $\implies$ positive definiteness.
\end{proposition}
\begin{proof}
    The eigenvalues of a diagonal matrix are exactly the values on the diagonal. Enforcing these elements to be positive is equivalent to enforcing that all eigenvalues are positive. This implies the positive definiteness of the matrix.
\end{proof}

\begin{proposition}
    Consider the \textbf{positives} approximation. Euclidean loss $\implies$ positive definiteness.
\end{proposition}
\begin{proof}
Consider the hessian of \cref{eq:hessian_approximation_positives} and substitute the expression \cref{eq:hessian_loss_euclidean_wrt_z}
    \begin{align}
    \nabla^2_\theta \mathcal{L}(\theta;\mathcal{D})
    & \approx 
    \sum_{\pp_{ij}\in\mathcal{D}^2_{\text{pos}}} \nabla^2_\theta\mathcal{L}_{y_{ij}}(f_\theta(x_i),f_\theta(x_j)) \\
    & = \sum_{\pp_{ij}\in\mathcal{D}^2_{\text{pos}}} y_{ij} 
    \begin{pmatrix}
        J_\theta f_\theta(x_i) \\
        J_\theta f_\theta(x_j)
    \end{pmatrix}^\top
    \begin{pmatrix}
        \mathbb{I}
         & -\mathbb{I} \\
        -\mathbb{I}
         & \mathbb{I}
    \end{pmatrix}
    \begin{pmatrix}
        J_\theta f_\theta(x_i) \\
        J_\theta f_\theta(x_j)
    \end{pmatrix}  \\
    & = 
    \frac{1}{|\mathcal{D}^2_{\text{pos}}|}
    \sum_{\pp_{ij}\in\mathcal{D}^2_{\text{pos}}}
    \begin{pmatrix}
        J_\theta f_\theta(x_i) \\
        J_\theta f_\theta(x_j)
    \end{pmatrix}^\top
    \begin{pmatrix}
        \mathbb{I}
         & -\mathbb{I} \\
        -\mathbb{I}
         & \mathbb{I}
    \end{pmatrix}
    \begin{pmatrix}
        J_\theta f_\theta(x_i) \\
        J_\theta f_\theta(x_j)
    \end{pmatrix}
\end{align}
which is a sum of positive definite matrixes.
\end{proof}

\begin{proposition}\label{prop:positive_definiteness}
    Consider the \textbf{fixed} approximation. Euclidean loss $\implies$ positive definiteness.
\end{proposition}

\begin{proof}
Consider the hessian of \cref{eq:hessian_approximation_fixed} and substitute the expression \cref{eq:hessian_loss_euclidean_wrt_z}
\begin{align}
    \nabla^2_\theta \mathcal{L}(\theta;\mathcal{D})
    & =
    2 \sum_{\pp_{ij}\in\mathcal{D}^2} 
    y_{ij} \, J_\theta f_\theta(x_i) ^\top
    J_\theta f_\theta(x_i)  \\
    & = 2 \sum_{x_i,c_i \in\mathcal{D}}\sum_{x_j,c_j \in\mathcal{D}} 
    y_{ij} \, J_\theta f_\theta(x_i) ^\top
    J_\theta f_\theta(x_i)  \\
    & = 2 
    \sum_{x_i,c_i \in\mathcal{D}} 
        \left(
            \sum_{x_j,c_j \in\mathcal{D}} 
                y_{ij} 
        \right)
        J_\theta f_\theta(x_i) ^\top
        J_\theta f_\theta(x_i) \\
    & = 2 
    \sum_{x_i,c_i \in\mathcal{D}} 
        \left(
            \frac{|\mathcal{D}^2_{\text{pos}}(x_i)|}{|\mathcal{D}^2_{\text{pos}}(x_i)|}
            -
            \frac{|\mathcal{D}^2_{\text{neg inside}}(x_i)|}{|\mathcal{D}^2_{\text{neg}}(x_i)|}
        \right)
        J_\theta f_\theta(x_i) ^\top
        J_\theta f_\theta(x_i) \\
    & = 2 
    \sum_{x_i,c_i \in\mathcal{D}} 
        \underbrace{
        \left(
            1
            -
            \frac{|\mathcal{D}^2_{\text{neg inside}}(x_i)|}{|\mathcal{D}^2_{\text{neg}}(x_i)|}
        \right)
        }_{\geq0}
        J_\theta f_\theta(x_i) ^\top
        J_\theta f_\theta(x_i)
\end{align}
where $|\mathcal{D}^2_{\text{pos}}(x_i)|$ is the number of positive pairs containing the point $x_i$, and similarly for $|\mathcal{D}^2_{\text{neg}}(x_i)|$ and $|\mathcal{D}^2_{\text{neg inside}}(x_i)|$. The hessian is then equivalent to a sum of positive definite matrixes multiplied by a non-negative factor, and thus it is positive definite.
\end{proof}
We highlight, as we can see from the last Proposition's proof, that if all the negatives are inside the margin, the hessian is exactly 0. The more negatives are outside the margin, the more positive the hessian is.

\section{Experimental details}\label{sec:experimental_details}

In this section, we provide details on the experiments. We highlight that code for all experiments and baselines are available at \url{https://github.com/FrederikWarburg/bayesian-metric-learning}. 

\subsection{Datasets}

\textbf{FashionMnist.} We use the standard test-train split for FashionMnist \cite{xiao2017fashionmnist} and for MNIST \cite{lecun1998mnist}. We normalize the images in the range $[0,1]$ and do not perform any data augmentation during training. We train with RMSProp with learning rate $10^{-3}$ and default PyTorch settings, and with exponential learning rate decay with $\gamma = \exp(-0.1)$. We use a memory factor of $0.0001$ and maximum $5000$ pairs per train step. We trained for $20$ epochs.

\textbf{CIFAR10.}  We use the standard test-train split for CIFAR \cite{} and for SVHN \cite{}.  We normalize the images in the range $[0,1]$ and do not perform any data augmentation during training. We train with RMSProp with learning rate $10^{-5}$ and default PyTorch settings, and with exponential learning rate decay with $\gamma = \exp(-0.1)$. We use a memory factor of $0.0001$ and maximum $5000$ pairs per train step. We trained for $20$ epochs.

\textbf{CUB200.} The CUB-200-2011 dataset~\citep{WahCUB_200_2011} consists of $11\,788$ images of $200$ bird species. The birds are captured from different perspectives and in different environments, making this a challenging dataset for image retrieval. We follow the procedure of \citet{musgrave2020metric} and divide the first $100$ classes into the training set and the last $100$ classes into the test set. In this zero-shot setting, the trained models have not seen any of the bird species in the test set, and the learned features must generalize well across species. Similarly to \citet{warburg2020btl},  we use the Stanford Car-196~\citep{krause2013cars} dataset as OoD data. The Car-196 dataset is composed of $16\,185$ images of $196$ classes of cars. We conduct a similar split as \citet{musgrave2020metric}, and only evaluate on the last $98$ classes. This constitutes a very challenging zero-shot OoD dataset, where the model at test time needs to distinguish cars from birds, however, the model has not seen any of the bird species at training time.  We use imagenet normalization and during training augment with random resized crops and random horizontal flipping. We image $224 x 224$ image resolution. We train with RMSProp with learning rate $10^{-7}$ and default PyTorch settings, and with exponential learning rate decay with $\gamma = \exp(-0.1)$. We use a memory factor of $0.0001$ and maximum $30$ pairs per train step. We trained for $20$ epochs.

\textbf{LFW.} We use the face recognition dataset LFW~\cite{LFWTech} with the standard zero-shot train/test split, CUB200 as OoD data. It is challenging because of the high number of classes and few observations per class. LFW~\citep{LFWTech} consists of $13\,000$ images of $5\,749$ people using the standard zero-shot train/test split. One reason reliable uncertainties are important for face recognition systems is to avoid granting access based on an erroneous prediction. An example of such failure happened with the initial release of the Apple Face ID software, which failed to recognize underrepresented groups that were missing or underrepresented in the training distribution~\citep{theweek2018facial}. Reliable uncertainties and OoD detection might have mitigated such issues. We use imagenet normalization and during training augment with random resized crops and random horizontal flipping. We image $224 x 224$ image resolution. We train with RMSProp with learning rate $10^{-7}$ and default PyTorch settings, and with exponential learning rate decay with $\gamma = \exp(-0.1)$. We use a memory factor of $0.0001$ maximum $30$ pairs per train step. We trained for $200$ epochs.

\textbf{MSLS.} We use standard zero-shot train/val/test split \cite{Warburg_2020_CVPR}. We image $224 x 224$ image resolution. We train with RMSProp with learning rate $10^{-7}$ and default PyTorch settings, and with exponential learning rate decay with $\gamma = \exp(-0.1)$. We use a memory factor of $0.0001$ maximum $10$ pairs per train step. We trained for $200$ epochs.

\subsection{Model Architectures}

For all experiments, we use a last-layer diagonal LA. Across all experiments, our networks follow standard practices in image retrieval. For FashionMnist, the network is \texttt{[conv2d(1,32), relu, conv2d(32, 64), relu, maxpool2d(2), Flatten, Linear(9216)]} and for CiFAR10, the network is \texttt{[conv2d(1,32), relu, conv2d(32, 64), relu, maxpool2d(2), conv2d(64,64), relu, Flatten, Linear(9216)]}. For CUB200 and LFW, we use a pretrained ResNet50 backbone followed by a Generalized-Mean pooling layer and a dimension-preserving learned whitening layer (linear layer). The weight posterior is learned only for the last layer

\subsection{Metrics}\label{sec:metrics}

Evaluating the models' uncertainty estimates is tricky. We propose four metrics that capture both \textit{interpolation} and \textit{extrapolation} behaviors uncertainty estimates. To measure \textit{extrapolation} behavior, we measure the performance of out-of-distribution (OOD) detection and report Area Under Receiver Operator Curve (AUROC) and Area Under Precision-Recall Curve (AUPRC). These metrics describe the model's ability to assign high uncertainty to observations it has not seen during training (e.g., we train a model on birds and use images of cars as out-of-distribution examples), and are often used in unsupervised representation learning~\cite{miani_2022_neurips}. To measure \textit{interpolation} behavior, we measure the models' ability to assign reliable uncertainties to in-distribution data. Similarly to \citet{Wang2018deepface}, we measure the Area Under the Sparsification Curve. The sparsification curve is computed by iteratively removing the observation with the highest variance and recomputing the predictive performance (mAP@$1$). Lastly, for closed set datasets, we propose a method to compute Expected Calibration Error (ECE) for image retrieval. This is done by sampling $100$ latent variables $z_i$ from the predicted latent distributions $p(z | I)$. We then perform nearest neighbor classification on these samples to obtain $c_i$ predictions. We use the mode as the model prediction and the consensus with mode as the confidence, e.g., if they constitute $60$ samples, then the prediction has $60$\% confidence. We compare the estimated confidence with accuracy with
 \begin{equation}
   \mathrm{ECE} = \sum_i^N \frac{1}{|B_i|} \abs{\mathrm{acc}(B_i) - \mathrm{conf}(B_i)} ,
 \end{equation}
 where we $\mathrm{acc}(B_i)$ and $\mathrm{conf}(B_i)$ is the accuracy and confidence of the $i$\textsuperscript{th} bin. 
 
\subsection{Details on baselines}

If nothing else is stated we use the same architecture and training/evaluation code for all models. We also release all baseline models.

\begin{itemize}
    \item Deterministic: we train with the contrastive loss.
    \item MC dropout: we train with dropout between all trainable layers (except for the ResNet50 blocks). We use a dropout rate $0.2$. Dropout was enabled during test time.
    \item Deep Ensemble: Each ensemble consists of $5$ models, each initialized with different seeds.
    \item PFE: We add an uncertainty module that learns the variance. The rest of the model is frozen during training and initialized with a deterministically trained model. The uncertainty model consists of \texttt{[linear, batchnorm1d, relu, linear, batchnorm1d]}. We experimented with the last batch normalization (BN) layer sharing parameters or just a standard BN layer. We found the former to work slightly better. 
    \item HIB: We use the same uncertainty module as in PFE, and a kl weight of $10^{-4}$. We use $8$ samples during training to compute the loss.
\end{itemize}

\end{document}